\documentclass{article}

\usepackage[accepted]{icml2021}

\usepackage[utf8]{inputenc}

\usepackage{amsmath}
\usepackage{amssymb}
\usepackage{amsthm}
\usepackage{mathtools}
 \usepackage[bb=boondox]{mathalfa}
\usepackage{bm}
\PassOptionsToPackage{usenames,dvipsnames}{xcolor}

\usepackage{tikz}

\usepackage{hyperref}
\usepackage{cleveref}
\usepackage{enumitem}

\usepackage{etoc}

\usepackage{booktabs}
\usepackage{stmaryrd}
\DeclareMathOperator*{\argmin}{arg\,min}

\newcommand\angl[1]{\langle #1 \rangle}

\newcommand\lgna[1]{\rangle #1 \langle}

\newcommand\vone{\mathbb{1}}
\newcommand\vzero{\mathbb{0}}
\newcommand\ico{\bm{\rho}}
\newcommand\mar{\bm{\pi}}
\newcommand\tsp{\bm{\sigma}}

\newcommand\LTZ{\mathfrak{d}} 
\newcommand\BZC{\mathfrak{m}} 
\newcommand\GTC{\mathfrak{u}} 

\newcommand\clip{\mathtt{clip}_C} 

\newcommand\Walrus{\mathtt{Walrus}}
\newcommand\Shark{\mathtt{Shark}}

\newcommand\upsetp[1]{I_{\GTC}^{#1}}
\newcommand\upcar{n_{\GTC}}
\newcommand\upcarp[1]{n_{\GTC}^{#1}}

\newcommand\mdsetp[1]{I_{\BZC}^{#1}}
\newcommand\mdcar{n_{\BZC}}
\newcommand\mdcarp[1]{n_{\BZC}^{#1}}

\newtheorem{theorem}{Theorem}[section]
\crefname{theorem}{Theorem}{Theorems}
\newtheorem{lemma}[theorem]{Lemma}
\crefname{lemma}{Lemma}{Lemmas}

\newtheorem{corollary}[theorem]{Corollary}
\crefname{corollary}{Corollary}{Corollaries}
\newtheorem{proposition}[theorem]{Proposition}
\crefname{proposition}{Proposition}{Propositions}

\crefname{conjecture}{Conjecture}{Conjecture}

\theoremstyle{definition}

\newtheorem{definition}[theorem]{Definition}
\crefname{definition}{Definition}{Definitions}
\crefname{figure}{Figure}{Figures}
\crefname{SCfigure}{Figure}{Figures}
\crefname{section}{Section}{Sections}
\crefname{subsection}{Section}{Sections}
\crefname{algorithm}{Algorithm}{Algorithms}
\crefname{table}{Table}{Tables}

\theoremstyle{remark}

\crefname{equation}{}{}

\definecolor{mycolor1}{rgb}{0.4,0.4,0.4}
\definecolor{mycolor2}{rgb}{0.4,0.4,0.4}

\newcommand{\fakesection}[1]{%
  \par\refstepcounter{section}
  \sectionmark{#1}
  \addcontentsline{toc}{section}{\protect\numberline{\thesection}#1}
}

\makeatletter
\newenvironment{subroutine}[1][htb]{%
  \renewcommand{\ALG@name}{Subroutine}
  \begin{algorithm}[#1]%
  }{\end{algorithm}
}
\makeatother

\usepackage{appendix}

\usepackage{minitoc}

\icmltitlerunning{Weston-Watkins SVM subproblem}

\begin{document}
\dosecttoc
\faketableofcontents

\twocolumn[
\icmltitle{An Exact Solver for the Weston-Watkins SVM Subproblem}



\icmlsetsymbol{equal}{*}

\begin{icmlauthorlist}
\icmlauthor{Yutong Wang}{eecs}
\icmlauthor{Clayton Scott}{eecs,stats}
\end{icmlauthorlist}

\icmlaffiliation{eecs}{Department of Electrical Engineering and Computer Science,
University of Michigan.}
\icmlaffiliation{stats}{Department of Statistics,
University of Michigan}

\icmlcorrespondingauthor{Yutong Wang}{yutongw@umich.edu}
\icmlcorrespondingauthor{Clayton Scott}{clayscot@umich.edu}

\icmlkeywords{Machine Learning, ICML}

\vskip 0.3in
]



\printAffiliationsAndNotice{}  

\begin{abstract}
Recent empirical evidence suggests that the Weston-Watkins support vector machine is among the best performing multiclass extensions of the binary SVM.
Current state-of-the-art solvers repeatedly solve a particular subproblem approximately using an iterative strategy.
In this work, we propose an algorithm that solves the subproblem exactly using a novel reparametrization of the Weston-Watkins dual problem.
For linear WW-SVMs, our solver shows significant speed-up over the state-of-the-art solver when the number of classes is large.
Our exact subproblem solver also allows us to prove linear convergence of the overall solver.




\end{abstract}


\section{Introduction}

Support vector machines (SVMs) \cite{boser1992training,cortes1995support} are a powerful class of algorithms for classification.
In the large scale studies by \citet{fernandez2014we} and by \citet{klambauer2017self}, SVMs are shown to be among the best performing classifers.

The original formulation of the SVM  handles only binary classification.
Subsequently, several variants of multiclass SVMs have been proposed \cite{lee2004multicategory,crammer2001algorithmic,weston1999support}.
However, as pointed out by \citet{dogan2016unified}, no variant has been considered canonical.

The empirical study of \citet{dogan2016unified} compared nine prominent variants of multiclass SVMs and demonstrated that the Weston-Watkins (WW) and Crammer-Singer (CS) SVMs performed the best with the WW-SVM holding a slight edge in terms of both efficiency and accuracy.
This work focuses on the computational issues of solving the WW-SVM optimization efficiently.

SVMs are typically formulated as quadratic programs.
State-of-the-art solvers such as LIBSVM \cite{chang2011libsvm} and LIBLINEAR \cite{fan2008liblinear}
apply block coordinate descent to
the associated dual problem,
which entails repeatedly solving
many small subproblems.
For the binary case, these subproblems are easy to solve exactly.

The situation in the multiclass case is more complex, where the form of the subproblem depends on the variant of the multiclass SVM.
For the CS-SVM, the subproblem can be solved exactly in $O(k \log k)$ time where $k$ is the number of classes
\cite{crammer2001algorithmic,duchi2008efficient,blondel2014large,condat2016fast}.
However, for the WW-SVM, only iterative algorithms that approximate the subproblem minimizer have been proposed, and these lack runtime guarantees \cite{keerthi2008sequential,igel2008shark}.

In this work, we propose an algorithm called \emph{Walrus}\footnote{\underline{W}W-subproblem \underline{a}nalytic \underline{l}og-linear \underline{ru}ntime \underline{s}olver} that finds the exact solution of the Weston-Watkins subproblem in $O(k\log k)$ time.
We implement Walrus in C++ inside the LIBLINEAR  framework, yielding a new solver for the \emph{linear} WW-SVM.
For datasets with large number of classes, we demonstrate significant speed-up over the state-of-the-art linear solver Shark \cite{igel2008shark}.
We also rigorously prove the linear convergence of block coordinate descent for solving the dual problem of linear WW-SVM, confirming an assertion of \citet{keerthi2008sequential}.

\subsection{Related works}

Existing literature on solving the optimization from SVMs largely fall into two categories: linear and kernel SVM solvers.
The seminal work of \citet{platt1998sequential} introduced the sequential minimal optimization (SMO) for solving kernel SVMs.
Subsequently, many SMO-type algorithms were introduced which achieve faster convergence with theoretical guarantees
\citep{keerthi2001improvements,fan2005working,steinwart2011training,torres2021faster}.

SMO can be thought of as a form of \emph{(block) coordinate descent} where 
where the dual problem of the SVM optimization is decomposed into small subproblems.
As such, SMO-type algorithms are also referred to as \emph{decomposition methods}.
For binary SVMs, the smallest subproblems are $1$-dimensional and thus easy to solve exactly.
However, for multiclass SVMs with $k$ classes, the smallest subproblems are $k$-dimensional.
Obtaining exact solutions for the subproblems is nontrivial.

Many works have studied the convergence properties of decomposition focusing on
asymptotics \cite{list2004general},
rates \cite{chen2006study,list2009svm},
binary SVM without offsets 
\cite{steinwart2011training},
and multiclass SVMs 
\cite{hsu2002comparison}.
Another line of research focuses on primal convergence instead of the dual
\citep{hush2006qp,list2007general,list2007gaps,beck2018primal}.

Although kernel SVMs include linear SVMs as a special case, solvers specialized for linear SVMs can scale to larger data sets.
Thus, linear SVM solvers are often developed separately.
\citet{hsieh2008dual} proposed using coordinate  descent (CD) to solve the linear SVM dual problem and established linear convergence.
Analogously, \citet{keerthi2008sequential} proposed block coordinate descent (BCD) for multiclass SVMs.
Coordinate descent on the dual problem is now used by the current state-of-the-art linear SVM solvers LIBLINEAR \cite{fan2008liblinear}, liquidSVM \cite{steinwart2017liquidsvm}, and Shark \cite{igel2008shark}.

There are other approaches to solving linear SVMs, e.g.,
using the cutting plane method \cite{joachims2006training},
and stochastic subgradient descent on the primal optimization
\cite{shalev2011pegasos}.
However, these approaches do not converge as fast as CD on the dual problem \cite{hsieh2008dual}.

For the CS-SVM introduced by \citet{crammer2001algorithmic}, an exact solver for the subproblem is well-known and 
there is a line of research on improving the solver's efficiency
\cite{crammer2001algorithmic,duchi2008efficient,blondel2014large,condat2016fast}.
For solving the kernel CS-SVM dual problem,
convergence of an SMO-type algorithm was proven in \cite{lin2002formal}.
For solving the linear CS-SVM dual problem, 
linear convergence of coordinate descent was proven by \citet{lee2019distributed}.
Linear CS-SVMs with $\ell_1$-regularizer have been studied by 
\citet{babichev2019efficient}

The Weston-Watkins SVM was introduced by \citet{bredensteiner1999multicategory,weston1999support,vapnik1998statistical}.
Empirical results from \citet{dogan2016unified} suggest that the WW-SVM is the best performing multiclass SVMs among nine prominent variants.
The WW-SVM loss function has also been successfully used in natural language processing by \citet{schick2020s}.

\citet{hsu2002comparison} gave an SMO-type algorithm for solving the WW-SVM, although without convergence guarantees.
\citet{keerthi2008sequential} proposed using coordinate descent on the linear WW-SVM dual problem with an iterative subproblem solver.
Furthermore, they asserted that the algorithm converges linearly, although no proof was given.
The software Shark \cite{igel2008shark} features a solver for the linear WW-SVM where the subproblem is approximately minimized by a greedy coordinate descent-type algorithm.
MSVMpack \cite{didiot2015efficient} is a solver for multiclass SVMs which uses the Frank-Wolfe algorithm.
The experiments of \cite{van2016gensvm} showed that MSVMpack did not scale to larger number of classes for the WW-SVM.
To our knowledge, an exact solver for the subproblem has not previously been developed.

\subsection{Notations}

Let $n$ be a positive integer. Define $[n] := \{1,\dots, n\}$.
All vectors are assumed to be column vectors unless stated otherwise.
If $v \in \mathbb{R}^n$ is a vector and $i \in [n]$, we use the notation $[v]_i$ to denote the $i$-th component of $v$.
Let $\vone_n$ and $\vzero_n \in \mathbb{R}^n$ denote the vectors of all ones and zeros, respectively.
When the dimension $n$ can be inferred from the context,
we drop the subscript and simply write $\vone$ and $\vzero$.

Let $m$ be a positive integer. Matrices $\mathbf{w} \in \mathbb{R}^{m\times n}$ are denoted by boldface font.
The $(j,i)$-th entry of $\mathbf{w}$ is denoted by $w_{ji}$.
The columns of $\mathbf{w}$ are denoted by the same symbol $w_1,\dots, w_n$ using regular font with a single subscript, i.e., $[w_{i}]_j = w_{ji}$.
A column of $\mathbf{w}$ is sometimes referred to as a \emph{block}.
We will also use boldface Greek letter to denote matrices, e.g., $\bm{\alpha} \in \mathbb{R}^{m\times n}$ with columns $\alpha_1,\dots, \alpha_n$.

The 2-norm of a vector $v$ is denoted by $\|v\|$. The Frobenius norm of a matrix $\mathbf{w}$ is denoted by $\|\mathbf{w}\|_F$.
The $m\times m$ identity and all-ones matrices are denoted by $\mathbf{I}_m$ and $\mathbf{O}_m$, respectively.
When $m$ is clear from the context, we drop the subscript and simply write $\mathbf{I}$ and $\mathbf{O}$.

For referencing, section numbers from our supplementary materials will be prefixed with an ``A'', e.g., Section~\ref{section: full proof of generic solver}.

\section{Weston-Watkins linear SVM}

Throughout this work, let $k \ge 2$ be an integer denoting the number of classes.
Let $\{(x_i,y_i)\}_{i \in [n]}$ be a training dataset of size $n$ where the instances $x_i \in \mathbb{R}^d$ and labels $y_i \in [k]$.
The Weston-Watkins linear SVM 
\footnote{
  Similar to other works on multiclass linear SVMs \citep{hsu2002comparison,keerthi2008sequential},
the formulation \cref{equation: WW-SVM primal optimization}
  does not use \emph{offsets}.
For discussions, see Section~\ref{section: offsets}.
}
solves the optimization
\begin{equation}
  \label{equation: WW-SVM primal optimization}
  \tag{P}
  \min_{\mathbf{w} \in \mathbb{R}^{d\times k}}
  \frac{1}{2}\|\mathbf{w}\|_F^2 + C \sum_{i=1}^n 
  \sum_{\substack{j \in [k]:\\ j \ne y_i}}
  \mathtt{hinge}(w_{y_i}' x_i - w_j'x_i)
\end{equation}
where $\mathtt{hinge}(t) = \max\{0, 1- t\}$
and $C > 0$ is a hyperparameter.

Note that if an instance $x_i$ is the zero vector, then for any $\mathbf{w} \in \mathbb{R}^{d \times k}$ we have 
$\mathtt{hinge}(w_{y_i}' x_i - w_j'x_i) = 1$.
Thus, we can simply ignore such an instance. Below, we assume that $\|x_i \| > 0$ for all $i \in [n]$.

\subsection{Dual of the linear SVM}

In this section, we recall the dual of \cref{equation: WW-SVM primal optimization}.
Derivation of all results here can be found in \citet{hsu2002comparison,keerthi2008sequential}.

We begin by defining the function $f: \mathbb{R}^{k\times n} \to \mathbb{R}$
\[
  f(\bm{\alpha}):=
  \frac{1}{2}
  \sum_{i,s\in [n]}
  x_s'x_i\alpha_i'\alpha_s
  -
  \sum_{i\in[k]}
  \sum_{\substack{j \in [k]:\\ j \ne y_i}}
  \alpha_{ij}
\]
and the set
\begin{align*}
  \mathcal{F}
  :=
  \Big\{&
    \bm{\alpha} \in \mathbb{R}^{k\times n}\mid\\
                                           &0 \le \alpha_{ij} \le C, \, \forall i \in [n],\, j \in [k], \, j \ne y_i,\\
                                           &
            \alpha_{iy_i}
            =
            -
            \sum_{j\in[k] \setminus \{y_i\}}
            \alpha_{ij}, \, \forall i \in [n]
  \Big\}.
\end{align*}
The dual problem of \cref{equation: WW-SVM primal optimization} is
\begin{equation}
  \tag{D1}
  \label{equation: WW-SVM dual optimization}
  \min_{\bm{\alpha}  \in \mathcal{F}
} \,\, f(\bm{\alpha}).
\end{equation}

The primal and dual variables $\mathbf{w}$ and $\bm{\alpha}$ are related via
\begin{equation}
  \label{equation: primal-dual variable relation}
  \mathbf{w}
  =
  -\sum_{i \in [n]} x_i \alpha_i'.
\end{equation}

State-of-the-art solver
Shark
\cite{igel2008shark}
uses coordinate descent on the dual problem \cref{equation: WW-SVM dual optimization}.
It is also possible to solve the primal problem \cref{equation: WW-SVM primal optimization} using stochastic gradient descent (SGD) as in Pegasos \cite{shalev2011pegasos}.
However, the empirical results of \citet{hsieh2008dual} show that CD on the dual problem converges faster than SGD on the primal problem.
Hence, we focus on the dual problem.

\subsection{Solving the dual with block coordinate descent}
\label{section: BCD}

Block coordinate descent (BCD) is an iterative algorithm for solving the dual problem \cref{equation: WW-SVM dual optimization} by repeatedly improving a candidate solution $\bm{\alpha} \in \mathcal{F}$.
Given an $i \in [n]$,
an \emph{inner iteration} performs the update $\bm{\alpha} \gets \widetilde{\bm{\alpha}}$ where $\widetilde{\bm{\alpha}}$ is a minimizer of
the \emph{$i$-th subproblem}:
\begin{equation}
  \tag{S1}
  \label{equation: WW-SVM dual subproblem}
  \min_{\widehat{\bm{\alpha}} \in \mathcal{F}}
  \,\, f(\widehat{\bm{\alpha}}) \,\, \mbox{such that} \,\, 
  \, \widehat{\alpha}_s = \alpha_s,\, \forall s \in [n] \setminus \{i\}.
\end{equation}
An \emph{outer iteration} performs the inner iteration once for each $i \in [n]$ possibly in a random order.
By running several outer iterations, an (approximate) minimizer of \cref{equation: WW-SVM dual optimization} is putatively obtained.

Later, we will see that it is useful to keep track of $\mathbf{w}$ so that \cref{equation: primal-dual variable relation} holds throughout the BCD algorithm.
Suppose that $\bm{\alpha}$ and $\mathbf{w}$ satisfy \cref{equation: primal-dual variable relation}.
Then $\mathbf{w}$ must be updated via
\begin{equation}
  \label{equation: W update formula}
  \mathbf{w} \gets \mathbf{w} - x_i(\widetilde{\alpha}_i - \alpha_i)'
\end{equation}
prior to updating $\bm{\alpha} \gets \widetilde{\bm{\alpha}}$.

\section{Reparametrization of the dual problem}
In this section, we introduce a new way to parametrize the dual optimization \cref{equation: WW-SVM dual optimization} which allows us to derive an algorithm for finding the exact minimizer of 
  \cref{equation: WW-SVM dual subproblem}.

  Define the matrix $
    \mar := 
    \begin{bmatrix}
      \vone & -\mathbf{I}
    \end{bmatrix}
 \in \mathbb{R}^{(k-1)\times k}$.
 For each $y \in [k]$, let $\tsp_y \in \mathbb{R}^{k\times k}$ be the permutation matrix which switches the $1$st and the $y$th indices.
 In other words, given a vector $v \in \mathbb{R}^k$, we have
 \[
   [\tsp_y(v)]_{j}
   =
   \begin{cases}
     v_1 &: j = y \\
     v_y &: j = 1\\
     v_j &: j \not \in \{1,y\}.
   \end{cases}
 \]

 Define the function $g : \mathbb{R}^{(k-1)\times n} \to \mathbb{R}$
  \[
    g(\bm{\beta}):=
  \frac{1}{2}
\sum_{i,s \in [n]}
x_s' x_i
  \beta_i' \mar \tsp_{y_i} 
  \tsp_{y_s} \mar'
  \beta_s
  -
\sum_{i \in [n]} \vone' \beta_i
  \]
  and the set
\begin{align*}
  \mathcal{G}
  :=
  \Big\{&
    \bm{\beta} \in \mathbb{R}^{(k-1)\times n}\mid\\
        &0 \le \beta_{ij} \le C, \, \forall i \in [n],\, j \in [k - 1]
  \Big\}.
\end{align*}
Consider the following optimization:
\begin{equation}
  \tag{D2}
  \label{equation: WW-SVM dual optimization reparametrized}
  \min_{\bm{\beta} \in \mathcal{G}
}
  \,\, g(\bm{\beta}).
\end{equation}

Up to a change of variables, the optimization \cref{equation: WW-SVM dual optimization reparametrized} is equivalent to the dual of the linear WW-SVM \cref{equation: WW-SVM dual optimization}.
In other words, \cref{equation: WW-SVM dual optimization reparametrized} is a reparametrization of \cref{equation: WW-SVM dual optimization}.
Below, we make this notion precise.

\begin{definition}
  Define a map $\Psi : \mathcal{G} \to \mathbb{R}^{k\times n}$ as follows:
  Given $\bm{\beta} \in \mathcal{G}$, construct an element $\Psi(\bm{\beta}) := \bm{\alpha} \in \mathbb{R}^{k\times n}$ whose $i$-th block is
\begin{equation}
  \label{equation: beta alpha relation via transposition}
  \alpha_i =
  -\tsp_{y_i}  \mar' \beta_i.
\end{equation}
\end{definition}
The map $\Psi$ will serve as the change of variables map, where $\mar$ reduces the dual variable's dimension from $k$ for $\alpha_i$ to $k-1$ for $\beta_i$. Furthermore, $\tsp_{y_i}$ eliminates the  dependency on $y_i$ in the constraints.
The following proposition shows that $\Psi$ links the two optimization problems \cref{equation: WW-SVM dual optimization} and \cref{equation: WW-SVM dual optimization reparametrized}.

\begin{proposition}
  \label{proposition: reparametrized dual problem}
  The image of $\Psi$ is $\mathcal{F}$, i.e., $\Psi(\mathcal{G}) = \mathcal{F}$.
  Furthermore, $\Psi : \mathcal{G} \to \mathcal{F}$ is a bijection and 
  \[f(\Psi(\bm{\beta})) = g(\bm{\beta}).\]
\end{proposition}
\begin{proof}
  [Sketch of proof]
  Define another map $\Xi: \mathcal{F} \to \mathbb{R}^{(k-1) \times n}$ as follows:
  For each $\bm{\alpha} \in \mathcal{F}$, define $\bm{\beta} := \Xi(\bm{\alpha})$ block-wise by
\[
  \beta_i := \mathtt{proj}_{2:k}(\tsp_{y_i}\alpha_i) \in \mathbb{R}^{k-1}
\]
where 
\[
  \mathtt{proj}_{2:k}
  =
  \begin{bmatrix}
    \vzero &
    \mathbf{I}_{k-1}
  \end{bmatrix}
  \in \mathbb{R}^{(k-1) \times k}.
\]
Then the range of  $\Xi$ is in $\mathcal{G}$.
Furthermore, $\Xi$ and $\Psi$ are inverses of each other.
This proves that $\Psi$ is a bijection.
\end{proof}



\subsection{Reparametrized subproblem}

Since the map $\Psi$ respects the block-structure of $\bm{\alpha}$ and $\bm{\beta}$, the result below follows immediately from 
  \cref{proposition: reparametrized dual problem}:
\begin{corollary}
  \label{corollary: equivalence of dual subproblems}
  Let $\bm{\beta} \in \mathcal{G}$ and $i \in [n]$. Let $\bm{\alpha} = \Psi(\bm{\beta})$.
  Consider
\begin{equation}
  \tag{S2}
  \label{equation: WW-SVM dual subproblem - reparametrized}
\min_{\widehat{\bm{\beta}} \in \mathcal{G}}
  \,\, g(\widehat{\bm{\beta}}) \,\, \mbox{such that} \,\, 
  \, \widehat{\beta}_s = \beta_s,\, \forall s \in [n] \setminus \{i\}.
\end{equation}
Let $\widetilde{\bm{\beta}} \in \mathcal{F}$ be arbitrary.
Then $\widetilde{\bm{\beta}}$ is a minimizer of \cref{equation: WW-SVM dual subproblem - reparametrized} if and only if $\widetilde{\bm{\alpha}} := \Psi(\widetilde{\bm{\beta}})$ is a minimizer of \cref{equation: WW-SVM dual subproblem}.
\end{corollary}

Below, we focus on solving \cref{equation: WW-SVM dual optimization reparametrized} with BCD, i.e., repeatedly performing the update $\bm{\beta} \gets \widetilde{\bm{\beta}}$ where $\widetilde{\bm{\beta}}$ is a minimizer of \cref{equation: WW-SVM dual subproblem - reparametrized} over different $i\in [n]$.
By \cref{corollary: equivalence of dual subproblems}, this is equivalent to solving \cref{equation: WW-SVM dual optimization} with BCD, up to the change of variables $\Psi$.

The reason we focus on solving \cref{equation: WW-SVM dual optimization reparametrized} with BCD is because 
the subproblem can be cast in a simple form that makes an exact solver more apparent.
To this end, we first show that the subproblem \cref{equation: WW-SVM dual subproblem - reparametrized} is a quadratic program of a particular form.
Define the matrix
  $
  \bm{\Theta} := \mathbf{I}_{k-1} + \mathbf{O}_{k-1}.
$
\begin{theorem}
  \label{theorem: subproblem generic solver}
  Let $v \in \mathbb{R}^{k-1}$ be arbitrary and $C > 0$.
  Consider the optimization
\begin{align}
  \label{equation: dual subproblem generic}
  \min_{b \in \mathbb{R}^{k-1}}
  \quad & \frac{1}{2}b' \bm{\Theta} b - v' b
  \\
  s.t. \quad & 0 \le b \le C. \nonumber
\end{align}
Then \cref{algorithm: subproblem generic solver}, $\mathtt{solve\_subproblem}(v,C)$,
computes the unique minimizer of \cref{equation: dual subproblem generic} in $O(k\log k)$ time.
\end{theorem}
We defer further discussion of \cref{theorem: subproblem generic solver} and \cref{algorithm: subproblem generic solver} to the next section.
The quadratic program \cref{equation: dual subproblem generic} is the \emph{generic} form of the subproblem \cref{equation: WW-SVM dual subproblem - reparametrized}, as the following result shows:
\begin{proposition}
  \label{proposition: subproblem reduction to generic form}
  In the situation of \cref{corollary: equivalence of dual subproblems}, let $\widetilde{\beta}_i$ be the $i$-th block of the minimizer $\widetilde{\bm{\beta}}$ of
  \cref{equation: WW-SVM dual subproblem - reparametrized}.
  Then $\widetilde{\beta}_i$ is the unique minimizer of \cref{equation: dual subproblem generic} with
    \[v := (\vone
      -  \mar\tsp_{y_i} \mathbf{w}' x_i)/\|x_i\|_2^2
    + \bm{\Theta} \beta_i 
  \]
  and $\mathbf{w}$ as in \cref{equation: primal-dual variable relation}.
\end{proposition}

\subsection{BCD for the reparametrized dual problem}
As mentioned in \cref{section: BCD}, it is useful to 
keep track of $\mathbf{w}$ so that \cref{equation: primal-dual variable relation} holds throughout the BCD algorithm.
In 
\cref{proposition: subproblem reduction to generic form}, we see that $\mathbf{w}$ is used to compute $v$.
The update formula \cref{equation: W update formula} for $\mathbf{w}$ in terms of $\widetilde{\bm{\alpha}}$ can be cast in terms of $\bm{\beta}$ and $\widetilde{\bm{\beta}}$ by using \cref{equation: beta alpha relation via transposition}:
\begin{align*}
  \mathbf{w} \gets \mathbf{w} - x_i(\widetilde{\alpha}_i - \alpha_i)'
=\mathbf{w} + x_i (\widetilde{\beta}_i -\beta_i)'\mar\tsp_{y_i}.
\end{align*}

We now have all the ingredients to state the reparametrized block coordinate descent pseudocode in \cref{algorithm: BCD}.
\begin{algorithm}
  \begin{algorithmic}[1]
    \STATE $\bm{\beta} \gets \mathbf{0}_{(k-1)\times n}$
    \STATE $\mathbf{w} \gets \mathbf{0}_{d\times k}$
 \WHILE{not converged}
  \FOR{$i \gets 1$ \textbf{to} $n$} 
    \STATE $v \gets  (\vone 
    - \mar\tsp_{y_i} \mathbf{w}' x_i)/\|x_i\|_2^2
    + \bm{\Theta} \beta_i 
    $
    \STATE $\widetilde{\beta}_i \gets \mathtt{solve\_subproblem}(v,C)$
    (\cref{algorithm: subproblem generic solver})
    \STATE $\mathbf{w} \gets \mathbf{w} + x_i (\widetilde{\beta}_i -\beta_i)'\mar\tsp_{y_i}$
    \STATE$\beta_i \gets \widetilde{\beta}_i$
  \ENDFOR
  \ENDWHILE
  \caption{Block coordinate descent on \cref{equation: WW-SVM dual optimization reparametrized}}
 \label{algorithm: BCD}
  \end{algorithmic}
\end{algorithm}

Multiplying a vector by the matrices $\bm{\Theta}$ and $\mar$ both only takes $O(k)$ time.
Multiplying a vector by $\tsp_{y_i}$ takes $O(1)$ time since $\tsp_{t_i}$ simply swaps two entries of the vector.
Hence, the speed bottlenecks of \cref{algorithm: BCD} are 
  computing 
$\mathbf{w}' x_i$ and $x_i (\widetilde{\beta}_i -\beta_i)'$, both taking $O(dk)$ time and 
running
$\mathtt{solve\_subproblem}(v,C)$, which takes $O(k \log k)$ time.
Overall, a single inner iteration of \cref{algorithm: BCD} takes $O(dk + k \log k)$ time.
If $x_i$ is $s$-sparse (only $s$ entries are nonzero), then the iteration takes $O(sk + k \log k)$ time.

\subsection{Linear convergence} 

Similar to the binary case \cite{hsieh2008dual}, BCD converges \emph{linearly}, i.e., it produces an $\epsilon$-accurate solution in $O(\log(1/\epsilon))$ outer iterations:
\begin{theorem}
  \label{theorem: linear convergence}
  \cref{algorithm: BCD} has global linear convergence.
More precisely, let $\bm{\beta}^t$ be $\bm{\beta}$ at the end of the $t$-th iteration of the outer loop of \cref{algorithm: BCD}.
  Let $g^* = \min_{\bm{\beta} \in \mathcal{G}} g(\bm{\beta})$.
Then there exists $\Delta \in (0,1)$ such that
  \begin{equation}
    \label{equation: Q-linear convergence}
    g(\bm{\beta}^{t+1}) - g^* \le \Delta (g(\bm{\beta}^{t}) - g^*), \qquad \forall t = 0,1,2\dots
  \end{equation}
  where $\Delta$ depends on 
  the data $\{(x_i,y_i)\}_{i \in [n]}$, $k$ and $C$.
\end{theorem}
\citet{luo1992convergence} proved asymptotic\footnote{
Asymptotic in the sense that \cref{equation: Q-linear convergence} is only guaranteed after $t > t_0$ for some unknown $t_0$.
} linear convergence for cyclic coordinate descent for a certain class of minimization problems where the subproblem in each coordinate is \emph{exactly} minimized.
Furthermore, \citet{luo1992convergence} claim that the same result holds if the subproblem is \emph{approximately} minimized, but did not give a precise statement (e.g., approximation in which sense).

\citet{keerthi2008sequential} asserted without proof that the results of \citet{luo1992convergence} can be applied to BCD for WW-SVM.
Possibly, no proof was given since no solver, exact nor approximate with approximation guarantees, was known at the time.
\cref{theorem: linear convergence} settles this issue, which we prove in Section~\ref{section: global linear convergence}
 by extending the analysis of  \citet{luo1992convergence,wang2014iteration} to the multiclass case.

\section{Sketch of proof of \texorpdfstring{\cref{theorem: subproblem generic solver}}{main theorem}}
\label{section: the subproblem solver}

Throughout this section, let $v \in \mathbb{R}^{k-1}$ and $C > 0$ be fixed.
  We first note that \cref{equation: dual subproblem generic} is a minimization of a strictly convex function over a compact domain, and hence has unique minimizer $\widetilde{b} \in \mathbb{R}^{k-1}$.
  Furthermore, it is the unique point satisfying the KKT conditions, which we present below.
  Our goal is to sketch the argument that \cref{algorithm: subproblem generic solver} outputs the minimizer upon termination.
The full proof can be found in Section~\ref{section: full proof of generic solver}.

\subsection{Intuition}
We first study the structure of the minimizer $\widetilde{b}$ in and of itself.
The KKT conditions for a point $b \in \mathbb{R}^{k-1}$ to be optimal for \cref{equation: dual subproblem generic} are as follows:
\begin{align}
  \forall i \in [k-1], \exists \lambda_i,\mu_i \in \mathbb{R} &\mbox{ satisfying} \nonumber \\
  [(\mathbf{I}+\mathbf{O})b]_i + \lambda_i - \mu_i =v_i \quad & \mbox{stationarity} 
  \label{roadmap: KKT}\tag{KKT}
  \\
  C  \ge b_i  \ge 0 \quad &\mbox{primal feasibility} \nonumber
  \\
  \lambda_i  \ge 0
  ,\mbox{ and }
  \mu_i  \ge 0 
  \quad &\mbox{dual feasibility} \nonumber
  \\
  \lambda_i (C-b_i) = 0 
  ,\mbox{ and }
  \mu_i b_i = 0
  \quad &\mbox{complementary slackness} \nonumber
\end{align}
Below, let $\max_{i \in [k-1]} v_i = : v_{\max}$,
 and $\angl{1}, \dots, \angl{k-1}$ be an argsort of $v$, i.e.,
    $
    v_{\angl{1}} \ge \dots \ge v_{\angl{k-1}}.
    $

\begin{definition}
  The \emph{clipping map} $\clip : \mathbb{R}^{k-1} \to [0,C]^{k-1}$ is the function defined as follows: for 
  $w \in \mathbb{R}^{k-1}$, $[\clip(w)]_i :=\max\{0, \min \{C,w_i\}\}$.
\end{definition}

Using the KKT conditions, we check that
$
  \widetilde{b} = \clip(v- \widetilde{\gamma} \vone)
$
for some (unknown) $\widetilde{\gamma} \in \mathbb{R}$ and that $\widetilde{\gamma} = \vone ' \widetilde{b}$.
\begin{proof}
Let $\widetilde{\gamma} \in \mathbb{R}$ be such that $\mathbf{O}\widetilde{b} = \widetilde{\gamma} \vone$.
The stationarity condition can be rewritten as $\widetilde{b}_i + \lambda_i - \mu_i = v_i - \widetilde{\gamma}$.
Thus, by complementary slackness and dual feasibility, we have
\[
  \widetilde{b}_i
  \begin{cases}
    \le v_i - \widetilde{\gamma} &: 
    \widetilde{b}_i = C\\
    = v_i - \widetilde{\gamma} &: 
    \widetilde{b}_i \in (0,C)\\
    \ge v_i - \widetilde{\gamma} &: 
    \widetilde{b}_i = 0
  \end{cases}
\]
Note that this is precisely
$\widetilde{b} = \clip(v- \widetilde{\gamma} \vone)$.
\end{proof}

For $\gamma \in \mathbb{R}$, let $b^\gamma := \clip(v- {\gamma} \vone) \in \mathbb{R}^{k-1}$.
Thus,  the $(k-1)$-dimensional vector $\widetilde{b}$ can be recovered from the scalar $\widetilde{\gamma}$ via $b^{\widetilde{\gamma}}$, reducing the search space from $\mathbb{R}^{k-1}$ to $\mathbb{R}$.

However, the search space $\mathbb{R}$ is still a continuum.
We show that the search space for $\widetilde{\gamma}$ can be further reduced to a finite set of candidates.
To this end, let us define
\begin{align*} 
  \upsetp{\gamma} &:=\{i \in [k-1]: b^\gamma_i = C\}\\
  \mdsetp{\gamma} &:=\{i \in [k-1]: b^\gamma_i \in (0,C)\}.
\end{align*}
Note that $\upsetp{\gamma}$ and $\mdsetp{\gamma}$ are determined by their cardinalities, denoted $\upcarp{\gamma}$ 
    and
    $\mdcarp{\gamma}$, respectively. This is because
\begin{align*}
  \upsetp{\gamma} &= 
  \{\angl{1}, \angl{2},\dots, \angl{n_{\GTC}^\gamma}\}\\
  \mdsetp{\gamma} &=
  \{\angl{n_{\GTC}^\gamma+1},\angl{n_{\GTC}^\gamma+2},\dots, \angl{n_{\GTC}^\gamma + n_{\BZC}^\gamma}\}.
\end{align*}

    Let $\llfloor k \rrfloor := \{0\} \cup [k-1]$.
    By definition, $\mdcarp{\gamma},
    \upcarp{\gamma} \in \llfloor k \rrfloor$.
For $(n_{\BZC},n_{\GTC}) \in \llfloor k \rrfloor^2$, define
$S^{(n_{\BZC}, n_{\GTC})}$, $\widehat{\gamma}^{(n_{\BZC}, n_{\GTC})} \in \mathbb{R}$ by
\begin{align}
  S^{(n_{\BZC}, n_{\GTC})} &:= 
  \sum_{i = n_{\GTC}+1}^{n_{\GTC} + n_{\BZC}} v_{\angl{i}},
  \label{roadmap: S superscript t}
  \\
  \widehat{\gamma}^{(n_{\BZC}, n_{\GTC})} 
                           &:=
                           \left(C\cdot n_{\GTC} + S^{(n_{\BZC}, n_{\GTC})}\right)/( n_{\BZC} + 1 ).
  \label{roadmap: gamma superscript t}
\end{align}
Furthermore, define 
$\widehat{b}^{(n_{\BZC}, n_{\GTC})} 
\in \mathbb{R}^{k-1}$ such that, for $i \in [k-1]$, the $\angl{i}$-th entry is
\begin{align*}
    \widehat{b}_{\angl{i}}^{(n_{\BZC}, n_{\GTC})}
           &:=
      \begin{cases}
        C &: i \le n_{\GTC}\\
        v_{\angl{i}} - \gamma^{(n_{\BZC}, n_{\GTC})} &: n_{\GTC} < i \le n_{\GTC} + n_{\BZC}\\
        0 &: n_{\GTC} + n_{\BZC} < i.
      \end{cases}
    \end{align*}
    Using the KKT conditions, we check that
    \[
\widetilde{b}
=
\widehat{b}^{(n_{\BZC}^{\widetilde{\gamma}}, n_{\GTC}^{\widetilde{\gamma}})}
=
\clip(v - 
\widehat{\gamma}^{(n_{\BZC}^{\widetilde{\gamma}}, n_{\GTC}^{\widetilde{\gamma}})} 
  \vone).
    \]
\begin{proof}
  It suffices to prove that $\widetilde{\gamma}
=
\widehat{\gamma}^{(n_{\BZC}^{\widetilde{\gamma}},n_{\GTC}^{\widetilde{\gamma}})}$.
To this end, let $i \in [k-1]$.  If $i \in \mdsetp{\widetilde{\gamma}}$, then $\widetilde{b}_i= v_i - \widetilde{\gamma}$.
If $i \in \upsetp{\widetilde{\gamma}}$, then $\widetilde{b}_i = C$. Otherwise, $\widetilde{b}_i = 0$.
Thus
\begin{equation}
  \label{equation: trial formula for tilde gamma}
  \widetilde{\gamma}=
  \vone' \widetilde{b}
  =
  C \cdot n_{\GTC}^{\widetilde{\gamma}}
  +
  S^{(\mdcarp{\widetilde{\gamma}}, \upcarp{\widetilde{\gamma}})}
  -  \widetilde{\gamma}\cdot \mdcarp{\widetilde{\gamma}}
\end{equation}
Solving for $\widetilde{\gamma}$, we have
\[
  \widetilde{\gamma}
  = \left(C\cdot \upcarp{\widetilde{\gamma}} + 
  S^{(\mdcarp{\widetilde{\gamma}}, \upcarp{\widetilde{\gamma}})}
\right)/( \mdcarp{\widetilde{\gamma}} + 1 )
=
\widehat{\gamma}^{(n_{\BZC}^{\widetilde{\gamma}},n_{\GTC}^{\widetilde{\gamma}})},
  \]
  as desired.
\end{proof}
    
    Now, since $(n_{\BZC}^{\widetilde{\gamma}}, n_{\GTC}^{\widetilde{\gamma}}) \in \llfloor k \rrfloor^2$, 
    to find $\widetilde{b}$
    we can simply check for each $(n_{\BZC},n_{\GTC})  \in \llfloor k \rrfloor^2$ if 
$\widehat{b}^{(n_{\BZC}, n_{\GTC})}$ satisfies the KKT conditions.
However, this naive approach leads to an $O(k^2)$ runtime.

To improve upon the naive approach, define
\begin{equation}
  \Re:=
  \{ 
    ( n_{\BZC}^\gamma, n_{\GTC}^\gamma)
: \gamma \in \mathbb{R}
\}.
\end{equation}
Since $(n_{\BZC}^{\widetilde{\gamma}},n_{\GTC}^{\widetilde{\gamma}}) \in \Re$,
to find $\widetilde{b}$
it suffices to search through $(n_{\BZC},n_{\GTC})  \in \Re$ instead of $\llfloor k \rrfloor^2$.
Towards enumerating all elements of $\Re$, a key result is
that the function 
$\gamma \mapsto
(\mdsetp{\gamma}, \upsetp{\gamma})$
    is locally constant outside of the set of discontinuities:
    \[
      \mathtt{disc} := \{ v_i : i \in [k-1] \} \cup \{ v_i - C: i \in [k-1]\}.
    \]
\begin{proof}
  Let $\gamma_1,\gamma_2,\gamma_3,\gamma_4 \in \mathbb{R}$ satisfy the following: 1) $\gamma_1 < \gamma_2 < \gamma_3 < \gamma_4$, 
  2) $\gamma_1,\gamma_4 \in \mathtt{disc}$, and
  3) $\gamma \not \in \mathtt{disc}$ for all $\gamma \in (\gamma_1,\gamma_4)$. 
  Assume for the sake of contradiction that 
$
(\mdsetp{\gamma_2}, \upsetp{\gamma_2})
\ne 
(\mdsetp{\gamma_3}, \upsetp{\gamma_3})
$.
Then $\mdsetp{\gamma_2} \ne \mdsetp{\gamma_3}$ 
or
$\upsetp{\gamma_2} \ne \upsetp{\gamma_3}$. Consider the case 
$\mdsetp{\gamma_2} \ne \mdsetp{\gamma_3}$.
Then at least one of the sets $\mdsetp{\gamma_2} \setminus \mdsetp{\gamma_3}$ and
$\mdsetp{\gamma_3} \setminus \mdsetp{\gamma_2}$ is nonempty.
Consider the case when $\mdsetp{\gamma_2} \setminus \mdsetp{\gamma_3}$ is nonempty.
Then there exists $i \in [k-1]$ such that $v_i - \gamma_2 \in (0,C)$ but $v_i - \gamma_3 \not \in (0,C)$.
This implies that there exists some $\gamma' \in (\gamma_2,\gamma_3)$ such that $v_i - \gamma' \in \{0,C\}$, or equivalently, $\gamma' \in \{v_i, v_i - C\}$.
Hence, $\gamma' \in \mathtt{disc}$, which is a contradiction.
For the other cases not considered, similar arguments lead to the same contradiction.
\end{proof}

    Thus, as we sweep $\gamma$ from $+\infty$ to $-\infty$, we observe finitely many distinct tuples of sets
    $(\mdsetp{\gamma}, \upsetp{\gamma})$ and their cardinalities
    $(n_{\BZC}^\gamma, n_{\GTC}^\gamma)$. 
    Using the index $t=0,1,2\dots$,
    we keep track of these data
    in the variables
$(\mdsetp{t}, \upsetp{t})$
and
$(\mdcarp{t}, \upcarp{t})$.
    For this proof sketch, we make the assumption that $|\mathtt{disc}| = 2(k-1)$, i.e., no elements are repeated.

    By construction, the maximal element of $\mathtt{disc}$ is $v_{\max}$.
    When $\gamma > v_{\max}$, we check that $\mdcarp{\gamma} = \upcarp{\gamma} = \emptyset$.
    Thus, we put
$
  \mdsetp{0} = \upsetp{0} = \emptyset
  $ and $
(n_{\BZC}^0, n_{\GTC}^0) = (0,0).$

Now, suppose $\gamma$ has swept across $t-1$ points of discontinuity and that 
$\mdsetp{t-1}, \upsetp{t-1}, \mdcarp{t-1}, \upcarp{t-1}$ have all been defined.
Suppose that $\gamma$ crossed a single new point of discontinuity $\gamma' \in \mathtt{disc}$.
In other words, $\gamma'' < \gamma < \gamma'$
where $\gamma''$ is the largest element of $\mathtt{disc}$ such that $\gamma'' < \gamma'$.

By the assumption that no elements of $\mathtt{disc}$ are repeated, exactly one of the two following possibilities is true: 
\begin{align}
  &\mbox{there exists }i \in [k-1] \mbox{ such that }\gamma' = v_i, \label{equation: Entry}\tag{Entry}\\
  &\mbox{there exists } i \in [k-1] \mbox{ such that }\gamma' = v_i - C.  \label{equation: Exit}\tag{Exit}
\end{align}

Under the \cref{equation: Entry} case, the index $i$ gets added to $\mdsetp{t-1}$  while $\upsetp{t-1}$ remains unchanged.
Hence, we have the updates
\begin{align}
  &
  \mdsetp{t} := \mdsetp{\gamma} = \mdsetp{t-1} \cup \{i\},
\quad
\upsetp{t} := \upsetp{\gamma} = \upsetp{t-1}
\\
  &
  \mdcarp{t} := \mdcarp{\gamma} = \mdcarp{t-1} +1,
\quad
\upcarp{t} := \upcarp{\gamma} = \upcarp{t-1}.
\label{equation: Entry-update}
\end{align}
Under the \cref{equation: Exit} case, the index $i$ moves from $\mdsetp{t-1}$ to $\upsetp{t-1}$. Hence, we have the updates
\begin{align}
  &
  \mdsetp{t} := \mdsetp{\gamma} = \mdsetp{t-1} \setminus \{i\},
\quad
\upsetp{t} := \upsetp{\gamma} = \upsetp{t-1} \cup \{i\}
\\
  &
  \mdcarp{t} := \mdcarp{\gamma} = \mdcarp{t-1} -1,
\quad
\upcarp{t} := \upcarp{\gamma} = \upcarp{t-1}+1.
\label{equation: Exit-update}
\end{align}
Thus, 
$
  \{(\mdcarp{t},\upcarp{t})\}_{t=0}^{2(k-1)} = \Re.
$
The case when $\mathtt{disc}$ has repeated elements requires more careful analysis which is done in the full proof.
Now, we have all the ingredients for understanding \cref{algorithm: subproblem generic solver} and its subroutines.



\subsection{A walk through of the solver}
If $v_{\max} \le 0$, then $\widetilde{b} = \vzero$ satisfies the KKT conditions.
\cref{algorithm: subproblem generic solver}-line~3 handles this exceptional case.
Below, we assume $v_{\max} > 0$.

\begin{algorithm}
  \begin{algorithmic}[1]
    \STATE {\bfseries Input:} $v \in \mathbb{R}^{k-1}$
    \STATE
    Let $\angl{1}, \dots, \angl{k-1}$ sort $v$, i.e.,
    $
    v_{\angl{1}} \ge \dots \ge v_{\angl{k-1}}.
    $
    \STATE \textbf{if} $v_{\angl{1}} \le 0$ \textbf{then HALT and output:} 
    $\vzero \in \mathbb{R}^{k-1}$.
    \vspace{0.5em}
    \STATE $n_{\GTC}^0 := 0$,
    $n_{\BZC}^0 := 0$,
    $S^0 := 0$
    \vspace{0.5em}
    \STATE $(\delta_1,\dots, \delta_\ell) \gets \mathtt{get\_up\_dn\_seq}()$
    (Subroutine~\ref{algorithm: construct critical set})
    \vspace{0.5em}
    \FOR{$t = 1 ,\dots, \ell$}
    \STATE 
    $(n_{\BZC}^t,\,
    n_{\GTC}^t,\,
    S^t)
    \gets \mathtt{update\_vars}()
    $
    (Subroutine~\ref{algorithm: update variables}).
    \vspace{0.5em}
    \STATE 
    $\widehat{\gamma}^t := (C\cdot n_{\GTC}^t + S^t)/( n_{\BZC}^t + 1 )$
    \vspace{0.5em}
    \IF{$\mathtt{KKT\_cond}()$ (Subroutine~\ref{algorithm: check KKT condition}) returns true}
    \STATE 
    \textbf{HALT and output}: $\widehat{b}^t \in \mathbb{R}^{k-1}$ where
    \vspace{-0.5em}
    \[
    \widehat{b}_{\angl{i}}^t
      :=
      \begin{cases}
        C &: i \le n_{\GTC}^t\\
        v_{\angl{i}} - \gamma^t &: n_{\GTC}^t < i \le n_{\GTC}^t + n_{\BZC}^t\\
        0 &: n_{\GTC}^t + n_{\BZC}^t < i.
      \end{cases}
    \]
    \vspace{-0.5em}
    \ENDIF
  \ENDFOR
  \caption{$\mathtt{solve\_subproblem}(v,C)$}
  \label{algorithm: subproblem generic solver}
  \end{algorithmic}
\end{algorithm}

\cref{algorithm: subproblem generic solver}-line~4 initializes
the state variables $n_{\BZC}^t$ and $n_{\GTC}^t$ as discussed in the last section.
The variable $S^t$ is also initialized and will be updated to maintain $S^t = S^{(\mdcarp{t}, \upcarp{t})}$ where the latter is defined at \cref{roadmap: S superscript t}.

\cref{algorithm: subproblem generic solver}-line~5 calls Subroutine~\ref{algorithm: construct critical set} to construct the $\mathtt{vals}$ ordered set, which is similar to the set of discontinuities $\mathtt{disc}$, but different in three ways:
1) $\mathtt{vals}$ consists of tuples $(\gamma', \delta')$
where $\gamma' \in \mathtt{disc}$ and 
$\delta' \in \{\mathtt{up},\mathtt{dn}\}$ is a decision variable indicating whether $\gamma'$ satisfies the \cref{equation: Entry} or the \cref{equation: Exit} condition, 
2) $\mathtt{vals}$ is sorted so that the $\gamma'$s are in descending order,
and 3) only positive values of $\mathtt{disc}$ are needed.
The justification for the third difference is because we prove that \cref{algorithm: subproblem generic solver} always halts before reaching the negative values of $\mathtt{disc}$.
Subroutine~\ref{algorithm: construct critical set} returns the list of symbols $(\delta_1,\dots, \delta_\ell)$ consistent with the ordering.

\begin{subroutine}
  \begin{algorithmic}[1]
    \STATE $\mathtt{vals} \gets \{(v_i, \mathtt{dn}): v_i > 0, \, i = 1,\dots, k-1 \} \,\,\cup\,\, \{(v_i- C, \mathtt{up}) : v_i > C , \, i = 1,\dots, k-1\}$ as a multiset, where elements may be repeated.
    \STATE Order the set 
    $
      \mathtt{vals} =
      \{(\gamma_1,\delta_1),\dots, (\gamma_\ell, \delta_\ell)\}
    $
    such that
    $\gamma_1\ge \dots \ge \gamma_{\ell}$, $\ell= |\mathtt{vals}|$, and for all $j_1, j_2 \in [\ell]$ such that $j_1<j_2$ and $\gamma_{j_1} = \gamma_{j_2}$, we have $\delta_{j_1} = \mathtt{dn}$ implies $\delta_{j_2} = \mathtt{dn}$.

    Note that by construction, for each $t \in [\ell]$,  there exists $i \in [k-1]$ such that $\gamma_t = v_i$ or $\gamma_t = v_i - C$.
    \STATE \textbf{Output:} sequence $(\delta_1,\dots, \delta_\ell)$ whose elements are retrieved in order from left to right.
    \caption{$\mathtt{get\_up\_dn\_seq}$
    \hspace{1em} \emph{Note: all variables from \cref{algorithm: subproblem generic solver} are assumed to be visible here.}
    }
    \label{algorithm: construct critical set}
  \end{algorithmic}
\end{subroutine}

In the ``for'' loop, 
\cref{algorithm: subproblem generic solver}-line~7 calls Subroutine~\ref{algorithm: update variables} which updates the variables $\mdcarp{t}, \upcarp{t}$ using \cref{equation: Entry-update} or \cref{equation: Exit-update}, depending on $\delta_t$.
The variable $S^t$ is updated accordingly so that $S^t = S^{(\mdcarp{t}, \upcarp{t})}$.

\begin{subroutine}
  \begin{algorithmic}[1]
    \IF{$\delta_t = \mathtt{up}$}
    \STATE $n_{\GTC}^t := n_{\GTC}^{t-1}+1$,\hspace{1em}
    $n_{\BZC}^t := n_{\BZC}^{t-1} -1$
    \vspace{0.5em}
    \STATE $S^t := S^{t-1} - v_{\angl{n_{\GTC}^{t-1}}}$
    \ELSE
    \STATE $n_{\BZC}^t := n_{\BZC}^{t-1} + 1$, \hspace{1em}
    $n_{\GTC}^t := n_{\GTC}^{t-1}$.
    \vspace{0.5em}
    \STATE $S^t := S^{t-1} + v_{\angl{n_{\GTC}^{t} + n_{\BZC}^{t} }}$
    \ENDIF
    \STATE \textbf{Output:} 
    $(n_{\BZC}^t,\,
    n_{\GTC}^t,\,
    S^t)$
    \caption{$\mathtt{update\_vars}$
    \hspace{1em} \emph{Note: all variables from \cref{algorithm: subproblem generic solver} are assumed to be visible here.}
    }
    \label{algorithm: update variables}
  \end{algorithmic}
\end{subroutine}

We skip to 
\cref{algorithm: subproblem generic solver}-line~9 which constructs the putative solution $\widehat{b}^t$.
Observe that $\widehat{b}^t = \widehat{b}^{(\mdcarp{t}, \upcarp{t})}$ where the latter is defined in the previous section.

Going back one line, \cref{algorithm: subproblem generic solver}-line~8 calls Subroutine~\ref{algorithm: check KKT condition} which checks if the putative solution $\widehat{b}^t$ satisfies the KKT conditions.
We note that this can be done \emph{before} the putative solution is constructed.
\begin{subroutine}
  \begin{algorithmic}[1]
    \STATE 
    $\mathtt{kkt\_cond} \gets true$
    \IF{$n_{\GTC}^t > 0$}
    \STATE $\mathtt{kkt\_cond} \gets \mathtt{kkt\_cond}\land
    \left(C+\widehat{\gamma}^t \le v_{\angl{n_{\GTC}^t}}\right)
    $
    \\
    \emph{Note: $\land$ denotes the logical ``and''.}
    \ENDIF

    \IF{$n_{\BZC}^t > 0$}
    \STATE $\mathtt{kkt\_cond} \gets \mathtt{kkt\_cond}\land
    \left(
      v_{\angl{n_{\GTC}^t+1}} \le C+\widehat{\gamma}^t
    \right)
    $
    \STATE $\mathtt{kkt\_cond} \gets \mathtt{kkt\_cond}\land
    \left(
      \widehat{\gamma}^t \le v_{\angl{n_{\GTC}^t + n_{\BZC}^t}}
  \right)
    $
    \ENDIF
    \IF{$n_{\LTZ}^t := k - 1 - n_{\GTC}^t - n_{\BZC}^t > 0$}
    \STATE $\mathtt{kkt\_cond} \gets \mathtt{kkt\_cond}\land
    \left(v_{\angl{n_{\GTC}^t + n_{\BZC}^t + 1}} \le \widehat{\gamma}^t \right)
    $
    \ENDIF
    \STATE \textbf{Output:} $\mathtt{kkt\_cond}$
    \caption{$\mathtt{KKT\_cond}$ \hspace{1em} \emph{Note: all variables from \cref{algorithm: subproblem generic solver} are assumed to be visible here.}
    }
    \label{algorithm: check KKT condition}
  \end{algorithmic}
\end{subroutine}

For the runtime analysis, we note that Subroutines~\ref{algorithm: check KKT condition} 
and \ref{algorithm: update variables} both use
$O(1)$ FLOPs without dependency on $k$.
The main ``for'' loop of \cref{algorithm: subproblem generic solver} (line 6 through 11) has $O(\ell)$ runtime where $\ell \le 2(k-1)$.
Thus, the bottlenecks are \cref{algorithm: subproblem generic solver}-line 2 and 5 which sort lists of length at most $k-1$ and $2(k-1)$, respectively. Thus, both lines run in $O(k \log k)$ time.

\section{Experiments}

LIBLINEAR is one of the state-of-the-art solver for linear SVMs \cite{fan2008liblinear}.
However, as of the latest version 2.42, the linear Weston-Watkins SVM is not supported.
We implemented our linear WW-SVM subproblem solver, \emph{Walrus} (\cref{algorithm: subproblem generic solver}), along with the BCD \cref{algorithm: BCD} as an extension to LIBLINEAR.
The solver and code for generating the figures are available\footnote{See Section~\ref{section: code availability}.}.

We compare our implementation to \emph{Shark} \cite{igel2008shark},
which solves the dual subproblem
  \cref{equation: WW-SVM dual subproblem}
using a form of greedy coordinate descent.
For comparisons, we reimplemented Shark's solver also as a LIBLINEAR extension.
When clear from the context, we use the terms ``Walrus'' and ``Shark'' when referring to either the subproblem solver or the overall BCD algorithm.

We perform benchmark experiments on 8 datasets from 
``LIBSVM Data: Classification (Multi-class)\footnote{See Section~\ref{section: datasets}.}'' spanning a range of $k$ from $3$ to $1000$.
See \cref{table: data sets}.
\begin{table}[t]
\caption{Data sets used. Variables $k,\,n$ and $d$ are, respectively, the number of classes, training samples, and features.}
\label{table: data sets}
\vskip 0.15in
\begin{center}
\begin{small}
\begin{sc}
\begin{tabular}{lrrr}
\toprule
Data set & $k$ & $n$ & $d$ \\
\midrule
dna      & 3     & 2,000   & 180    \\
satimage & 6     & 4,435   & 36     \\
mnist    & 10    & 60,000  & 780    \\
news20   & 20    & 15,935  & 62,061 \\
letter   & 26    & 15,000  & 16     \\
rcv1     & 53    & 15,564  & 47,236 \\
sector   & 105   & 6,412   & 55,197 \\
aloi     & 1,000 & 81,000 & 128   \\
\bottomrule
\end{tabular}
\end{sc}
\end{small}
\end{center}
\vskip -0.1in
\end{table}

In all of our experiments, Walrus and Shark perform identically in terms of testing accuracy. We report the accuracies in Section~\ref{section: supp mat - accuracies}. Below, we will only discuss runtime.

For measuring the runtime, we start the timer after the data sets have been loaded into memory and before the state variables $\bm{\beta}$ and $\mathbf{w}$ have been allocated.
The primal objective is the value of \cref{equation: WW-SVM primal optimization} at the current $\mathbf{w}$ and the dual objective is $-1$ times the value of
\cref{equation: WW-SVM dual optimization reparametrized} at the current $\bm{\beta}$.
The duality gap is the primal minus the dual objective.
The objective values and duality gaps are measured after each \emph{outer} iteration, during which the timer is paused.

For solving the subproblem, Walrus is guaranteed to return the minimizer in $O(k\log k)$ time.
On the other hand, to the best of our knowledge, Shark does not have such guarantee.
Furthermore, Shark uses a doubly-nested for loop, each of which has length $O(k)$, yielding a worst-case runtime of $O(k^2)$.
For these reasons, we hypothesize that Walrus scales better with larger $k$.

As exploratory analysis, we ran
Walrus and Shark 
on the \textsc{satimage}  and \textsc{sector} data sets\footnote{
The regularizers are set to the corresponding values from Table 5 of the supplementary material of \citet{dogan2016unified} chosen by cross-validation.
}, which has $6$ and $105$ classes, respectively.
The results, shown in \cref{figure: trajectories}, support our hypothesis:
Walrus and Shark are equally fast for \textsc{satimage} while
Walrus is faster for \textsc{sector}.

\begin{figure}
  \includegraphics[width=1\linewidth]{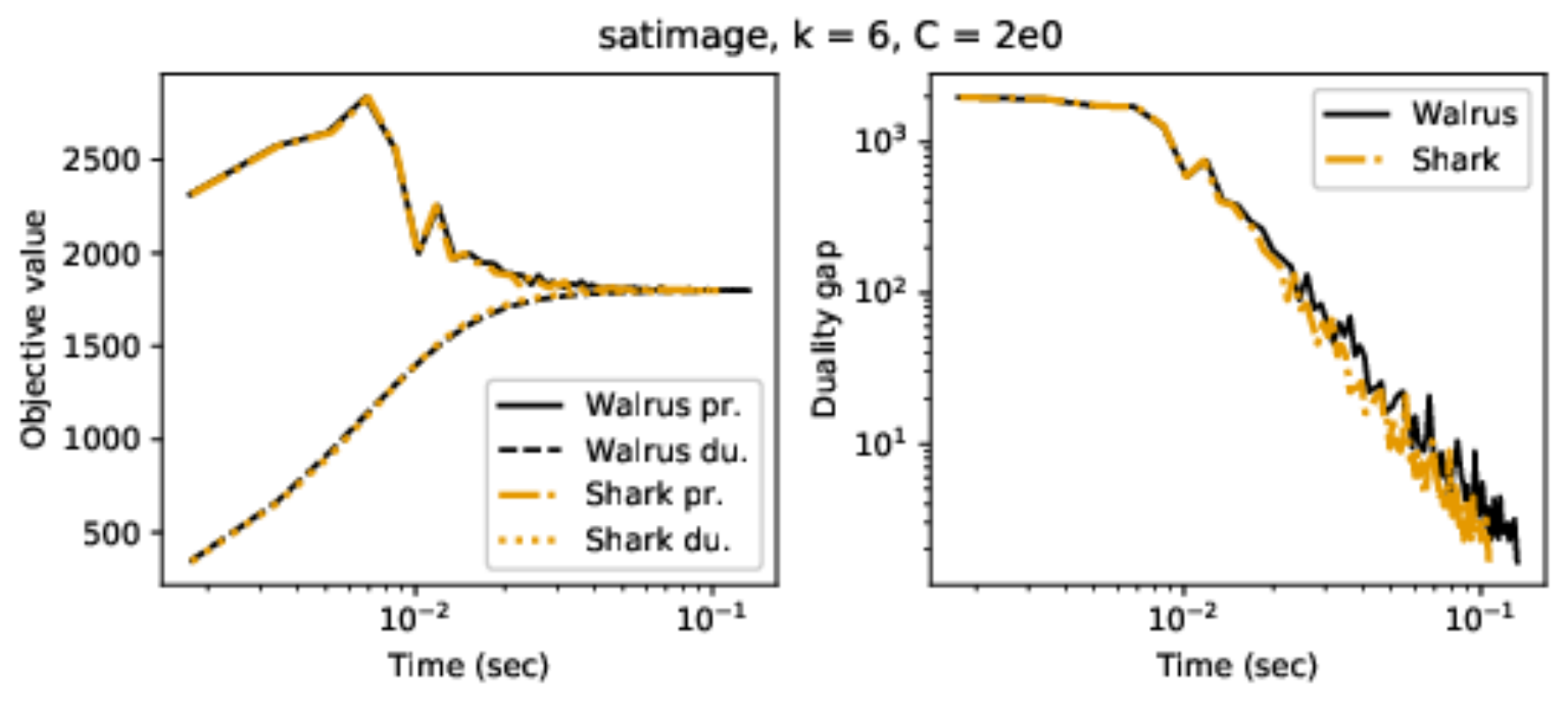}
  \includegraphics[width=1\linewidth]{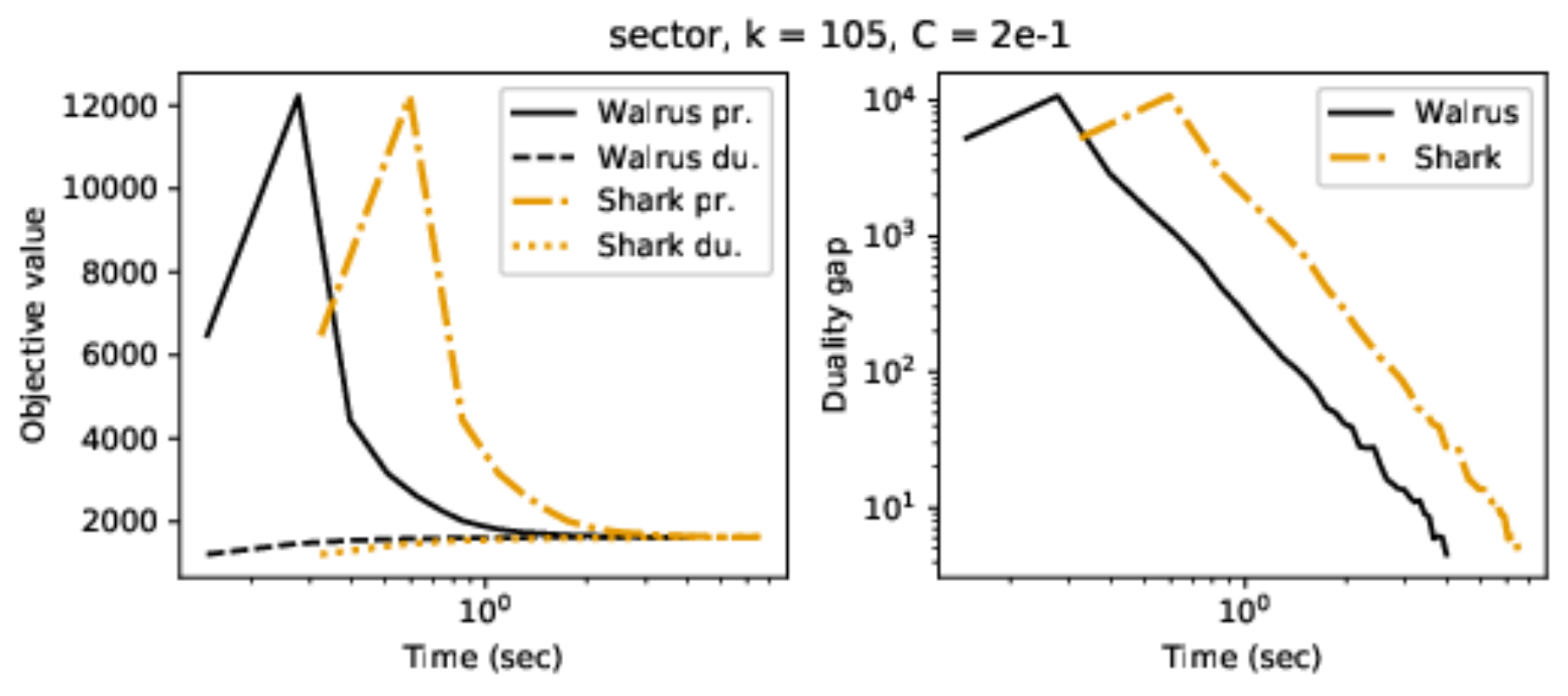}
  \caption{Runtime comparison of Walrus and Shark.
    Abbreviations: pr.\ = primal and du.\ = dual.
    The X-axes show time elapsed.
  }
  \label{figure: trajectories}
\end{figure}

We test our hypothesis on a larger scale by running Walrus and Shark on the datasets in \cref{table: data sets} over the grid of hyperparameters $C \in \{2^{-6},2^{-5},\dots,2^2, 2^{3}\}$.
The results are shown in \cref{figure: time comparison} where
each dot represents a triplet $(\mbox{\textsc{data set}}, C, \delta)$ where $\delta$ is a quantity we refer to as the \emph{duality gap decay}.
The Y-axis shows the comparative metric of runtime
$\mathtt{ET}_{\Walrus}^\delta/ \mathtt{ET}_{\Shark}^{\delta}$ to be defined next.

Consider a single run of Walrus on a fixed data set with a given hyperparameter $C$. Let $\mathtt{DG}_{\Walrus}^t$ denote the \underline{d}uality \underline{g}ap achieved by Walrus at the end of the $t$-th outer iteration.
Let $\delta \in (0,1)$.
Define $\mathtt{ET}_{\Walrus}^\delta$ to be the \underline{e}lapsed \underline{t}ime at the end of the $t$-th iteration where $t$ is minimal such that $\mathtt{DG}_{\Walrus}^t \le \delta \cdot \mathtt{DG}_{\Walrus}^1$.
Define $\mathtt{DG}_{\Shark}^t$ and $\mathtt{ET}_{\Shark}^\delta$ similarly.
In all experiments $\mathtt{DG}_{\Walrus}^1 / \mathtt{DG}_{\Shark}^1 \in [0.99999,1.00001]$.
Thus, the ratio  $\mathtt{ET}_{\Walrus}^\delta/ \mathtt{ET}_{\Shark}^{\delta}$ measures how much faster Shark is relative to Walrus.

From \cref{figure: time comparison}, it is evident that in general Walrus converges faster on data sets with larger number of classes.
Not only does Walrus beat Shark for large $k$, but it also seems to not do much worse for small $k$. In fact Walrus seems to be at least as fast as Shark for all datasets except \textsc{satimage}.

The absolute amount of time saved by Walrus is often more significant on datasets with larger number of classes.
To illustrate this, we let $C=1$ and compare the times for the duality gap to decay by a factor of $0.01$.
On the data set \textsc{satimage} with $k=6$, Walrus and Shark take $0.0476$ and $0.0408$ seconds, respectively.
On the data set \textsc{aloi} with $k=1000$, Walrus and Shark take $188$ and $393$ seconds, respectively.

We remark that \cref{figure: time comparison} also suggests that Walrus tends to be faster during early iterations but can be slower at late stages of the optimization.
To explain this phenomenon, we note that Shark solves the subproblem using an iterative descent algorithm and is set to stop when the KKT violations fall below a hard-coded threshold.
When close to optimality, Shark takes fewer descent steps, and hence less time, to reach the stopping condition on the subproblems.
On the other hand, Walrus takes the same amount of time regardless of proximity to optimality.

For the purpose of grid search, a high degree of optimality is not needed.
In Section~\ref{section: supp mat - accuracies}, we provide empirical evidence that stopping early versus late does not change the result of grid search-based hyperparameter tuning.
Specifically, \cref{table: combined table} shows that running the solvers until $\delta \approx 0.01$ or until $\delta \approx 0.001$ does not change the cross-validation outcomes.

Finally, the optimization \cref{equation: dual subproblem generic} is a convex quadratic program and hence can be solved using general-purpose solvers \cite{voglis2004boxcqp}.
However, we find that Walrus, being specifically tailored to the optimization \cref{equation: dual subproblem generic}, is orders of magnitude faster. See \cref{table: benchmarking subproblem solver varying k,table: benchmarking subproblem solver varying C} in the Appendix.

\begin{figure}
  \includegraphics[width=1\linewidth]{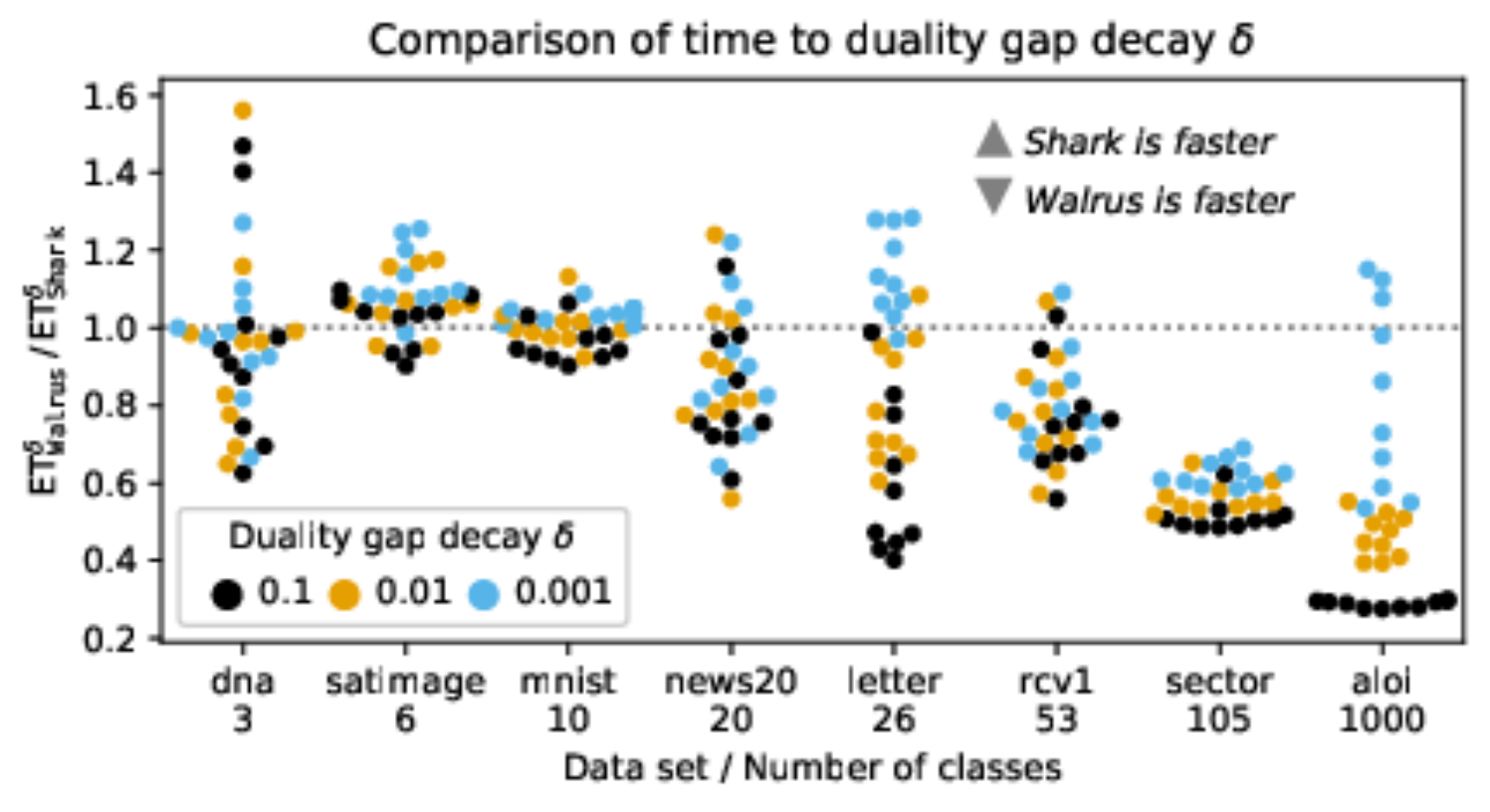}
  \caption{
X-coordinates jittered for better visualization.
  }
  \label{figure: time comparison}
\end{figure}

\section{Discussions and future works}
We presented an algorithm called Walrus for exactly solving the WW-subproblem which scales with the number of classes.
We implemented Walrus in the LIBLINEAR framework and demonstrated empirically that BCD using Walrus is significantly faster than state-of-the-art linear WW-SVM solver Shark on datasets with a large number of classes, and comparable to Shark for small number of classes.

One possible direction for future research is whether Walrus can improve \emph{kernel} WW-SVM solver.
Another direction is lower-bounding time complexity of solving the WW-subproblem \cref{equation: dual subproblem generic}.

  \section*{Acknowledgements}
The authors were supported in part by the National Science Foundation under awards 1838179 and 2008074, by the Department of Defense, Defense Threat Reduction Agency under award HDTRA1-20-2-0002, and by the Michigan Institute for Data Science.

\bibliography{references}
\bibliographystyle{icml2021}


\onecolumn

\newpage

\appendix


\icmltitle{Supplementary materials for\texorpdfstring{\\}{} An Exact Solver for the Weston-Watkins SVM Subproblem}
\begin{appendices}
\secttoc

\fakesection{Appendix}


\subsection{Regarding offsets}
\label{section: offsets}

In this section, we review the literature on SVMs in particular with regard to offsets.

For binary kernel SVMs, \citet{steinwart2011training} demonstrates that kernel SVMs without offset achieve comparable classification accuracy as kernel SVMs with offset.
Furthermore, they propose algorithms that solve kernel SVMs without offset that are significantly faster than solvers for kernel SVMs with offset.

For binary linear SVMs,
\citet{hsieh2008dual} introduced coordinate descent for the dual problem associated to linear SVMs without offsets, or with the bias term included in the $\mathbf{w}$ term.
\citet{chiu2020two} studied whether the method of \citet{hsieh2008dual} can be extended to allow offsets, but found evidence that the answer is negative.
For multiclass linear SVMs,
\citet{keerthi2008sequential} studied block coordinate descent for the CS-SVM and WW-SVM, both without offsets.
We are not aware of a multiclass analogue to \citet{chiu2020two} although the situation should be similar.

The previous paragraph discussed coordinate descent in relation to the offset.
Including the offset presents challenges to primal methods as well.
In Section 6 of \citet{shalev2011pegasos}, the authors argue that including an unregularized offset term in the primal objective leads to slower convergence guarantee.
Furthermore, \citet{shalev2011pegasos} observed that including an unregularized offset did not significantly change the classification accuracy.

The original Crammer-Singer (CS) SVM was proposed without offsets \cite{crammer2001algorithmic}.
In Section VI of \cite{hsu2002comparison}, the authors show the CS-SVM with offsets do \emph{not} perform better than CS-SVM without offsets.
Furthermore, CS-SVM with offsets requires twice as many iterations to converge than without.

\subsection{Proof of \texorpdfstring{\cref{proposition: reparametrized dual problem}}{reparametrization proposition}}
  Below, let $i \in [n]$ be arbitrary.
First, we note that 
$
  -\mar' =
  \begin{bmatrix}
    -\vone' \\
    \mathbf{I}_{k-1}
  \end{bmatrix}
$
and so
\begin{equation}
  \label{equation: proof of reparametrized dual problem 1}
\mar' \beta_i
=
  \begin{bmatrix}
    -\vone' \beta_i \\
    \beta_i
  \end{bmatrix}.
\end{equation}
Now, let $j \in [k]$, we have
by
  \cref{equation: beta alpha relation via transposition}
  that
\begin{equation}
  [\alpha_{i}]_j =
  [-\tsp_{y_i}  \mar' \beta_i]_j
  =
  [-  \mar' \beta_i]_{\tsp_{y_i}(j)}.
\end{equation}
Note that if $j \ne y_i$, then $\tsp_{y_i}(j) \ne 1$ and so
$
[\alpha_i]_j=
  [-  \mar' \beta_i]_{\tsp_{y_i}(j)}
  =[\beta_{i}]_{\tsp_{y_i}(j)-1} \in [0,C]$.
  On the other hand, if $j = y_i$, then $\tsp_{y_i}(y_i) = 1$ and $[\alpha_i]_{y_i}
  = [-  \mar' \beta_i]_{1}
  = -\vone' \beta_i
  =
  -\sum_{t \in [k-1]} [\beta_i]_t
  =
  -\sum_{t \in [k]: t \ne y_i}
  [\beta_{i}]_{\tsp_{y_i}(t)-1}
  =
  -\sum_{t \in [k]: t \ne y_i}
  [\alpha_{i}]_t
  $.
  Thus, $\bm{\alpha} \in \mathcal{F}$.
  This proves that $\Psi(\mathcal{G}) \subseteq \mathcal{F}$.

  Next, let us define another map $\Xi: \mathcal{F} \to \mathbb{R}^{(k-1) \times n}$ as follows:
  For each $\bm{\alpha} \in \mathcal{F}$, define $\bm{\beta} := \Xi(\bm{\alpha})$ block-wise by
\[
  \beta_i := \mathtt{proj}_{2:k}(\tsp_{y_i}\alpha_i) \in \mathbb{R}^{k-1}
\]
where 
\[
  \mathtt{proj}_{2:k}
  =
  \begin{bmatrix}
    \vzero &
    \mathbf{I}_{k-1}
  \end{bmatrix}
  \in \mathbb{R}^{(k-1) \times k}.
\]
By construction, we have for each $j \in [k-1]$ that
$[\beta_i]_j
=
[\tsp_{y_i}\alpha_i]_{j+1}
=
[\tsp_{y_i}\alpha_i]_{j+1}
=
[\alpha_i]_{\tsp_{y_i}(j+1)}
$
Since $j+1 \ne 1$ for any $j \in [k-1]$, we have that $\tsp_{y_i}(j+1) \ne y_i$ for any $j \in [k-1]$.
Thus, $[\beta_i]_j = 
[\alpha_i]_{\tsp_{y_i}(j+1)} \in [0,C]$. This proves that $\Xi(\mathcal{F}) \subseteq \mathcal{G}$.

Next, we prove that for all $\bm{\alpha} \in \mathcal{F}$ and $\bm{\beta} \in \mathcal{G}$, we have $\Xi(\Psi(\bm{\beta})) = \bm{\beta}$ and 
$\Psi(\Xi(\bm{\alpha})) = \bm{\alpha}$.

By construction, the $i$-th block of $\Xi(\Psi(\bm{\beta}))$ is given by
\begin{align*}
  \mathtt{proj}_{2:k}(\tsp_{y_i} (-\tsp_{y_i} \mar' \beta_i))
  &=
  -\mathtt{proj}_{2:k}(\tsp_{y_i} \tsp_{y_i} \mar' \beta_i)
  \\
  &=
  -\mathtt{proj}_{2:k}( \mar' \beta_i)
  \\
  &=
  -
  \begin{bmatrix}
    \vzero &
    \mathbf{I}_{k-1}
  \end{bmatrix}
  \begin{bmatrix}
    \vone' \\
    -\mathbf{I}_{k-1}
  \end{bmatrix}
  \beta_i
  \\
  &=
    \mathbf{I}_{k-1}
  \beta_i
  =
  \beta_i.
\end{align*}
For the second equality, we used the fact that $\tsp_y^2 = \mathbf{I}$ for all $y \in [k]$.
Thus, $\Xi(\Psi(\bm{\beta})) = \bm{\beta}$.

Next, note that the $i$-th block of 
$\Psi(\Xi(\bm{\alpha}))$ is, by construciton,
\begin{equation}
  \label{equation: proof of reparametrized dual problem 2}
-\tsp_{y_i} \mar'
  \mathtt{proj}_{2:k}(\tsp_{y_i} \alpha_i)
  =
-\tsp_{y_i}
\pi'
  \begin{bmatrix}
    \vzero &
    \mathbf{I}_{k-1}
  \end{bmatrix}
  \tsp_{y_i} \alpha_i
  =
-\tsp_{y_i}
  \begin{bmatrix}
    \vzero &
\pi'
  \end{bmatrix}
  \tsp_{y_i} \alpha_i
\end{equation}
Recall that $\pi ' = \begin{bmatrix}
  \vone'\\
  -\mathbf{I}_{k-1}
\end{bmatrix}$ and so 
$
  \begin{bmatrix}
    \vzero &
\pi'
  \end{bmatrix}
  =
  \begin{bmatrix}
    0 & \vone' \\
    \vzero &  -\mathbf{I}_{k-1}
  \end{bmatrix}
  $. Therefore,
\[
  \left[
  \begin{bmatrix}
    \vzero &
\pi'
  \end{bmatrix}
  \tsp_{y_i} \alpha_i
  \right]_1
  =
  \sum_{j=2}^k
  [
  \tsp_{y_i} \alpha_i
  ]
  =
  \sum_{j \in [k]: j\ne y_i}
  [
  \alpha_i
  ]_j
  =
  -[\alpha_i]_{y_i}
  =
  -
  [
  \tsp_{y_i} \alpha_i
  ]_1
\]
and, for $j = 2,\dots, k$,
\[
  \left[
  \begin{bmatrix}
    \vzero &
\pi'
  \end{bmatrix}
  \tsp_{y_i} \alpha_i
  \right]_j
  =
  -
  [
  \tsp_{y_i} \alpha_i
  ]_j.
\]
Hence, we have just shown that $
  \begin{bmatrix}
    \vzero &
\pi'
  \end{bmatrix}
  \tsp_{y_i} \alpha_i
  =
  -
  \tsp_{y_i} \alpha_i$.
  Continuing from 
  \cref{equation: proof of reparametrized dual problem 2}, we have 
\[
-\tsp_{y_i} \mar'
  \mathtt{proj}_{2:k}(\tsp_{y_i} \alpha_i)
  =
  -\tsp_{y_i}(-\tsp_{y_i} \alpha_i)
  =
  \tsp_{y_i}\tsp_{y_i} \alpha_i
  =\alpha_i.
\]
This proves that 
$\Psi(\Xi(\bm{\alpha})) = \bm{\alpha}$.
Thus, we have shown that $\Psi$ and $\Xi$ are inverses of one another. This proves that $\Psi$ is a bijection.

Finally, we prove that
  \[f(\Psi(\bm{\beta})) = g(\bm{\beta}).\]
  Recall that 
\[
  f(\bm{\alpha}):=
  \frac{1}{2}
  \sum_{i,s\in [n]}
  x_s'x_i\alpha_i'\alpha_s
  -
  \sum_{i\in[k]}
  \sum_{\substack{j \in [k]:\\ j \ne y_i}}
  \alpha_{ij}
\]
Thus,
\[
  \alpha_i' \alpha_s
  =
  (-\tsp_{y_i}\mar' \beta_i )'
  (-\tsp_{y_s} \mar' \beta_s )
  =
  \beta_i' \mar \tsp_{y_i} 
  \tsp_{y_s}' \mar'
  \beta_s
\]
On the other hand, \cref{equation: beta alpha relation via transposition} implies that
$\tsp_{y_i} \alpha_i = - \mar'\beta_i$.
Hence
\begin{equation*}
            \sum_{j\in[k] \setminus \{y_i\}}
            \alpha_{ij}
            =
            \sum_{j \in [k]: j \ne 1}
            [\alpha_{i}]_{\tsp_{y_i}(j)}
            =
            \sum_{j \in [k]: j \ne 1}
            [\tsp_{y_i}\alpha_{i}]_{j}
            =
            \sum_{j \in [k]: j \ne 1}
            [-\mar' \beta_i]_{j}
            =
            \sum_{j \in [k-1]}
            [\beta_i]_j
            =
  \vone' \beta_i.
          \end{equation*}

          Thus,
\[
  f(\bm{\alpha}):=
  \frac{1}{2}
  \sum_{i,s\in [n]}
  x_s'x_i\alpha_i'\alpha_s
  -
  \sum_{i\in[k]}
  \sum_{\substack{j \in [k]:\\ j \ne y_i}}
  \alpha_{ij}
  =
  \frac{1}{2}
  \sum_{i,s\in [n]}
  x_s'x_i
  \beta_i' \mar \tsp_{y_i} 
  \tsp_{y_s}' \mar'
  \beta_s
  -
  \sum_{i\in[k]}
  \vone' \beta_i
  =
  g(\bm{\beta})
\]
as desired.
Finally, we note that $\tsp_y = \tsp_y'$ for all $y\in [k]$.
This concludes the proof of 
  \cref{proposition: reparametrized dual problem}.
\hfill \qedsymbol

\subsection{Proof of \texorpdfstring{\cref{proposition: subproblem reduction to generic form}}{reduction to generic form}}


We prove the following
lemma which essentially unpacks the succinct \cref{proposition: subproblem reduction to generic form}:
\begin{lemma}
  \label{lemma: dual subproblem}
  Recall the situation of
  \cref{corollary: equivalence of dual subproblems}:
  Let $\bm{\beta} \in \mathcal{G}$ and $i \in [n]$. Let $\bm{\alpha} = \Psi(\bm{\beta})$.
  Consider
\begin{equation}
  \label{equation: WW-SVM dual subproblem - reparametrized - appendix}
\min_{\widehat{\bm{\beta}} \in \mathcal{G}}
  \,\, g(\widehat{\bm{\beta}}) \,\, \mbox{such that} \,\, 
  \, \widehat{\beta}_s = \beta_s,\, \forall s \in [n] \setminus \{i\}.
\end{equation}
Let $\mathbf{w}$ be as in \cref{equation: primal-dual variable relation},
  i.e., 
  $\mathbf{w}
  =
  -\sum_{i \in [n]} x_i \alpha_i'.$
  Then a solution to \cref{equation: WW-SVM dual subproblem - reparametrized - appendix} is given by
  $[\beta_1,\dots, \beta_{i-1}, \widetilde{\beta_i},\beta_{i+1},\dots, \beta_n]$ where $\widetilde{\beta_i}$ is a minimizer of
  \[
    \min_{\widehat{\beta}_i \in \mathbb{R}^{k-1}}
    \,\,
  \frac{1}{2} 
  \widehat{\beta}_i'
      \bm{\Theta}
  \widehat{\beta}_i
               -
    \widehat{\beta}_i'
    \left(
      (
     \vone
    -
    \mar
    \tsp_{y_i}
    \mathbf{w}'
x_i
)/
  \| x_i\|_2^2
      +
      \bm{\Theta}
    \beta_i
  \right)
  \,\, \mbox{such that} \,\,
  0 \le \widehat{\beta}_i \le C.
  \]
Furthermore, the above optimization has a unique minimizer which is equal to the minimizer of \cref{equation: dual subproblem generic}
where
    \[v := (\vone 
    - \ico_{y_i} \mar \mathbf{w}' x_i
    + \bm{\Theta} \beta_i \|x_i\|_2^2 
  )/\|x_i\|_2^2\]
  and $\mathbf{w}$ is as in \cref{equation: primal-dual variable relation}.
\end{lemma}

\begin{proof}
  First, we prove a simple identity:
\begin{equation}
  \label{equation: pi pi transpose is Theta}
  \mar
  \mar'
  =
\begin{bmatrix}
  \vone&
  -\mathbf{I}_{k-1}
\end{bmatrix}
\begin{bmatrix}
  \vone'\\
  -\mathbf{I}_{k-1}
\end{bmatrix}
=
\mathbf{I} + \mathbf{O} = \bm{\Theta}.
\end{equation}
  Next, recall that by definition, we have
  \[
    g(\bm{\beta}):=
  \left(
  \frac{1}{2}
\sum_{s,t \in [n]}
x_s ' x_t
  \beta_t' 
  \mar
  \tsp_{y_t}
    \tsp_{y_s}
    \mar'
    \beta_s
    \right)
  -
  \left(
\sum_{s \in [n]} \vone' \beta_s
\right).
  \]
  Let us group the terms of $g(\bm{\beta})$ that depends on $\beta_i$:
\begin{align*}
  g(\bm{\beta})
  &=
  \frac{1}{2} 
  x_i' x_i\beta_i'
  \mar \tsp_{y_i} \tsp_{y_i} \mar'
  \beta_i
    \\
               &\quad+
  \frac{1}{2}
\sum_{s \in [n]: s \ne i}
x_s ' x_i
  \beta_i' 
  \mar \tsp_{y_i} \tsp_{y_s} \mar'
    \beta_s
    \\
               &\quad+
  \frac{1}{2}
\sum_{t \in [n]: t \ne i}
x_i ' x_t
  \beta_t' 
  \mar \tsp_{y_t} \tsp_{y_i} \mar'
  \beta_i
    \\
               &\quad+
  \frac{1}{2}
\sum_{s,t \in [n]}
x_s ' x_t
  \beta_t'
  \mar \tsp_{y_t} \tsp_{y_s} \mar'
    \beta_s - \sum_{s \in [n]} \vone' \beta_s \\
  &=
  \frac{1}{2} 
  x_i' x_i\beta_i'  \bm{\Theta}  \beta_i \qquad \because 
  \tsp_{y_i}^2 = \mathbf{I} \mbox{ and }
  \cref{equation: pi pi transpose is Theta}
    \\
               &\quad+
\sum_{s \in [n]: s \ne i}
x_s ' x_i
  \beta_i' 
  \mar \tsp_{y_i} \tsp_{y_s} \mar'
  \beta_s
    \\
               & \quad - 
               \vone' \beta_i
    \\
               &\quad+
               \underbrace{
  \frac{1}{2}
\sum_{s,t \in [n]}
x_s ' x_t
  \beta_t'
  \mar \tsp_{y_t} \tsp_{y_s} \mar'
  \beta_s - \sum_{s \in [n]: s \ne i} \vone' \beta_s
  }_{=:C_i}
\end{align*}
where $C_i$ is a scalar quantity which does not depend on $\beta_i$.
Thus, plugging in $\widehat{\bm{\beta}}$, we have
\begin{equation}
  \label{equation: g restricted to a single block}
  g(\widehat{\bm{\beta}})
  =
  \frac{1}{2} 
  \| x_i\|_2^2
  \widehat{\beta}_i'
  \bm{\Theta} 
  \widehat{\beta}_i
               +
\sum_{s \in [n]: s \ne i}
x_s ' x_i
\widehat{\beta}_i' 
  \mar \tsp_{y_i} \tsp_{y_s} \mar'
    \beta_s
               - 
               \vone' \widehat{\beta}_i
               +C_i.
\end{equation}
Furthermore, 
\begin{align*}
\sum_{s \in [n]: s \ne i}
x_s ' x_i
\widehat{\beta}_i'
  \mar \tsp_{y_i} \tsp_{y_s} \mar'
\beta_s
&=
\sum_{s \in [n]: s \ne i}
\widehat{\beta}_i'
  \mar \tsp_{y_i} \tsp_{y_s} \mar'
\beta_s
x_s ' x_i
\\
    &=
    \widehat{\beta}_i'
    \mar \tsp_{y_i}
    \left(
\sum_{s \in [n]: s \ne i}
\tsp_{y_s} \mar'
    \beta_s
    x_s '\right) x_i
    \\
    &=
    \widehat{\beta}_i'
    \mar \tsp_{y_i}
    \left(
-
\tsp_{y_i} \mar'
    \beta_i
    x_i '
    +
\sum_{s \in [n]}
\tsp_{y_s} \mar'
    \beta_s
    x_s '
  \right) x_i
  \\
    &=
    \widehat{\beta}_i'
    \mar \tsp_{y_i}
    \left(
-
\tsp_{y_i} \mar'
    \beta_i
    x_i '
    -
\sum_{s \in [n]}
\alpha_s
    x_s '
  \right) x_i \qquad \because \cref{equation: beta alpha relation via transposition}
  \\
    &=
    \widehat{\beta}_i'
    \mar \tsp_{y_i}
    \left(
-
\tsp_{y_i} \mar'
    \beta_i
    x_i '
    +
    \mathbf{w}'
  \right) x_i \qquad \because \cref{equation: primal-dual variable relation}
  \\
    &=
    \widehat{\beta}_i'
    \left(
-
    \mar \tsp_{y_i}
\tsp_{y_i} \mar'
    \beta_i
    \|x_i\|^2_2
    +
    \mar \tsp_{y_i}
    \mathbf{w}'x_i  \right)
  \\
    &=
    \widehat{\beta}_i'
    \left(
    \mar \tsp_{y_i}
    \mathbf{w}'x_i
-
    \mar  \mar'
    \beta_i
  \|x_i\|^2_2 \right)  \qquad \because \tsp_{y_i}^2 = \mathbf{I}
  \\
    &=
    \widehat{\beta}_i'
    \left(
    \mar \tsp_{y_i}
    \mathbf{w}'x_i
-
\bm{\Theta}
    \beta_i
  \|x_i\|^2_2 \right) 
  \qquad \because \cref{equation: pi pi transpose is Theta}
\end{align*}
Therefore, we have
\begin{align*}
  g(\widehat{\bm{\beta}})
  &=
  \frac{1}{2} 
  \| x_i\|_2^2
  \widehat{\beta}_i'
  \bm{\Theta} 
  \widehat{\beta}_i
               +
    \widehat{\beta}_i'
    \left(
    \mar
    \tsp_{y_i}
    \mathbf{w}'
x_i
      -
  \bm{\Theta} 
    \beta_i
    \|x_i \|_2^2
    - \vone
  \right) +C_i
  \\
  &=
  \frac{1}{2} 
  \| x_i\|_2^2
  \widehat{\beta}_i'
  \bm{\Theta} 
  \widehat{\beta}_i
               -
    \widehat{\beta}_i'
    \left(
    \vone
    -
    \mar
    \tsp_{y_i}
    \mathbf{w}'
x_i
      +
  \bm{\Theta} 
    \beta_i
    \|x_i \|_2^2
  \right) +C_i
\end{align*}
Thus, \cref{equation: WW-SVM dual subproblem - reparametrized - appendix} is equivalent to
\begin{align*}
\min_{\widehat{\bm{\beta}} \in \mathcal{G}}
\quad
&
  \frac{1}{2} 
  \| x_i\|_2^2
  \widehat{\beta}_i'
  \bm{\Theta} 
  \widehat{\beta}_i
               -
    \widehat{\beta}_i'
    \left(
    \vone
    -
    \mar
    \tsp_{y_i}
    \mathbf{w}'
x_i
      +
  \bm{\Theta} 
    \beta_i
    \|x_i \|_2^2
  \right) +C_i
  \\
 s.t. \quad &
  \, \widehat{\beta}_s = \beta_s,\, \forall s \in [n] \setminus \{i\}.
\end{align*}
Dropping the constant $C_i$ and dividing through by $\|x_i\|_2^2$ does not change the minimizers.
Hence,
\cref{equation: WW-SVM dual subproblem - reparametrized - appendix} has the same set of minimizers as
\begin{align*}
\min_{\widehat{\bm{\beta}} \in \mathcal{G}}
\quad
&
  \frac{1}{2} 
  \widehat{\beta}_i'
  \bm{\Theta} 
  \widehat{\beta}_i
               -
    \widehat{\beta}_i'
    \left(
      (
    \vone
    -
    \mar
    \tsp_{y_i}
    \mathbf{w}'
x_i
)/ \|x_i \|_2^2
      +
  \bm{\Theta} 
    \beta_i
  \right) 
  \\
 s.t. \quad &
  \, \widehat{\beta}_s = \beta_s,\, \forall s \in [n] \setminus \{i\}.
\end{align*}
Due to the equality constraints, the only free variable is $\widehat{\beta}_i$.
Note that the above optimization, when restricted to $\widehat{\beta}_i$, is equivalent to the optimization \cref{equation: dual subproblem generic}
with
    \[v := (\vone 
    -  \mar\tsp_{y_i} \mathbf{w}' x_i
  )/\|x_i\|_2^2
    + \bm{\Theta} \beta_i 
\]
  and $\mathbf{w}$ is as in \cref{equation: primal-dual variable relation}.
  The uniqueness of the minimizer is guaranteed by \cref{theorem: subproblem generic solver}.
\end{proof}

\subsection{Global linear convergence}\label{section: global linear convergence}
\citet{wang2014iteration} established the global linear convergence of the so-called \emph{feasible descent method} when applied to a certain class of problems. As an application, they prove global linear convergence for coordinate descent for solving the dual problem of the binary SVM with the hinge loss. 
\citet{wang2014iteration} considered optimization problems of the following form:
\begin{equation}
  \label{equation: wang and lin optim}
  \min_{x \in \mathcal{X}} f(x) := g(\mathbf{E}x) + b'x
\end{equation}
where 
$f : \mathbb{R}^n \to \mathbb{R}$ is a function such that $\nabla f$ is Lipschitz continuous,
$\mathcal{X} \subseteq \mathbb{R}^n$ is a polyhedral set,
$\argmin_{x \in \mathcal{X}} f(x)$ is nonempty,
$g : \mathbb{R}^m \to \mathbb{R}$ is a strongly convex function such that $\nabla g$ is Lipschitz continuous,
and $\mathbf{E} \in \mathbb{R}^{m\times n}$ and $b \in \mathbb{R}^n$ are fixed matrix and vector, respectively.

Below, let $\mathcal{P}_{\mathcal{X}} : \mathbb{R}^n \to \mathcal{X}$ denote the orthogonal projection on $\mathcal{X}$.

\begin{definition}
  In the context of \cref{equation: wang and lin optim}, an iterative algorithm that produces a sequence $\{x^0,x^1,x^2,\dots\} \subseteq \mathcal{X}$ is a \emph{feasible descent method} if there exists a sequence $\{\epsilon^0,\epsilon^1,\epsilon^2,\dots\} \subseteq \mathbb{R}^n$ such that for all $t \ge 0$
  \begin{align}
  x^{t+1} 
  &=
  \mathcal{P}_{\mathcal{X}}\left(
  x^t
  - \nabla f(x^t) + {\epsilon}^t
  \right)
  \label{equation: feasible descent 1 - generic}
  \\
  \|{\epsilon}^t\|
  &\le
  B
  \| x^t- x^{t+1} \|
  \label{equation: feasible descent 2 - generic}
  \\
  f(x^{t})
  -
  f(x^{t+1})
  &\ge
  \Gamma \| x^t-x^{t+1} \|^2
  \label{equation: feasible descent 3 - generic}
  \end{align}
  where $B, \Gamma > 0$.

\end{definition}
One of the main result of \cite{wang2014iteration} is
\begin{theorem}[Theorem 8 from \cite{wang2014iteration}]
  \label{theorem: wang and lin}
  Suppose an optimization problem $\min_{x \in \mathcal{X}} f(x)$ is of the form \cref{equation: wang and lin optim} and $\{x^0,x^1,x^2,\dots \} \subseteq \mathcal{X}$ is a sequence generated by a feasible descent method. Let $f^* := \min_{x \in \mathcal{X}} f(x)$. Then there exists $\Delta \in (0,1)$ such that 
  \[
  f(x^{t+1}) - f^*
\le \Delta ( f(x^t) - f^*),\quad \forall t \ge 0.\]
\end{theorem}

Now, we begin verifying that the WW-SVM dual optimization and the BCD algorithm for WW-SVM satisfies the requirements of \cref{theorem: wang and lin}.

Given $\bm{\beta} \in \mathbb{R}^{(k-1)\times n}$, define its vectorization
\[
  \mathrm{vec}(\bm{\beta})
  =
  \begin{bmatrix}
    \beta_1 \\
    \vdots\\
    \beta_n
  \end{bmatrix}
  \in \mathbb{R}^{(k-1)n}.
\]
Define the matrix 
$\mathbf{P}_{is} = \mar \tsp_{y_i} x_i' x_s \tsp_{y_s} \mar' \in \mathbb{R}^{(k-1)\times (k-1)}$, and $\mathbf{Q} \in \mathbb{R}^{(k-1)n \times (k-1)n}$ by
\[
  \mathbf{Q}
  =
  \begin{bmatrix}
    \mathbf{P}_{11} & \mathbf{P}_{12} & \cdots & \mathbf{P}_{1n} \\
    \mathbf{P}_{21} & \mathbf{P}_{22} & \cdots & \mathbf{P}_{2n} \\
    \vdots & \vdots & \ddots & \vdots \\
    \mathbf{P}_{n1} & \mathbf{P}_{n2} & \cdots & \mathbf{P}_{nn} \\
  \end{bmatrix}.
\]
Let
\[
  \mathbf{E} = 
  \begin{bmatrix}
    x_1 \tsp_{y_1} \mar'\\
    x_2 \tsp_{y_2} \mar'\\
    \vdots \\
    x_n \tsp_{y_n} \mar'\\
  \end{bmatrix}.
\]
We observe that $\mathbf{Q} = \mathbf{E}' \mathbf{E}$.
Thus, $\mathbf{Q}$ is symmetric and positive semi-definite.
Let $\|\mathbf{Q}\|_{op}$ be the operator norm of $\mathbf{Q}$.
\begin{proposition}
  \label{proposition: WW-SVM dual optimization satisfies assumption 2}
  The optimization \cref{equation: WW-SVM dual optimization reparametrized} is of the form \cref{equation: wang and lin optim}.
  More precisely, the optimization \cref{equation: WW-SVM dual optimization reparametrized} can be expressed as
\begin{equation}
  \label{equation: WW-SVM dual optimization reparametrized assumption 2}
  \min_{\bm{\beta} \in \mathcal{G}}
  \quad
  g(\bm{\beta})=
  \varphi(\mathbf{E} \mathrm{vec}(\bm{\beta})) - \vone'  \mathrm{vec}(\bm{\beta})
\end{equation}
where the feasible set $\mathcal{G}$ is a nonempty polyhedral set (i.e., defined by a system of linear inequalities, hence convex), $\varphi$ is strongly convex, and $\nabla g$ is Lipschitz continuous with Lipschitz constant 
$L:=\|\mathbf{Q}\|_{op}$.
Furthermore, \cref{equation: WW-SVM dual optimization reparametrized assumption 2} has at least one minimizer.
\end{proposition}
\begin{proof}
  
  Observe
\begin{align*}
  g(\bm{\beta})
  &=
  \frac{1}{2}
\sum_{i,s \in [n]}
x_s' x_i
  \beta_i' \mar \tsp_{y_i} 
  \tsp_{y_s} \mar'
  \beta_s
  -
\sum_{i \in [n]} \vone' \beta_i
\\
  &=
\frac{1}{2}
\mathrm{vec}(\bm{\beta})'
\mathbf{Q}
\mathrm{vec}(\bm{\beta})
- \vone' \mathrm{vec}(\bm{\beta})
\\
  &=
\frac{1}{2}
(\mathbf{E} \mathrm{vec}(\bm{\beta}))'
(\mathbf{E} \mathrm{vec}(\bm{\beta}))
- \vone' \mathrm{vec}(\bm{\beta})\\
  &=
  \varphi(\mathbf{E} \mathrm{vec}(\bm{\beta})) - \vone'  \mathrm{vec}(\bm{\beta})
\end{align*}
where $\varphi(\bullet ) = \frac{1}{2} \| \bullet \|^2$.
Note that $\mathrm{vec}(\nabla g(\bm{\beta})) = \mathbf{Q} \mathrm{vec}(\bm{\beta}) - \vone$.
Hence, the Lipschitz constant of $g$ is $\|\mathbf{Q}\|_{op}$.
For the ``Furthermore'' part, note that the above calculation shows that \cref{equation: WW-SVM dual optimization reparametrized assumption 2} is a quadratic program where the second order term is positive semi-definite and the constraint set is convex. Hence, \cref{equation: WW-SVM dual optimization reparametrized assumption 2} has at least one minimizer.
\end{proof}

Let $B = [0,C]^{k-1}$.
Let $\bm{\beta}^t$ be $\bm{\beta}$ at the end of the $t$-iteration of the outer loop of \cref{algorithm: BCD}.
Define
\[
  \bm{\beta}^{t,i} :
  =
  [
    \beta_1^{t+1} , \cdots , \beta_i^{t+1} , \beta_{i+1}^t , \cdots , \beta_{n}^t
    ].
\]
By construction, we have
\begin{equation}
  \label{equation: linear convergence argmin condition}
  \beta^{t+1}_i
  =
  \argmin_{\beta \in B}
  g\left(
    [
    \beta_1^{t+1} , \cdots , \beta_{i-1}^{t+1} , \beta , \beta_{i+1}^t , \cdots , \beta_{n}^t
    ]
  \right)
\end{equation}
For each $i=1,\dots, n$, let
\[
  \nabla_i g(\bm{\beta})
  =
  \left[
  \frac{\partial g}{\partial \beta_{1i}}(\bm{\beta}) ,
  \frac{\partial g}{\partial \beta_{2i}}(\bm{\beta}) ,
  \dots ,
  \frac{\partial g}{\partial \beta_{(k-1)i}}(\bm{\beta})
  \right]'.
\]
By Lemma 24 \cite{wang2014iteration}, we have
\[
  \beta_i^{t+1} =
  \mathcal{P}_B(\beta_i^{t+1} - \nabla_i g(\bm{\beta}^{t,i}))
\]
where $\mathcal{P}_B$ denotes orthogonal projection on to $B$.
Now, define $\bm{\epsilon}^t \in \mathbb{R}^{(k-1)\times n}$ such that
\[
  \epsilon_i^t
  =
  \beta_i^{t+1} - \beta_i^{t}
  -
  \nabla_i g(\bm{\beta}^{t,i})
  + \nabla_i g(\bm{\beta}^t).
\]

\begin{proposition}
  \label{proposition: BCD for WW-SVM is a feasible descent method}
  The BCD algorithm for the WW-SVM is a feasible descent method. More precisely,
  the sequence $\{\bm{\beta}^0, \bm{\beta}^1,\dots \}$ satisfies the following conditions:
  \begin{align}
  \bm{\beta}^{t+1} 
  &=
  \mathcal{P}_{\mathcal{G}}\left(
  \bm{\beta}^t
  - \nabla g(\bm{\beta}^t) + \bm{\epsilon}^t
  \right)
  \label{equation: feasible descent 1}
  \\
  \|\bm{\epsilon}^t\|
  &\le
  (1+
  \sqrt{n}
  L
  )
  \| \bm{\beta}^t- \bm{\beta}^{t+1} \|
  \label{equation: feasible descent 2}
  \\
  g(\bm{\beta}^{t})
  -
  g(\bm{\beta}^{t+1})
  &\ge
  \Gamma \| \bm{\beta}^t-\bm{\beta}^{t+1} \|^2
  \label{equation: feasible descent 3}
  \end{align}
  where $L$ is as in \cref{proposition: WW-SVM dual optimization satisfies assumption 2}, $\Gamma := \min_{i\in [n]} \frac{\|x_i\|^2}{2}$, $\mathcal{G}$ is the feasible set of 
  \cref{equation: WW-SVM dual optimization reparametrized}, and $\mathcal{P}_{\mathcal{G}}$ is the orthogonal projection onto $\mathcal{G}$.
\end{proposition}

  The proof of \cref{proposition: BCD for WW-SVM is a feasible descent method} essentially generalizes Proposition 3.4 of \cite{luo1993error} to the higher dimensional setting:
\begin{proof}
Recall that $\mathcal{G} = B^{\times n}:= B \times \cdots \times B$.
Note that the $i$-th block of $
  \bm{\beta}^t
  - \nabla g(\bm{\beta}^t) + \bm{\epsilon}^t$ is
  \[
    {\beta}^t_i
    -\nabla_i g(\bm{\beta}^t) + \epsilon_i^t
    =
    {\beta}^t_i
    -\nabla_i g(\bm{\beta}^t) + (
  \beta_i^{t+1} 
  - \beta_i^{t}
  - \nabla_i g(\bm{\beta}^{t,i})
  + \nabla_i g(\bm{\beta}^t))
  =
  \beta_i^{t+1} 
  - \nabla_i g(\bm{\beta}^{t,i}).
  \]
Thus, the $i$-th block of $
\mathcal{P}_{\mathcal{G}}(
  \bm{\beta}^t
  - \nabla g(\bm{\beta}^t) + \bm{\epsilon}^t)$ is
  \[
    \mathcal{P}_B(
  \beta_i^{t+1} 
  - \nabla_i g(\bm{\beta}^{t,i})
  )
  =
  \beta_i^{t+1}.
  \]
  This is precisely the identity \cref{equation: feasible descent 1}.

Next, we have
\begin{align*}
  \|\epsilon^t_i\| 
  &\le
  \|\beta_i^{t+1} - \beta_i^{t}\|
  +
  \|
  \nabla_i g(\bm{\beta}^{t,i})
  -
  \nabla_i g(\bm{\beta}^t)
\|
\\
  &\le
  \|\beta_i^{t+1} - \beta_i^{t}\|
  +
  L
  \|
  \bm{\beta}^{t,i}
  -
  \bm{\beta}^t
\|
\\
  &\le
  \|\beta_i^{t+1} - \beta_i^{t}\|
  +
  L
  \|
  \bm{\beta}^{t+1}
  -
  \bm{\beta}^t
\|.
\end{align*}
From this, we get that
\begin{align*}
  \|\bm{\epsilon}^t\|
  &=
  \sqrt{
  \sum_{i=1}^n
  \|\epsilon^t_i\|^2
}
  \\
  &\le
  \sqrt{
  \sum_{i=1}^n
  (
  \|\beta_i^{t+1} - \beta_i^{t}\|
  +
  L
  \|
  \bm{\beta}^{t+1}
  -
  \bm{\beta}^t
\|
)^2
}
\\
  &
  \le
  \sqrt{
  \sum_{i=1}^n
  \|\beta_i^{t+1} - \beta_i^{t}\|^2
}
  +
  \sqrt{
  \sum_{i=1}^n
  L^2
  \|
  \bm{\beta}^{t+1}
  -
  \bm{\beta}^t
\|^2
}
\\
  &
  =
  \|\bm{\beta}^{t+1} - \bm{\beta}^t\|
  +
  \sqrt{n}
  L
  \|\bm{\beta}^{t+1} - \bm{\beta}^t\|
  \\
  &=
  (1+
  \sqrt{n}
  L
  )
  \|\bm{\beta}^{t+1} - \bm{\beta}^t\|.
\end{align*}
Thus, we conclude that 
$
  \|\bm{\epsilon}^t\|
  \le
  (1+
  \sqrt{n}
  L
  )
  \|\bm{\beta}^{t+1} - \bm{\beta}^t\|
  $ which is \cref{equation: feasible descent 2}.

Finally, we show that 
\[
  g(\bm{\beta}^{t,i-1})
  -
  g(\bm{\beta}^{t,i})
  +
  \nabla_i g(\bm{\beta}^{t,i}) ' 
  (\beta_i^{t+1} - \beta_i^t)
  \ge
  \Gamma
  \|\beta_i^{t+1} - \beta_i^t\|^2
\]
where $\Gamma := \min_{i \in [n]} \frac{\|x_i\|^2}{2}$.

\begin{lemma}
  \label{lemma: g restricted to a single block}
Let $
\beta_1 , \cdots , \beta_{i-1} , \beta , \beta_{i+1} , \cdots , \beta_{n} \in \mathbb{R}^{k-1}$ be arbitrary.
Then there exist $v \in \mathbb{R}^{k-1}$ and $C \in \mathbb{R}$ which depend only on $\beta_1,\dots, \beta_{i-1}, \beta_{i+1},\dots, \beta_{n}$, but not on $\beta$, such that
\[
  g\left(
    [
    \beta_1 , \cdots , \beta_{i-1} , \beta , \beta_{i+1} , \cdots , \beta_{n}
    ]
  \right)
  =
  \frac{1}{2} \|x_i\|^2 \beta' \beta
  - v' \beta - C.
\]
In particular, we have
\[
  \nabla_i 
  g\left(
    [
    \beta_1 , \cdots , \beta_{i-1} , \beta , \beta_{i+1} , \cdots , \beta_{n}
    ]
  \right)
  =
  \|x_i\|^2 \beta - v.
\]
\end{lemma}
\begin{proof}
  The result follows immediately from the identity
  \cref{equation: g restricted to a single block}.
\end{proof}

\begin{lemma}
  \label{lemma: strongly convex}
Let $
\beta_1 , \cdots , \beta_{i-1} , \beta , \eta, \beta_{i+1} , \cdots , \beta_{n} \in \mathbb{R}^{k-1}$ be arbitrary.
Then we have
\begin{align*}
  &
  g\left(
    [
    \beta_1 , \cdots , \beta_{i-1} , \eta , \beta_{i+1} , \cdots , \beta_{n}
    ]
  \right)
  -
  g\left(
    [
    \beta_1 , \cdots , \beta_{i-1} , \beta , \beta_{i+1} , \cdots , \beta_{n}
    ]
  \right)
  \\
  &\qquad
  +
  \nabla_i 
  g\left(
    [
    \beta_1 , \cdots , \beta_{i-1} , \beta , \beta_{i+1} , \cdots , \beta_{n}
    ]
  \right)'
  (\beta - \eta)
  \\
  &
  =
  \frac{\|x_i\|^2}{2}
    \|\eta - \beta\|^2
\end{align*}

\end{lemma}

\begin{proof}
  Let $v,C$ be as in \cref{lemma: g restricted to a single block}.
  We have
  
  \begin{align*}
    &
  g\left(
    [
    \beta_1 , \cdots , \beta_{i-1} , \eta , \beta_{i+1} , \cdots , \beta_{n}
    ]
  \right)
  -
  g\left(
    [
    \beta_1 , \cdots , \beta_{i-1} , \beta , \beta_{i+1} , \cdots , \beta_{n}
    ]
  \right)
  \\
  &=
  \frac{\|x_i\|^2}{2} \|\eta\|^2
   - v' \eta
   -
  \frac{\|x_i\|^2}{2} \|\beta\|^2
   + v' \beta
  \\
  &=
  \frac{\|x_i\|^2}{2} (\|\eta\|^2 - \|\beta\|^2)
   + v' (\beta-\eta)
\end{align*}
and

\[
  \nabla_i
  g\left(
    [
    \beta_1 , \cdots , \beta_{i-1} , \beta , \beta_{i+1} , \cdots , \beta_{n}
    ]
  \right)'
  (\beta - \eta)
  =
  (\|x_i\|^2\beta - v)' (\beta - \eta)
  =
  \|x_i\|^2(\|\beta\|^2 - \beta' \eta)
  -
  v'(\beta-\eta).
\]

Thus,

\begin{align*}
  &
  g\left(
    [
    \beta_1 , \cdots , \beta_{i-1} , \eta , \beta_{i+1} , \cdots , \beta_{n}
    ]
  \right)
  -
  g\left(
    [
    \beta_1 , \cdots , \beta_{i-1} , \beta , \beta_{i+1} , \cdots , \beta_{n}
    ]
  \right)
  \\
  &\qquad
  +
  \nabla_i 
  g\left(
    [
    \beta_1 , \cdots , \beta_{i-1} , \beta , \beta_{i+1} , \cdots , \beta_{n}
    ]
  \right)'
  (\beta - \eta)
  \\
  &
  =
  \frac{\|x_i\|^2}{2} (\|\eta\|^2 - \|\beta\|^2)
   + v' (\beta-\eta)
   +
  \|x_i\|^2(\|\beta\|^2 - \beta' \eta)
  -
  v'(\beta-\eta)
  \\
  &=
  \frac{\|x_i\|^2}{2} (\|\eta\|^2 - \|\beta\|^2)
   +
  \|x_i\|^2(\|\beta\|^2 - \beta' \eta)
  \\
  &=
  \|x_i\|^2
  \left(
  \frac{1}{2} (\|\eta\|^2 - \|\beta\|^2)
   +
  (\|\beta\|^2 - \beta' \eta)
  \right)
  \\
  &=
  \|x_i\|^2
  \left(
  \frac{1}{2} (\|\eta\|^2 + \|\beta\|^2)
   -
   \beta' \eta
  \right)
  \\
  &=
  \frac{\|x_i\|^2}{2}
    \|\eta - \beta\|^2
\end{align*}
as desired.
\end{proof}

Applying \cref{lemma: strongly convex}, we have
\[
  g(\bm{\beta}^{t,i-1})
  -
  g(\bm{\beta}^{t,i})
  +
  \nabla_i g(\bm{\beta}^{t,i}) ' 
  (\beta_i^{t+1} - \beta_i^t)
  \ge
  \frac{\|x_i\|^2}{2}
  \|\beta_i^{t+1} - \beta_i^t\|^2.
\]
Since \cref{equation: linear convergence argmin condition} is true, we have by Lemma 24 of \cite{wang2014iteration} that
\[
  \nabla_i g(\bm{\beta}^{t,i}) ' 
  (\beta_i^t - \beta_i^{t+1}) \ge 0
\]
Equivalently, $ \nabla_i g(\bm{\beta}^{t,i}) ' 
  ( \beta_i^{t+1}- \beta_i^t ) \le 0$.
Thus, we deduce that
\[
  g(\bm{\beta}^{t,i-1})
  -
  g(\bm{\beta}^{t,i})
  \ge
  \frac{\|x_i\|^2}{2}
  \|\beta_i^{t+1} - \beta_i^t\|^2
  \ge
  \Gamma
  \|\beta_i^{t+1} - \beta_i^t\|^2
\]
Summing the above identity over $i\in [n]$, we have
\[
  g(\bm{\beta}^{t,0})
  -
  g(\bm{\beta}^{t,n})
  =
  \sum_{i=1}^n
  g(\bm{\beta}^{t,i-1})
  -
  g(\bm{\beta}^{t,i})
  \ge
  \Gamma
  \sum_{i=1}^n
  \|\beta_i^{t+1} - \beta_i^t\|^2
  =
  \Gamma \|\bm{\beta}^{t+1} - \bm{\beta}^t\|^2
\]
Since $(\bm{\beta}^{t,0}) = \bm{\beta}^t$ and $\bm{\beta}^{t,n} = \bm{\beta}^{t+1}$, we conclude that
$
  g(\bm{\beta}^{t})
  -
  g(\bm{\beta}^{t+1})
  \ge
  \Gamma \|\bm{\beta}^{t+1} - \bm{\beta}^t\|^2.
$
\end{proof}

To conclude the proof of \cref{theorem: linear convergence}, we note that
  \cref{proposition: BCD for WW-SVM is a feasible descent method}
  and 
  \cref{proposition: WW-SVM dual optimization satisfies assumption 2}
  together
  imply that the requirements of Theorem 8 from \cite{wang2014iteration} (restated as \cref{theorem: wang and lin} here) are satisfied for the BCD algorithm for WW-SVM. Hence, we are done. \qed

  \subsection{Proof of \texorpdfstring{\cref{theorem: subproblem generic solver}}{main theorem}}
  \label{section: full proof of generic solver}
The goal of this section is to prove \cref{theorem: subproblem generic solver}.
The time complexity analysis has been carried out at the end of \cref{section: the subproblem solver} of the main article.
Below, we focus on the part of the theorem on the correctness of the output.
Throughout this section, $k \ge 2$, $C > 0$ and $v\in \mathbb{R}^{k-1}$ are assumed to be fixed.
Additional variables used are summarized in 
\cref{table: variables}.

\renewcommand{\arraystretch}{1.5}
\begin{table}[t]
\caption{Variables used in 
  Section~\ref{section: full proof of generic solver}
  }
\label{table: variables}
\vskip 0.15in
\begin{center}
\begin{small}
\begin{sc}
\begin{tabular}{lll}
\toprule
Variable(s) & defined in   & nota bene
\\
\midrule
$t$ & \cref{algorithm: subproblem generic solver} & iteration index\\
$\ell,\mathtt{vals}, \delta_t, \gamma_t$ & Subroutine~\ref{algorithm: construct critical set}& $t \in [\ell]$ is an iteration index\\
$\mathtt{up},\mathtt{dn}$ & Subroutine~\ref{algorithm: construct critical set}& symbols\\
$\widetilde{b},\widetilde{\gamma}, v_{\max}$ & 
\cref{lemma: subproblem clipping representation}&\\
$\angl{1},\dots, \angl{k-1}$ & \cref{algorithm: subproblem generic solver} & \\
$\mdcarp{t}, \upcarp{t}, S^t, \widehat{\gamma}^t, \widehat{b}^t$ & \cref{algorithm: subproblem generic solver} & $t \in [\ell]$ is an iteration index\\
$\llfloor k \rrfloor$, 
  $\upsetp{\gamma},
  \mdsetp{\gamma} ,
\upcarp{\gamma}, \mdcarp{\gamma} $ &\cref{definition: up and md}& $\gamma \in \mathbb{R}$ is a real number\\
                         $S^{(n_{\BZC}, n_{\GTC})}$, $\widehat{\gamma}^{(n_{\BZC}, n_{\GTC})}$,
$\widehat{b}^{(n_{\BZC}, n_{\GTC})}$
                         &\cref{definition: reco}& $(n_{\BZC}, n_{\GTC}) \in \llfloor k \rrfloor^2$\\
$\mathtt{vals}^+$ & \cref{definition: vals plus set} &\\
$\mathtt{u}(j), \mathtt{d}(j)$ &
  \cref{definition: u and d functions}
                               & $j\in [k-1]$ is an integer
  \\
$\mathtt{crit}_1$,
$\mathtt{crit}_2$
                         & 
  \cref{definition: critical sets}
                 & \\
                 $\mathtt{KKT\_cond}()$
                         & 
  Subroutine~\ref{algorithm: check KKT condition}
                 & \\
\bottomrule
\end{tabular}
\end{sc}
\end{small}
\end{center}
\vskip -0.1in
\end{table}
\renewcommand{\arraystretch}{1}

\subsubsection{The clipping map}

First, we recall the clipping map:
\begin{definition}
  The \emph{clipping map} $\clip : \mathbb{R}^{k-1} \to [0,C]^{k-1}$ is the function defined as follows: for 
  $w \in \mathbb{R}^{k-1}$, $[\clip(w)]_i :=\max\{0, \min \{C,w_i\}\}$.
\end{definition}

\begin{lemma}
  \label{lemma: subproblem clipping representation}
  Let $v_{\max} = \max_{i \in [k-1]} v_i$.
  The optimization \cref{equation: dual subproblem generic} has a unique global minimum $\widetilde{b}$ satisfying the following:
  \begin{enumerate}
    \item $\widetilde{b} = \clip(v - \widetilde{\gamma} \vone)$
  for some $\widetilde{\gamma} \in \mathbb{R}$
\item 
  $\widetilde{\gamma} = \sum_{i=1}^{k-1} \widetilde{b}_i$. In particular, $\widetilde{\gamma} \ge 0$.
\item 
  If $v_i \le 0$, then $\widetilde{b}_i = 0$. In particular, if $v_{\max} \le 0$, then $\widetilde{b} = \vzero$.
\item 
  If $v_{\max} > 0$, then
  $0 < \widetilde{\gamma} < v_{\max}$.
  \end{enumerate}
\end{lemma}

\begin{proof}
  We first prove part 1.
  The optimization \cref{equation: dual subproblem generic} is a minimization over a convex domain with strictly convex objective, and hence  has a unique global minimum $\widetilde{b}$.
  For each $i \in [k-1]$, let $\lambda_i, \mu_i \in \mathbb{R}$ be the dual variables for the constraints $0 \ge b_i - C$ and $0 \ge -b_i$, respectively.
  The Lagrangian for the optimization \cref{equation: dual subproblem generic} is
  \[
    \mathcal{L}(b, \lambda, \mu)
    =
   \frac{1}{2}b' (\mathbf{I} + \mathbf{O}) b - v' b
   +
   (b - C)'\lambda
   +
   (-b)' \mu.
  \]
  Thus, the stationarity (or gradient vanishing) condition is
  \[
    0=
    \nabla_b \mathcal{L}(b, \lambda, \mu)
    =
   (\mathbf{I} + \mathbf{O}) b - v
   +
   \lambda
   -
    \mu.
  \]
The KKT conditions are as follows:
\begin{align}
  \mbox{for all $i \in [k-1]$, the following holds:}&
  \nonumber
  \\
  [(\mathbf{I}+\mathbf{O})b]_i + \lambda_i - \mu_i &=v_i \quad \mbox{stationarity}
  \label{equation: subproblem KKT stationarity} \\
  C  \ge b_i & \ge 0 \quad \mbox{primal feasibility}
  \label{equation: subproblem KKT primal feasibility} \\
  \lambda_i & \ge 0 \quad \mbox{dual feasibility}
  \label{equation: subproblem KKT dual feasibility} \\
  \mu_i & \ge 0 \quad\quad\mbox{\texttt{"}}
  \label{equation: subproblem KKT dual feasibility 2} \\
  \lambda_i (C-b_i) &= 0 \quad \mbox{complementary slackness}
  \label{equation: subproblem KKT complementary slackness} \\
  \mu_i b_i &= 0\quad \quad\mbox{\texttt{"}}
  \label{equation: subproblem KKT complementary slackness 2}
\end{align}
\Crefrange{equation: subproblem KKT stationarity}{equation: subproblem KKT complementary slackness 2} are satisfied if and only if $b = \widetilde{b}$ is the global minimum.

Let $\widetilde{\gamma} \in \mathbb{R}$ be such that $\widetilde{\gamma} \vone = \mathbf{O} \widetilde{b}$.
Note that by definition, part 2 holds.
Furthermore, \cref{equation: subproblem KKT stationarity} implies
\begin{equation}
  \label{equation: global minimum beta expression}
  \widetilde{b} = v - \widetilde{\gamma} \vone - \lambda + \mu.
\end{equation}
Below, fix some $i \in [k-1]$.
Note that $\lambda_i$ or $\mu_i$ cannot both be nonzero. Otherwise, \cref{equation: subproblem KKT complementary slackness} and \cref{equation: subproblem KKT complementary slackness 2} would imply that $C = \widetilde{b}_i = 0$, a contradiction.
We claim the following:
\begin{enumerate}
  \item If $v_i - \widetilde{\gamma} \in [0,C]$, then $\lambda_i = \mu_i = 0$ and $\widetilde{b}_i = v_i - \widetilde{\gamma}$.
  \item If $v_i - \widetilde{\gamma} > C$, then $\widetilde{b}_i = C$.
  \item $v_i - \widetilde{\gamma} <0$, then $\widetilde{b}_i = 0$.
\end{enumerate}
We prove the first claim. 
To this end, suppose $v_i - \widetilde{\gamma} \in [0,C]$.
We will show $\lambda_i = \mu_i = 0$ by contradiction.
Suppose $\lambda_i > 0$.
Then we have $C = \widetilde{b}_i$ and $\mu_i  =0$.
Now, \cref{equation: global minimum beta expression} implies that
$C = \widetilde{b}_i = v_i - \widetilde{\gamma} - \lambda_i$.
However, we now have $v_i - \widetilde{\gamma} - \lambda_i \le C - \lambda_i < C$, a contradiction.
Thus, $\lambda_i = 0$.
Similarly, assuming $\mu_i > 0$ implies \[0 = \widetilde{b}_i = v_i - \lambda + \mu_i \ge 0 + \mu_i > 0,\] a contradiction.
This proves the first claim.

Next, we prove the second claim.
Note that
\[
  C \ge \widetilde{b}_i = v_i - \widetilde{\gamma} - \lambda_i + \mu_i > C - \lambda_i + \mu_i
  \implies 
  0 > -\lambda_i + \mu_i \ge -\lambda_i.
\]
In particular, we have $\lambda_i > 0$ which implies $C = \widetilde{b}_i$ by complementary slackness.

Finally, we prove the third claim.
Note that
\[
  0 \le \widetilde{b}_i = v_i - \widetilde{\gamma} - \lambda_i + \mu_i
  < -\lambda_i + \mu_i \le \mu_i 
\]
Thus, $\mu_i > 0$ and so $0 = \widetilde{b}_i$ by complementary slackness.
This proves that $\widetilde{b} = \clip(v - \widetilde{\gamma} \vone)$, which concludes the proof of part 1.

For part 2, note that $\widetilde{\gamma} = \sum_{i=1}^{k-1} \widetilde{b}_i$ holds by definition. The ``in particular'' portion follows immediately from $\widetilde{b} \ge 0$.

We prove part 3 by contradiction. Suppose there exists $i \in [k-1]$ such that $v_i \le 0$ and $\widetilde{b}_i > 0$.
Thus, by \cref{equation: subproblem KKT complementary slackness 2},  we have $\mu_i = 0$. By \cref{equation: subproblem KKT stationarity}, we have $b_i + \widetilde{\gamma} \le b_i + \widetilde{\gamma} + \lambda_i = v_i \le 0$.
Thus, we have $-\widetilde{\gamma} \ge b_i > 0$, or equivalently, $\widetilde{\gamma} < 0$. However, this contradicts part 2. Thus, $\widetilde{b}_i = 0$ whenever $v_i \le 0$. The `` in particular'' portion follows immediately from the observation that $v_{\max} \le 0$ implies that $v_i \le 0$ for all $i \in [k-1]$.

For part 4, we first prove that $\widetilde{\gamma} < v_{\max}$ by contradiction.
Suppose that $\widetilde{\gamma} \ge v_{\max}$.
Then we have $v - \widetilde{\gamma} \vone \le 
v - v_{\max} \vone \le \vzero$. Thus, by part 1, we have $\widetilde{b}
=
\clip ( v-\widetilde{\gamma} \vone) = \vzero$.
By part 2, we must have that $\widetilde{\gamma} = \sum_{i=1}^{k-1} \widetilde{b}_i = 0$. However, $\widetilde{\gamma} \ge v_{\max} > 0$, which is a contradiction.

Finally, we prove that $\widetilde{\gamma} > 0$ again by contradiction. Suppose that $\widetilde{\gamma} = 0$. Then part 2 and the fact that $\widetilde{b} \ge \vzero$ implies that $\widetilde{b} = \vzero$.
However, by part 1, we have $\widetilde{b} = \clip(v)$.
Now, let $i^*$ be such that $v_{i^*} = v_{\max}$.
This implies that $\widetilde{b}_{i^*} = \clip(v_{\max}) > 0$, a contradiction.
\end{proof}

\subsubsection{Recovering \texorpdfstring{$\widetilde{\gamma}$}{tilde gamma} from discrete data}

\begin{definition}
  \label{definition: up and md}
For $\gamma \in \mathbb{R}$, let $b^\gamma := \clip(v- {\gamma} \vone) \in \mathbb{R}^{k-1}$.
Define
\begin{align*} 
  \upsetp{\gamma} &:=\{i \in [k-1]: b^\gamma_i = C\}\\
  \mdsetp{\gamma} &:=\{i \in [k-1]: b^\gamma_i \in (0,C)\}\\
  \upcarp{\gamma} &:= |\upsetp{\gamma}|, \quad \mbox{ and } \quad \mdcarp{\gamma} := |\mdcarp{\gamma}|.
\end{align*}
    Let $\llfloor k \rrfloor := \{0\} \cup [k-1]$.
    Note that by definition, $\mdcarp{\gamma},
    \upcarp{\gamma} \in \llfloor k \rrfloor$.
\end{definition}
Note that $\upsetp{\gamma}$ and $\mdsetp{\gamma}$ are determined by their cardinalities. This is because
\begin{align*}
  \upsetp{\gamma} &= 
  \{\angl{1}, \angl{2},\dots, \angl{n_{\GTC}^\gamma}\}\\
  \mdsetp{\gamma} &=
  \{\angl{n_{\GTC}^\gamma+1},\angl{n_{\GTC}^\gamma+2},\dots, \angl{n_{\GTC}^\gamma + n_{\BZC}^\gamma}\}.
\end{align*}
\begin{definition}
  \label{definition: disc set}  
Define
    \[
      \mathtt{disc}^+ := \{ v_i : i \in [k-1] , \, v_i > 0 \} \cup \{ v_i - C: i \in [k-1], v_i - C > 0 \}
      \cup \{0\}.
    \]
\end{definition}
    Note that $\mathtt{disc}^+$ is slightly different from $\mathtt{disc}$ as defined in the main text.

    \begin{lemma}

  \label{lemma: comp is locally constant}
  Let $\gamma', \gamma'' \in \mathtt{disc}^+$ be such that
  $\gamma \not \in \mathtt{disc}^+$ for all $\gamma \in (\gamma', \gamma'')$.
  The functions
  \begin{align*}
    (\gamma', \gamma'') \ni \gamma &\mapsto \mdsetp{\gamma}\\
    (\gamma', \gamma'') \ni \gamma &\mapsto \upsetp{\gamma}
  \end{align*}
  are constant.
    \end{lemma}

    \begin{proof}
      We first prove $\mdsetp{\lambda} = \mdsetp{\rho}$.
      Let $\lambda, \rho \in (\gamma', \gamma'')$ be such that $\lambda < \rho$.
      Assume for the sake of contradiction that 
$\mdsetp{\lambda} \ne \mdsetp{\rho}$.
Then either 1) $i \in [k-1]$ such that $v_i - \lambda \in (0,C)$ but $v_i - \rho \not \in (0,C)$
or 2)
$i \in [k-1]$ such that $v_i - \lambda \not\in (0,C)$ but $v_i - \rho \in (0,C)$.
This implies that there exists some $\gamma \in (\lambda,\rho)$ such that $v_i - \gamma \in \{0,C\}$, or equivalently, $\gamma \in \{v_i, v_i - C\}$.
Hence, $\gamma \in \mathtt{disc}^+$, which is a contradiction.
Thus, for all $\lambda, \rho \in (\gamma', \gamma'')$, we have 
$\mdsetp{\lambda} = \mdsetp{\rho}$.

      Next, we prove $\upsetp{\lambda} = \upsetp{\rho}$.
      Let $\lambda, \rho \in (\gamma', \gamma'')$ be such that $\lambda < \rho$.
      Assume for the sake of contradiction that 
$\upsetp{\lambda} \ne \upsetp{\rho}$.
Then either 1) $i \in [k-1]$ such that $v_i - \lambda \ge C $ but $v_i - \rho < C$
or 2)
$i \in [k-1]$ such that $v_i - \lambda < C$ but $v_i - \rho \ge C$.
This implies that there exists some $\gamma \in (\lambda,\rho)$ such that $v_i - \gamma = C $, or equivalently, $\gamma = v_i = C$.
Hence, $\gamma \in \mathtt{disc}^+$, which is a contradiction.
Thus, for all $\lambda, \rho \in (\gamma', \gamma'')$, we have 
$\upsetp{\lambda} = \upsetp{\rho}$.
    \end{proof}

\begin{definition}
  \label{definition: reco}
For $(n_{\BZC},n_{\GTC}) \in \llfloor k \rrfloor^2$, define
$S^{(n_{\BZC}, n_{\GTC})}$, $\widehat{\gamma}^{(n_{\BZC}, n_{\GTC})} \in \mathbb{R}$ by
\begin{align*}
  S^{(n_{\BZC}, n_{\GTC})} &:= 
  \sum_{i = n_{\GTC}+1}^{n_{\GTC} + n_{\BZC}} v_{\angl{i}},
  \\
  \widehat{\gamma}^{(n_{\BZC}, n_{\GTC})} 
                           &:=
                           \left(C\cdot n_{\GTC} + S^{(n_{\BZC}, n_{\GTC})}\right)/( n_{\BZC} + 1 ).
\end{align*}
Furthermore, define 
$\widehat{b}^{(n_{\BZC}, n_{\GTC})} 
\in \mathbb{R}^{k-1}$ such that, for $i \in [k-1]$, the $\angl{i}$-th entry is
\begin{align*}
    \widehat{b}_{\angl{i}}^{(n_{\BZC}, n_{\GTC})}
           &:=
      \begin{cases}
        C &: i \le n_{\GTC}\\
        v_{\angl{i}} - \gamma^{(n_{\BZC}, n_{\GTC})} &: n_{\GTC} < i \le n_{\GTC} + n_{\BZC}\\
        0 &: n_{\GTC} + n_{\BZC} < i.
      \end{cases}
    \end{align*}
\end{definition}
Below, recall $\ell$ as defined on
Subroutine~\ref{algorithm: construct critical set}-line~2.
\begin{lemma}
  \label{lemma: algorithm output and reco identity}
    Let $t \in [\ell]$.
    Let $n_{\BZC}^t$, $n_{\GTC}^t$, and $\widehat{b}^t$ be as in the for loop of \cref{algorithm: subproblem generic solver}.
    Then $\widehat{\gamma}^{(\mdcarp{t}, \upcarp{t})} = \widehat{\gamma}^t$
    and $\widehat{b}^{(\mdcarp{t}, \upcarp{t})} = \widehat{b}^t$.
\end{lemma}
\begin{proof}
  It suffices to show that $S^t = S^{(\mdcarp{t}, \upcarp{t})}$ where the former is defined as in \cref{algorithm: subproblem generic solver} and the latter is defined as in \cref{definition: reco}.
In other words, it suffices to show that
  \begin{equation}
    \label{equation: induction of Si identity}
S^t =
\sum_{j \in [k-1]\,:\, n_{\GTC}^t < j \le n_{\GTC}^t + n_{\BZC}^t} v_{\angl{j}}.
\end{equation}
We prove \cref{equation: induction of Si identity} by induction. The base case $t=0$ follows immediately due to the initialization in \cref{algorithm: subproblem generic solver}-line~4.

Now, suppose that \cref{equation: induction of Si identity} holds for $S^{t-1}$:
  \begin{equation}
    \label{equation: induction step of Si identity}
    S^{t-1} =
    \sum_{j \in [k-1]\,:\, n_{\GTC}^{t-1} < j \le n_{\GTC}^{t-1} + n_{\BZC}^{t-1}} v_{\angl{j}}.
\end{equation}
Consider the first case that $\delta_t = \mathtt{up}$. 
Then we have
$
n_{\GTC}^{t} + n_{\BZC}^{t}
=
n_{\GTC}^{t-1} + n_{\BZC}^{t-1}$ and
$n_{\GTC}^{t}
=n_{\GTC}^{t-1}+1$.
Thus, we have
\begin{align*}
  S^t &= S^{t-1} - v_{\angl{n^{t-1}_{\GTC}}}
  \quad \because  \mbox{Subroutine~\ref{algorithm: update variables}-line~3,}
  \\ &=
    \sum_{j \in [k-1]\,:\, n_{\GTC}^{t-1}+1 < j \le n_{\GTC}^{t-1} + n_{\BZC}^{t-1}} v_{\angl{j}}
    \quad
    \because \cref{equation: induction step of Si identity}
    \\ &=
    \sum_{j \in [k-1]\,:\, n_{\GTC}^{t} < j \le n_{\GTC}^{t} + n_{\BZC}^{t}} v_{\angl{j}}
\end{align*}
which is exactly the desired identity in \cref{equation: induction of Si identity}.

Consider the second case that $\delta_t = \mathtt{dn}$. 
Then we have
$
n_{\GTC}^{t} + n_{\BZC}^{t}
=
n_{\GTC}^{t-1} + n_{\BZC}^{t-1}+1$ and
$n_{\GTC}^{t}
=n_{\GTC}^{t-1}$.
Thus, we have
\begin{align*}
  S^t &= S^{t-1} + v_{\angl{n^t_{\GTC} + n^t_{\BZC}}}
  \quad \because  \mbox{Subroutine~\ref{algorithm: update variables}-line~6,}
  \\ &=
    \sum_{j \in [k-1]\,:\, n_{\GTC}^{t-1}+1 < j \le n_{\GTC}^{t-1} + n_{\BZC}^{t-1}+1} v_{\angl{j}}
    \quad
    \because \cref{equation: induction step of Si identity}
    \\ &=
    \sum_{j \in [k-1]\,:\, n_{\GTC}^{t} < j \le n_{\GTC}^{t} + n_{\BZC}^{t}} v_{\angl{j}}
\end{align*}
which, again, is exactly the desired identity in \cref{equation: induction of Si identity}.
\end{proof}

\begin{lemma}
  \label{lemma: reco comp identity at optimality}
  Let $\widetilde{\gamma}$ be as in \cref{lemma: subproblem clipping representation}. Then we have
    \[
\widetilde{b}
=
\widehat{b}^{(n_{\BZC}^{\widetilde{\gamma}}, n_{\GTC}^{\widetilde{\gamma}})}
=
\clip(v - 
\widehat{\gamma}^{(n_{\BZC}^{\widetilde{\gamma}}, n_{\GTC}^{\widetilde{\gamma}})} 
  \vone).
    \]
\end{lemma}

\begin{proof}
  It suffices to prove that $\widetilde{\gamma}
=
\widehat{\gamma}^{(n_{\BZC}^{\widetilde{\gamma}},n_{\GTC}^{\widetilde{\gamma}})}$.
To this end, let $i \in [k-1]$.  If $i \in \mdsetp{\widetilde{\gamma}}$, then $\widetilde{b}_i= v_i - \widetilde{\gamma}$.
If $i \in \upsetp{\widetilde{\gamma}}$, then $\widetilde{b}_i = C$. Otherwise, $\widetilde{b}_i = 0$.
Thus
\begin{equation*}
  \widetilde{\gamma}=
  \vone' \widetilde{b}
  =
  C \cdot n_{\GTC}^{\widetilde{\gamma}}
  +
  S^{(\mdcarp{\widetilde{\gamma}}, \upcarp{\widetilde{\gamma}})}
  -  \widetilde{\gamma}\cdot \mdcarp{\widetilde{\gamma}}
\end{equation*}
Solving for $\widetilde{\gamma}$, we have
\[
  \widetilde{\gamma}
  = \left(C\cdot \upcarp{\widetilde{\gamma}} + 
  S^{(\mdcarp{\widetilde{\gamma}}, \upcarp{\widetilde{\gamma}})}
\right)/( \mdcarp{\widetilde{\gamma}} + 1 )
=
\widehat{\gamma}^{(n_{\BZC}^{\widetilde{\gamma}},n_{\GTC}^{\widetilde{\gamma}})},
  \]
  as desired.
\end{proof}

\subsubsection{Checking the KKT conditions}
\begin{lemma}
  \label{lemma: kkt equivalent condition}
  Let $(\mdcar, \upcar) \in \llfloor k \rrfloor^2$.
  To simplify notation, let $b := \widehat{b}^{(\mdcar, \upcar)}$, $\gamma := \widehat{\gamma}^{(\mdcar, \upcar)}$.
We have
$\mathbf{O} b = \gamma \vone$ and
for all $i \in [k-1]$ that
\begin{equation}
  \label{equation: IOb identity}
  [(\mathbf{I} + \mathbf{O}) b]_{\angl{i}} = 
      \begin{cases}
        C+\gamma &: i \le \upcar \\
        v_{\angl{i}}&: \upcar < i \le \upcar + \mdcar \\
        \gamma &: \upcar + \mdcar < i.
      \end{cases}
\end{equation}
Furthermore, $b$ satisfies the KKT conditions 
\Crefrange{equation: subproblem KKT stationarity}{equation: subproblem KKT complementary slackness 2}
if
and only if, for all $i \in [k-1]$,
\begin{equation}
  \label{equation: pattern KKT condition}
  v_{\angl{i}}
\begin{cases}
  \ge C+\gamma &: i \le \upcar\\
  \in [\gamma, C + \gamma] &: \upcar < i \le \upcar + \mdcar \\
  \le \gamma &: \upcar + \mdcar < i.
\end{cases}
\end{equation}
\end{lemma}

\begin{proof}

First, we prove 
$\mathbf{O}b = \gamma \vone$ which is equivalent to
$
[\mathbf{O}b]_j  = \gamma$ for all $j \in [k-1]$. This is a straightforward calculation:
\begin{align*}
  [\mathbf{O}b]_{j}
=
\vone' b
&=
\sum_{i \in [k-1]}b_{\angl{i}}
\\
&=
\sum_{i \in [k-1]\,:\, i \le \upcar }
b_{\angl{i}}
+
\sum_{i \in [k-1]\,:\, \upcar < i \le \upcar + \mdcar }
b_{\angl{i}}
+
\sum_{i \in [k-1]\,:\, \upcar + \mdcar < i }
b_{\angl{i}}
\\
&=
\sum_{i \in [k-1]\,:\, i \le \upcar }
C
+
\sum_{i \in [k-1]\,:\, \upcar < i \le \upcar + \mdcar }
v_{\angl{i}} - \gamma
\\
&=
C
\cdot
\upcar
+ 
S^{(\mdcarp{t}, \upcarp{t})}
-
\mdcar \gamma
\\
& =\gamma.
\end{align*}

Since $[(\mathbf{I} + \mathbf{O})b]_i = [\mathbf{I}b]_i + [\mathbf{O}b]_i$, the identity
  \cref{equation: IOb identity} now follows immediately.

Next, we prove the ``Furthermore'' part. 
First, we prove the ``only if'' direction.
By assumption, we have $b = \widetilde{b}$ and so $\gamma = \widetilde{\gamma}$.
Furthermore, from \cref{lemma: subproblem clipping representation} we have $\widetilde{b} = \clip(v - \widetilde{\gamma} \vone)$ and so $b = \clip(v - \gamma \vone)$.
To proceed, recall that by construction, we have
\[
  b_{\angl{i}}
  =
  \begin{cases}
    C &: i \le \upcar \\
    v-\gamma &: \upcar < i \le \upcar + \mdcar \\
    0 &: \upcar +\mdcar < i
  \end{cases}
\]
Thus, if $ i \le \upcar$, then $C = b_{\angl{i}} = [\clip(v- \gamma \vone)]_{\angl{i}}$ implies that $v_{\angl{i}}- \gamma \ge C$.
If $\upcar <i \le \upcar + \mdcar$, then $b_{\angl{i}} = v_{\angl{i}} - \gamma$. Since $b_{j} \in [0,C]$ for all $j \in [k-1]$, we have in particular that $v_{\angl{i}} - \gamma \in [0,C]$.
Finally, if $\upcar + \mdcar < i$, then $0=b_{\angl{i}} =[\clip(v-\gamma\vone)]_{\angl{i}}$ implies that $v - \gamma \le 0$.
In summary,
\begin{equation*}
  v_{\angl{i}}
  -\gamma
\begin{cases}
  \ge C &: i \le \upcar\\
  \in [0, C] &: \upcar < i \le \upcar + \mdcar \\
  \le 0 &: \upcar + \mdcar < i.
\end{cases}
\end{equation*}
Note that the above identity immediately implies \cref{equation: pattern KKT condition}.

Next, we prove the ``if'' direction.
Using \cref{equation: IOb identity} and \cref{equation: pattern KKT condition}, we have
\[
  [(\mathbf{I} + \mathbf{O}) b]_{\angl{i}} 
  -v_{\angl{i}}
      \begin{cases}
        \le 0&: i \le \upcar \\
        =0&: \upcar < i \le \upcar + \mdcar \\
        \ge 0&: \upcar + \mdcar < i.
      \end{cases}
\]
For each $i \in [k-1]$, define $\lambda_i, \mu_i \in \mathbb{R}$ where
\[
  \lambda_{\angl{i}}
  =
      \begin{cases}
        -([(\mathbf{I} + \mathbf{O}) b]_{\angl{i}} 
  -v_{\angl{i}}
  )
        &: i \le \upcar \\
        0&: \upcar < i \le \upcar + \mdcar \\
         0&: \upcar + \mdcar < i
      \end{cases}
\]
and
\[
  \mu_{\angl{i}}
  =
      \begin{cases}
        0&: i \le \upcar \\
        0&: \upcar < i \le \upcar + \mdcar \\
        [(\mathbf{I} + \mathbf{O}) b]_{\angl{i}} 
  -v_{\angl{i}}
         &: \upcar + \mdcar < i.
      \end{cases}
\]
It is straightforward to verify that all of \Crefrange{equation: subproblem KKT stationarity}{equation: subproblem KKT complementary slackness 2} are satisfied for all $i \in [k-1]$, i.e., the KKT conditions hold at $b$.
\end{proof}

Recall that we use indices with angle brackets $\angl{1},\angl{2},\dots, \angl{k-1}$ to denote a fixed permutation of $[k-1]$ such that
\[
  v_{\angl{1}} \ge v_{\angl{2}} \ge \dots \ge v_{\angl{k-1}}.
\]

\begin{corollary}
  \label{corollary: kkt equivalent condition}
  Let $t \in [\ell]$
  and $\widetilde{b}$
  be the unique global minimum of
  the optimization \cref{equation: dual subproblem generic}. 
  Then $\widehat{b}^t = \widetilde{b}$ 
  if and only if
    $\mathtt{KKT\_cond}()$ returns true
during the $t$-th iteration of \cref{algorithm: subproblem generic solver}.
\end{corollary}
\begin{proof}
  First, by 
  \cref{lemma: subproblem clipping representation}
  we have $\widehat{b}^t = \widetilde{b}$ if and only if
$\widehat{b}^t$ satisfies the 
KKT conditions 
\Crefrange{equation: subproblem KKT stationarity}{equation: subproblem KKT complementary slackness 2}.
From \cref{lemma: algorithm output and reco identity}, we have $\widehat{b}^{(\mdcarp{t},\upcarp{t})} = \widehat{b}^t$ 
and
$\widehat{\gamma}^{(\mdcarp{t},\upcarp{t})} = \widehat{\gamma}^t$.
To simplify notation, let $\gamma = \widehat{\gamma}^{(\mdcarp{t}, \upcarp{t})}$.
By \cref{lemma: kkt equivalent condition}, $\widehat{b}^{(\mdcarp{t}, \upcarp{t})}$ satisfies the KKT conditions
\Crefrange{equation: subproblem KKT stationarity}{equation: subproblem KKT complementary slackness 2} if and only if the following are true:
\begin{equation*}
  v_{\angl{i}}
\begin{cases}
  \ge C + \gamma &: i \le \upcarp{t}\\
  \in [\gamma, 
  C + \gamma] &: \upcarp{t} < i \le \upcarp{t} + \mdcarp{t} \\
  \le \gamma &: \upcarp{t} + \mdcarp{t} < i.
\end{cases}
\end{equation*}
Since $v_{\angl{1}} \ge v_{\angl{2}} \ge \cdots$, the above system of inequalities holds for all $i\in [k-1]$ if and only if
\[
\begin{cases}
  C+\gamma \le v_{\angl{n^t_{\GTC}}} &: \mbox{if $n^t_{\GTC} >0$.}\\
  \gamma \le v_{\angl{n^t_{\GTC} + n^t_{\BZC}}}
  \mbox{ and }
  v_{\angl{n^t_{\GTC}+1}} \le C + \gamma
                                                    &: \mbox{if $n^t_{\BZC} > 0$,}\\
  v_{\angl{n^t_{\GTC} + n^t_{\BZC}+1}} \le \gamma &: \mbox{if $n^t_{\GTC} + n^t_{\BZC}  < k-1$.}
\end{cases}
\]
Note that the above system holds if and only if $\mathtt{KKT\_cond}()$ returns true.
\end{proof}

\subsubsection{The variables \texorpdfstring{$n_{\BZC}^t$ and $n_{\GTC}^t$}{in the main algorithm}}


\begin{definition}
  \label{definition: vals plus set}
Define the set
$\mathtt{vals}^+ = \{(v_j, \mathtt{dn}, j): v_j > 0, \, j = 1,\dots, k-1 \} \,\,\cup\,\, \{(v_j- C, \mathtt{up}, j) : v_j > C , \, j = 1,\dots, k-1\}$.
Sort the set $\mathtt{vals}^+
=
\{(\gamma_1,\delta_1, j_1),\dots, (\gamma_\ell, \delta_\ell, j_\ell)\}$ so that the ordering of $\{(\gamma_1,\delta_1),\dots, (\gamma_\ell,\delta_\ell)\}$ is identical to $\mathtt{vals}$ from 
Subroutine~\ref{algorithm: construct critical set}-line~2.
\end{definition}

To illustrate the definitions, we consider the following running example
  
\[
\begin{matrix}
  \langle j \rangle & = & \langle 1 \rangle & \langle2 \rangle &\langle3 \rangle &\langle4 \rangle &\langle5 \rangle &\langle6 \rangle &\langle7 \rangle &\langle 8 \rangle &\langle 9 \rangle &\langle 10 \rangle  \\
  \hline
  v_{\langle j \rangle} &=& 1.8&  1.4&  1.4&  1.4&  1.2&  0.7&  0.4&  0.4&  0.1&  -0.2
\end{matrix}
\]
\[
\begin{matrix}
  t &= & 1 & 2 & 3 & 4 & 5 & 6 & 7 & 8 & 9 & 10 & 11& 12 & 13 & 14\\
  \hline
  \gamma_t &= &1.8&  1.4&  1.4&  1.4&  1.2&  0.8&  0.7&  0.4&  0.4&  0.4&  0.4&  0.4&  0.2&  0.1\\
  \delta_t&= & \mathtt{dn}&  \mathtt{dn}&  \mathtt{dn}&  \mathtt{dn}&  \mathtt{dn}&  \mathtt{up}&  \mathtt{dn}&  \mathtt{up}&  \mathtt{up}&  \mathtt{up}&  \mathtt{dn}&  \mathtt{dn}&  \mathtt{up}&  \mathtt{dn}
\end{matrix}
\]

\begin{definition}
  \label{definition: u and d functions}
Define 
\begin{equation}
    \label{equation: definition of u(j) and d(j)}
  \mathtt{u}(j) := 
  \max \{ \tau \in [\ell] : v_{\angl{j}} - C = \gamma_\tau\},
\quad \mbox{and} \quad
\mathtt{d}(j) := 
\max \{ \tau \in [\ell] : v_{\angl{j}} = \gamma_\tau\},
\end{equation}
where $\max \emptyset = \ell+1$.
\end{definition}
Below, we compute $\mathtt{d}(3), \mathtt{d}(6)$ and $\mathtt{u}(3)$ for our running example.

\[
\begin{matrix}
          & &  &&&\mathtt{d}(3)&&&\mathtt{d}(6)&&&&&\mathtt{u}(3)&&\\
          & & &&&\downarrow&&&\downarrow&&&&&\downarrow&&\\
  t &= & 1 & 2 & 3 & 4 & 5 & 6 & 7 & 8 & 9 & 10 & 11& 12 & 13 & 14\\
  \hline
  \gamma_t &= &1.8&  1.4&  1.4&  1.4&  1.2&  0.8&  0.7&  0.4&  0.4&  0.4&  0.4&  0.4&  0.2&  0.1\\
  \delta_t&= & \mathtt{dn}&  \mathtt{dn}&  \mathtt{dn}&  \mathtt{dn}&  \mathtt{dn}&  \mathtt{up}&  \mathtt{dn}&  \mathtt{up}&  \mathtt{up}&  \mathtt{up}&  \mathtt{dn}&  \mathtt{dn}&  \mathtt{up}&  \mathtt{dn}
\end{matrix}
\]

\begin{definition}
  \label{definition: critical sets}
  Define the following sets
  \begin{align*}
  \mathtt{crit}_1(v)&=
\{ \tau \in [\ell] : \gamma_{\tau} > \gamma_{\tau+1}\}
\\
\mathtt{crit}_2(v)
                    &=
\{ \tau \in [\ell] : \gamma_{\tau} = \gamma_{\tau+1}, \, \delta_{\tau} = \mathtt{up},\, 
\delta_{\tau+1} = \mathtt{dn}\}\end{align*}
where $\gamma_{\ell+1}= 0$.
\end{definition}
Below, we illustrate the definition in our running example.
The arrows $\downarrow$ and $\Downarrow$ point to elements of $\mathtt{crit}_1(v)$ and $\mathtt{crit}_2(v)$, respectively.
\[
\begin{matrix}
          & & \downarrow &&&\downarrow&\downarrow&\downarrow&\downarrow&&&\Downarrow&&\downarrow&\downarrow&\downarrow\\
  t &= & 1 & 2 & 3 & 4 & 5 & 6 & 7 & 8 & 9 & 10 & 11& 12 & 13 & 14\\
  \hline
  \gamma_t &= &1.8&  1.4&  1.4&  1.4&  1.2&  0.8&  0.7&  0.4&  0.4&  0.4&  0.4&  0.4&  0.2&  0.1\\
  \delta_t&= & \mathtt{dn}&  \mathtt{dn}&  \mathtt{dn}&  \mathtt{dn}&  \mathtt{dn}&  \mathtt{up}&  \mathtt{dn}&  \mathtt{up}&  \mathtt{up}&  \mathtt{up}&  \mathtt{dn}&  \mathtt{dn}&  \mathtt{up}&  \mathtt{dn}
\end{matrix}
\]

Later, we will show that \cref{algorithm: subproblem generic solver} will halt and output the global optimizer $\widetilde{b}$ on or before the $t$-th iteration where $t \in \mathtt{crit}_1(v) \cup \mathtt{crit}_2(v)$.

\begin{lemma}
  \label{lemma: u d level set cardinality identity}
  Suppose that $t \in \mathtt{crit}_1(v)$. Then
\[
    \# \{ j\in[k-1]: \mathtt{d}(j) \le t\}
    =
    \#\{ \tau \in [t]: \delta_\tau = \mathtt{dn}\},
    \quad \mbox{ and } \quad
    \# \{j\in[k-1]: \mathtt{u}(j) \le t\}
    =
    \#\{ \tau \in [t]: \delta_\tau = \mathtt{up}\}.
\]
\end{lemma}
\begin{proof}
  First, we observe that
\[
  \#
  \{ \tau \in [t] : \delta_{\tau} = \mathtt{up} \}
  =
  \#
  \{(\gamma, \delta, j') \in \mathtt{vals}^+ : 
    \delta = \mathtt{up}, \, 
    \gamma \ge \gamma_t
  \}
\]
Next, note that
$j \mapsto (\gamma_{\mathtt{d}(j)}, \mathtt{up}, \angl{j})$ 
is a bijection from
$
  \{j \in [k-1]: \mathtt{d}(j) \le t\}
$
to
$
  \{(\gamma, \delta, j') \in \mathtt{vals}^+ : 
    \delta = \mathtt{up}, \, 
    \gamma \ge \gamma_t
  \}
$.
To see this, we view the permutation $\angl{1},\angl{2},\dots$ viewed as a bijective mapping $ \angl{\cdot} :[k-1] \to [k-1]$ given by $j\mapsto \angl{j}$.
Denote by $\lgna{\cdot}$ the inverse of $\angl{\cdot}$.
Then the (two-sided) inverse to
$j \mapsto (\gamma_{\mathtt{d}(j)}, \mathtt{up}, \angl{j})$ 
is clearly given by 
$(\gamma, \mathtt{up}, j') \mapsto \lgna{j'}$.
This proves the first identity of the lemma.

The proof of the second identity is completely analogous.
\end{proof}

\begin{lemma}
  \label{lemma: u and d are non-decreasing}
  The functions $\mathtt{u}$ and $\mathtt{d} : [k-1] \to [\ell+1]$ are non-decreasing.
  Furthermore, for all $j \in [k-1]$, we have $\mathtt{u}(j) < \mathtt{d}(j)$.
\end{lemma}
\begin{proof}
  Let $j',j'' \in [k-1]$ be such that $j' < j''$. By the sorting, we have $v_{\angl{j'}} \ge v_{\angl{j''}}$. 
  Now, suppose that $\mathtt{d}(j') > \mathtt{d}(j'')$, then by construction we have
  $\gamma_{\mathtt{d}(j')} < \gamma_{\mathtt{d}(j'')}$.
  On the other hand, we have \[ 
    \gamma_{\mathtt{d}(j')} 
    = v_{\angl{j'}} 
    \ge
    v_{\angl{j''}}
    =
    \gamma_{\mathtt{d}(j'')}
  \]
  which is a contradiction.

  For the ``Furthermore'' part, suppose the contrary that 
$\mathtt{u}(j) \ge \mathtt{d}(j)$.
Then we have
$\gamma_{\mathtt{u}(j)} \le \gamma_{\mathtt{d}(j)}$.
However, by definition, we have
$\gamma_{\mathtt{u}(j)} = v_{\angl{j}} > v_{\angl{j}} - C =  \gamma_{\mathtt{d}(j)}$. This is a contradiction.
\end{proof}

\begin{lemma}
  \label{lemma: the u function identity}
  Let $t \in \mathtt{crit}_1(v)$.
  Then $n^t_{\GTC} = 
    \# \{ j \in [k-1] : \mathtt{u}(j) \le t \}$.
    Furthermore, $[n^t_{\GTC}]
    =\{ j \in [k-1] : \mathtt{u}(j) \le t \}$.
Equivalently,
    for each $j \in [k-1]$, we have $j \le n^t_{\GTC}$ if and only if $\mathtt{u}(j) \le t$.
\end{lemma}
\begin{proof}
  First, we note that 
  \begin{align*}
    n^t_{\GTC}
    &=
    \# \{ \tau \in [t] : \delta_{\tau} = \mathtt{up}\}
    \quad \because\mbox{Subroutine~\ref{algorithm: update variables}-line~2}\\
    &=
    \# \{ j \in [k-1] : \mathtt{u}(j) \le t \}
    \quad \because \mbox{\cref{lemma: u d level set cardinality identity}}
  \end{align*}
  This proves the first part. For the ``Furthermore'' part, let $N := \# \{ j \in [k-1] : \mathtt{u}(j) \le t\}$.
  Since $\mathtt{u}$ is monotonic non-decreasing (\cref{lemma: u and d are non-decreasing}), we have 
  $
    \{ j \in [k-1] : \mathtt{u}(j) \le t\} = [N].
  $
  Since $N = n^t_{\GTC}$ by the first part, we are done.
\end{proof}

\begin{lemma}
  \label{lemma: the d function identity}
  Let $\hat{t}, \check{t} \in \mathtt{crit}_1(v)$ be such that there exists $t \in [\ell]$ where
  \begin{equation}
    \label{equation: the d function identity - main assumption}
    n^t_{\BZC}
    =
    \# \{ j\in[k-1]: \mathtt{d}(j) \le \check{t}\,\}
    -
    \# \{j\in[k-1]: \mathtt{u}(j) \le \hat{t}\,\}.
  \end{equation}
  Then $\mathtt{d}(j) \le \check{t}$ and $\hat{t} < \mathtt{u}(j)$ if and only if
  $n^{\hat{t}}_{\GTC} < j \le n^{\hat{t}}_{\GTC} + n^{t}_{\BZC}$.
\end{lemma}
\begin{proof}
  By \cref{lemma: the u function identity} and \cref{equation: the d function identity - main assumption}, we have
  $
    \# \{ j\in[k-1]: \mathtt{d}(j) \le \check{t}\,\}
    =
n^{\hat{t}}_{\GTC} + n^t_{\BZC}
  $.
  By \cref{lemma: u and d are non-decreasing}, $\mathtt{d}$ is monotonic non-decreasing and so
  $[n^{\hat{t}}_{\GTC} + n^t_{\BZC}]
  =\{ j\in[k-1]: \mathtt{d}(j) \le \check{t}\,\}$.
  Now,
  \begin{align*}
    &\{j\in[k-1]: \mathtt{d}(j) \le \check{t} ,\, \hat{t} < \mathtt{u}(j)\} 
    \\
    =\, &
    \{j\in[k-1]: \mathtt{d}(j) \le \check{t}\,\}\cap \{ j\in [k-1]:  \hat{t} < \mathtt{u}(j)\} 
    \\
    =\, &
    \{j\in[k-1]: \mathtt{d}(j) \le \check{t}\,\}\setminus \{ j\in [k-1]:  \mathtt{u}(j) \le \hat{t}\,\} 
    \\
    =\, &
    [n^{\hat{t}}_{\GTC} + n^{t}_{\BZC}] \setminus [n^{\hat{t}}_{\GTC}],
  \end{align*}
  where in the last equality, we used \cref{lemma: the u function identity}.
\end{proof}

\begin{corollary}
  \label{corollary: the d function identity - special case}
  Let $t \in \mathtt{crit}_1(v)$. 
  Then $\mathtt{d}(j) \le t$ and $t < \mathtt{u}(j)$ if and only if
  $n^{t}_{\GTC} < j \le n^{t}_{\GTC} + n^{t}_{\BZC}$.
\end{corollary}
\begin{proof}
  We apply \cref{lemma: the d function identity} with $t = \hat{t} = \check{t}$, which requires checking that 
  \begin{equation*}
    n^t_{\BZC}
    =
    \# \{ j\in[k-1]: \mathtt{d}(j) \le t\}
    -
    \# \{j\in[k-1]: \mathtt{u}(j) \le t\}.
  \end{equation*}
  This is true because from Subroutine~\ref{algorithm: update variables}-line~2 and 5, we have
  \[
    n^t_{\BZC}
    =
    \#\{ \tau \in [t]: \delta_\tau = \mathtt{dn}\}
    -
    \#\{ \tau \in [t]: \delta_\tau = \mathtt{up}\}.
  \]
  Applying \cref{lemma: u d level set cardinality identity}, we are done.
\end{proof}

\begin{lemma}
  \label{lemma: realizability of type 1}
Let $t \in \mathtt{crit}_1(v)$.
Let $\varepsilon > 0$ be such that for all $\tau,\tau' \in \mathtt{crit}_1(v)$ where $\tau' < \tau$, we have 
$\gamma_{\tau'} - \varepsilon > \gamma_{\tau}$.
Then 
$
  (\mdcarp{t}, \upcarp{t})
  =
  (\mdcarp{\gamma_t - \varepsilon}, \upcarp{\gamma_t - \varepsilon})
$.
\end{lemma}
\begin{proof}
  We claim that
  \begin{equation}
    \label{equation: three cases of the v - gamma + epsilon}
    v_{\angl{j}}
    - \gamma_{t}
     + \varepsilon
\begin{cases}
  < 0 &:  t < \mathtt{d}(j) \\
  \in (0,C) &: \mathtt{d}(j) \le t < \mathtt{u}(j) \\
  > C &: \mathtt{u}(j) \le t.
\end{cases}
\end{equation}
  To prove the $t < \mathtt{d}(j)$ case of \cref{equation: three cases of the v - gamma + epsilon}, we have 
  \begin{align*}
  v_{\angl{j}} - \gamma_{t} + \varepsilon
  &=
  \gamma_{{\mathtt{d}(j)}} - \gamma_{t} + \varepsilon
  \quad 
  \because
  \cref{equation: definition of u(j) and d(j)}
  \\
  &<
  -\varepsilon+ \varepsilon = 0
  \quad 
  \because
  \mbox{ $t < \mathtt{d}(j)$ implies that $\gamma_{t} - \varepsilon > \gamma_{{\mathtt{d}(j)}}$.}
  \end{align*}
  To prove the $\mathtt{d}(j) \le t < \mathtt{u}(j)$ case of \cref{equation: three cases of the v - gamma + epsilon}, 
  we note that
  \begin{align*}
    v_{\angl{j}} - \gamma_{t} + \varepsilon
  &=
  \gamma_{{\mathtt{d}(j)}}
  - \gamma_{t} + \varepsilon
  \quad
  \cref{equation: definition of u(j) and d(j)}
  \\
  &\ge 
  \varepsilon 
  > 0
  \quad \because\mbox{$\mathtt{d}(j) \le t$ implies $\gamma_{{\mathtt{d}(j)}} \ge \gamma_{t}$.}
  \end{align*}
  For the other inequality,
  \begin{align*}v_{\angl{j}} - \gamma_{t} + \varepsilon
  &=
  \gamma_{{\mathtt{u}(j)}} + C - \gamma_{t} + \varepsilon
  \quad \because\cref{equation: definition of u(j) and d(j)}\\
  &< -\varepsilon + C + \varepsilon  = C
  \quad 
  \because \mbox{$t < \mathtt{u}(j)$ implies $\gamma_{t} - \varepsilon > \gamma_{{\mathtt{u}(j)}}$.}
  \end{align*}
  Finally, we prove the $\mathtt{u}(j) \le t$ case of \cref{equation: three cases of the v - gamma + epsilon}. Note that 
  \begin{align*}v_{\angl{j}} - \gamma_{t} + \varepsilon
  &=
  \gamma_{{\mathtt{u}(j)}}  + C
  - \gamma_{t} + \varepsilon 
  \quad \because \cref{equation: definition of u(j) and d(j)}
\\
  &\ge C + \varepsilon > C
  \quad \because \mbox{$\mathtt{u}(j) \le t$ implies that $\gamma_{{\mathtt{u}(j)}} \ge \gamma_{t}$.}
  \end{align*}

  Thus, we have proven \cref{equation: three cases of the v - gamma + epsilon}.
  By \cref{lemma: the u function identity} and \cref{corollary: the d function identity - special case}, \cref{equation: three cases of the v - gamma + epsilon} can be rewritten as
  \begin{equation}
    \label{equation: three cases of the v - gamma + epsilon -- V2}
    v_{\angl{j}}
    - \gamma_{t}
     + \varepsilon
\begin{cases}
  < 0 &:  n^t_{\GTC} + n^t_{\BZC} < j, \\
  \in (0,C) &: n^t_{\GTC} < j \le n^t_{\GTC} + n^t_{\BZC},\\
  > C &: j \le n^t_{\GTC}.
\end{cases}
\end{equation}
Thus, we have 
$\upsetp{\gamma_t - \varepsilon} = \{ \angl{1},\dots, \angl{\upcarp{t}}\}$
and
$\mdsetp{\gamma_t - \varepsilon} = \{ \angl{\upcarp{t}+1},\dots, \angl{\upcarp{t} + \mdcarp{t}}\}$.
By the definitions of 
$\upcarp{\gamma_t - \varepsilon}$
and
$\mdcarp{\gamma_t - \varepsilon}$, we are done.
\end{proof}

\begin{lemma}
  \label{lemma: realizability of type 2}
Let $t \in \mathtt{crit}_2(v)$.
Then
$
  (\mdcarp{t}, \upcarp{t})
  =
  (\mdcarp{\gamma_t}, \upcarp{\gamma_t})
$.
\end{lemma}
\begin{proof}
Let $\hat{t} \in \mathtt{crit}_1(v)$ be such that $\gamma_{\hat{t}} = \gamma_t$,
and $\check{t} = \max \{\tau \in \mathtt{crit}_1(v) : \gamma_{\tau} > \gamma_t\}$.
We claim that
  \begin{equation}
    \label{equation: three cases of the v - gamma + epsilon -- type 2}
    v_{\angl{j}}
    - \gamma_{\hat{t}}
\begin{cases}
  \le 0 &:  \check{t} < \mathtt{d}(j), \\
  \in (0,C) &:
  \mathtt{d}(j) \le \check{t},\, \hat{t} < \mathtt{u}(j), \\
  \ge C &: \mathtt{u}(j) \le \hat{t}.
\end{cases}
\end{equation}
Note that by definition, we have $\gamma_{\check{t}} > \gamma_{\hat{t}}$, which implies that $\check{t}< \hat{t}$.

Consider the first case of \cref{equation: three cases of the v - gamma + epsilon -- type 2} that $\check{t} < \mathtt{d}(j)$.
See the running example \cref{figure: example of type 2 critical iterate - case 1}.
We have by construction that $v_{\langle j \rangle} = \gamma_{\mathtt{d}(j)}$ and so
$v_{\angl{j}} - \gamma_t
  =
  \gamma_{\mathtt{d}(j)} - \gamma_{\check{t}}
\le 0.
$

\begin{figure}
\[
\begin{matrix}
          & &  &&&&&&\check t&&&t&&\hat t&& \mathtt{d}(9)\\
          & & &&&&&&\downarrow&&&\Downarrow&&\downarrow&& \downarrow\\
   & & 1 & 2 & 3 & 4 & 5 & 6 & 7 & 8 & 9 & 10 & 11& 12 & 13 & 14\\
  \hline
  \gamma_t &= &1.8&  1.4&  1.4&  1.4&  1.2&  0.8&  0.7&  0.4&  0.4&  0.4&  0.4&  0.4&  0.2&  0.1\\
  \delta_t&= & \mathtt{dn}&  \mathtt{dn}&  \mathtt{dn}&  \mathtt{dn}&  \mathtt{dn}&  \mathtt{up}&  \mathtt{dn}&  \mathtt{up}&  \mathtt{up}&  \mathtt{up}&  \mathtt{dn}&  \mathtt{dn}&  \mathtt{up}&  \mathtt{dn}
\end{matrix}
\]
\caption{Example of a critical iterate type 2. The first case that $\check{t} < \mathtt{d}(j)$ where $j=9$.}
\label{figure: example of type 2 critical iterate - case 1}
\end{figure}

Next, consider the case when $\mathtt{d}(j) \le \check{t}$ and $\hat{t} < \mathtt{u}(j)$.
Thus, 
\begin{align*}v_{\angl{j}} - \gamma_{\hat{t}}
&> v_{\angl{j}} - \gamma_{\check{t}} \quad \because \gamma_{\check{t}} > \gamma_t  \\
&= \gamma_{\mathtt{d}(j)} - \gamma_{\check{t}} \quad \because \mbox{definition of $\mathtt{d}(j)$}\\
&\ge 0 \quad \because \mathtt{d}(j) \le \check{t} \implies \gamma_{\mathtt{d}(j)} \ge \gamma_{\check{t}}.\end{align*}
On the other hand
\begin{align*}
  v_{\angl{j}} - \gamma_{\hat{t}}
&=
\gamma_{\mathtt{u}(j)} + C - \gamma_{\hat{t}} \quad
\quad \because \mbox{definition of $\mathtt{u}(j)$}\\
&
< C
\quad \because
\hat{t} < \mathtt{u}(j) 
\implies
\gamma_{\hat{t}} > \gamma_{\mathtt{u}(j)}
\end{align*}
Thus, we've shown that in the second case, we have $v_{\angl{j}} - \gamma_{\hat{t}} \in (0,C)$.

We consider the final case that $\mathtt{u}(j) \le \hat{t}$. 
We have
\begin{align*}
  v_{\angl{j}} - \gamma_{\hat{t}}
  &= \gamma_{\mathtt{u}(j)} + C - \gamma_{\hat{t}}
  \quad \because \mbox{definition of $t$}
  \\
  & \ge C \quad \because \mathtt{u}(j) \le \hat{t} \implies \gamma_{\mathtt{u}(j)} \ge \gamma_{\hat{t}}.
\end{align*}
Thus, we have proven \cref{equation: three cases of the v - gamma + epsilon -- type 2}.

Next, we claim that $t,\hat{t},\check{t}$ satisfy the 
condition \cref{equation: the d function identity - main assumption} of \cref{lemma: the d function identity},
i.e.,
  \begin{equation*}
    n^t_{\BZC}
    =
    \# \{ j\in[k-1]: \mathtt{d}(j) \le \check{t}\,\}
    -
    \# \{j\in[k-1]: \mathtt{u}(j) \le \hat{t}\,\}.
  \end{equation*}
  To this end, we first recall that
  \[
    n^t_{\BZC}
    =
    \# \{ \tau \in [t] : \delta_\tau = \mathtt{dn}\}
    -
    \# \{ \tau \in [t] : \delta_\tau = \mathtt{up}\}.
  \]
  By assumption on $t$, for all $\tau$ such that $\check{t} < \tau \le t$, we have $\delta_\tau = \mathtt{up}$. Thus,
    \[
    \# \{ \tau \in [t] : \delta_\tau = \mathtt{dn}\}
    =
    \# \{ \tau \in [\check{t}\,] : \delta_\tau = \mathtt{dn}\}
    =
    \# \{ j\in[k-1]: \mathtt{d}(j) \le \check{t}\,\}
    \]
    where for the last equality, we used \cref{lemma: u d level set cardinality identity}.
    Similarly, for all $\tau$ such that $t < \tau \le \hat{t}$, we have $\delta_\tau = \mathtt{dn}$.
    Thus, we get that analogous result
    \begin{equation}
      \label{equation: t and hat t identity}
      n^t_{\GTC}=
    \# \{ \tau \in [t] : \delta_\tau = \mathtt{up}\}
    =
    \# \{ \tau \in [\hat{t}\,] : \delta_\tau = \mathtt{up}\}
    =
    \# \{ j\in[k-1]: \mathtt{u}(j) \le \hat{t}\,\}
    =
    n^{\hat{t}}_{\GTC}.
    \end{equation}
    Thus, we have verified the
condition \cref{equation: the d function identity - main assumption} of \cref{lemma: the d function identity}.
Now, applying \cref{lemma: the u function identity} and \cref{lemma: the d function identity}, we get
  \begin{equation}
    \label{equation: three cases of the v - gamma + epsilon -- type 2 -- version 2}
    v_{\angl{j}}
    - \gamma_{\hat{t}}
\begin{cases}
  \le 0 &:  
n^{\hat{t}}_{\GTC}  + n^{t}_{\BZC} < j,
  \\
  \in (0,C) &:
  n^{\hat{t}}_{\GTC} < j \le n^{\hat{t}}_{\GTC}  + n^{t}_{\BZC}\\
  \ge C &: 
  j \le n^{\hat{t}}_{\GTC}.
\end{cases}
\end{equation}
By \cref{equation: t and hat t identity} and that $\gamma_t = \gamma_{\hat{t}}$, the above reduces to
  \begin{equation}
    \label{equation: three cases of the v - gamma + epsilon -- type 2 -- version 3}
    v_{\angl{j}}
    - \gamma_{{t}}
\begin{cases}
  \le 0 &:  
  n^{{t}}_{\GTC}  + n^{t}_{\BZC} < j,
  \\
  \in (0,C) &:
  n^{{t}}_{\GTC} < j \le n^{{t}}_{\GTC}  + n^{t}_{\BZC}\\
  \ge C &: 
  j \le n^{{t}}_{\GTC}.
\end{cases}
\end{equation}
Thus,
$\upsetp{\gamma_t } = \{ \angl{1},\dots, \angl{\upcarp{t}}\}$
and
$\mdsetp{\gamma_t} = \{ \angl{\upcarp{t}+1},\dots, \angl{\upcarp{t} + \mdcarp{t}}\}$.
By the definitions of 
$\upcarp{\gamma_t}$
and
$\mdcarp{\gamma_t}$, we are done.
\end{proof}

\subsubsection{Putting it all together}
  If $v_{\max} \le 0$, then \cref{algorithm: subproblem generic solver} returns $\vzero$.

  Otherwise, by \cref{lemma: subproblem clipping representation}, we have $\widetilde{\gamma} \in (0, v_{\max})$.
  \begin{lemma}
    \label{lemma: subproblem generic solver - technical lemma}
    Let $t \in [\ell]$ be such that 
    $(\mdcarp{t}, \upcarp{t})  =
(\mdcarp{\widetilde{\gamma}}, \upcarp{\widetilde{\gamma}})
    $.
    Then 
    during the $t$-th loop of \cref{algorithm: subproblem generic solver}
    we have $\widetilde{b} = \widehat{b}^t$
    and
    $\mathtt{KKT\_cond}()$ returns true.
    Consequently, \cref{algorithm: subproblem generic solver} returns the optimizer $\widetilde{b}$ on or before the $t$-th iteration.
  \end{lemma}
  \begin{proof}
    We have
  \begin{align*}
    \widetilde{b} &= 
  \widehat{b}^{
  (\mdcarp{\widetilde{\gamma}}, \upcarp{\widetilde{\gamma}})}
  \quad \because \mbox{
\cref{lemma: reco comp identity at optimality}
  }
  \\
                  &=
  \widehat{b}^{
  (\mdcarp{t}, \upcarp{t})}
  \quad \because \mbox{Assumption}
  \\
                  &=
                  \widehat{b}^{t}
                  \quad \because \mbox{
  \cref{lemma: algorithm output and reco identity}.
                  }
  \end{align*}
  Thus, by \cref{corollary: kkt equivalent condition}
  $\mathtt{KKT\_cond}()$ returns true
  on the $t$-th iteration.
  This means that \cref{algorithm: subproblem generic solver} halts on or before iteration $t$.
  Let $\tau \in [\ell]$ be the iteration where \cref{algorithm: subproblem generic solver} halts and outputs $\widehat{b}^\tau$.
  Then $\tau \le t$.
  Furthermore,
  by \cref{corollary: kkt equivalent condition},
  $\widehat{b}^\tau = \widetilde{b}$, which proves the ``Consequently'' part of the lemma.
  \end{proof}

  By \cref{lemma: subproblem generic solver - technical lemma}, it suffices to show that 
    $(\mdcarp{t}, \upcarp{t})  =
(\mdcarp{\widetilde{\gamma}}, \upcarp{\widetilde{\gamma}})
$ for some $t \in [\ell]$.
  
  We first consider the case when $\widetilde{\gamma} \ne \gamma_t$ for any $t \in \mathtt{crit}_1(v)$.
  Thus, there exists $t \in \mathtt{crit}_1(v)$ such that
  $\gamma_{t+1} < \widetilde{\gamma} < \gamma_t$, where we recall that $\gamma_{\ell+1} :=0$.

  Now, we return to the proof of \cref{theorem: subproblem generic solver}.
  \begin{align*}
  (\mdcarp{t}, \upcarp{t})
  &=
  (\mdcarp{\gamma_t - \varepsilon}, \upcarp{\gamma_t - \varepsilon})
    \quad \because 
    \mbox{\cref{lemma: realizability of type 1}}
    \label{equation: continue from here}
    \\ &=
    (\mdcarp{\widetilde{\gamma}}, \upcarp{\widetilde{\gamma}})
  \quad \because \mbox{\cref{lemma: comp is locally constant}, and that both $\widetilde{\gamma}$ and $\gamma_i - \varepsilon \in (\gamma_{t+1}, \gamma_t)$}.
  \end{align*}
  Thus, \cref{lemma: subproblem generic solver - technical lemma} implies the result of 
  \cref{theorem: subproblem generic solver}
  under the assumption that
  $\widetilde{\gamma} \ne \gamma_t$ for any $t \in \mathtt{crit}_1(v)$.

  Next, we consider when $\widetilde{\gamma} = \gamma_t$ for some $t \in \mathtt{crit}_1(v)$.
  There are three possibilities:
  \begin{enumerate}
    \item There does not exist $j \in [k-1]$ such that $v_{\angl{j}} = \gamma_t$,
    \item There does not exist $j \in [k-1]$ such that $v_{\angl{j}} - C = \gamma_t$,
    \item There exist $j_1,j_2 \in [k-1]$ such that $v_{\angl{j_1}} = \gamma_t$ and $v_{\angl{j_2}} - C = \gamma_t$.
  \end{enumerate}
  First, we consider case 1.
  We claim that 
  \begin{equation}
    \label{equation: perturbation proof 1}
    (\mdcarp{\gamma_t},
    \upcarp{\gamma_t}
    )
    =
    (\mdcarp{\gamma_t-\varepsilon'},
    \upcarp{\gamma_t - \varepsilon'}
    )
    \quad
    \mbox{for all $\varepsilon'>0$ sufficiently small}.
  \end{equation}
  We first note that 
  $\upcarp{\gamma_t} = \upcarp{\gamma_t - \varepsilon'}$ for all $\varepsilon' > 0$ sufficiently small.
  To see this, let $i \in [k-1]$ be arbitrary. Note that
  \begin{align*}
    i \in \upsetp{\gamma_t}
    \iff
    v_i - \gamma_t \ge C
    &\iff
    v_i - \gamma_t + \varepsilon' \ge C,\, \forall \epsilon' >0, \mbox{ sufficiently small}\\
    &\iff
    i \in \upsetp{\gamma_t - \varepsilon'}
    , \forall \epsilon' >0, \mbox{ sufficiently small.}
  \end{align*}
  Next, we show that 
  $\mdcarp{\gamma_t} = \mdcarp{\gamma_t - \varepsilon'}$ for all $\varepsilon' > 0$ sufficiently small.
  To see this, let $i \in [k-1]$ be arbitrary. Note that
  \begin{align*}
    i \in \mdsetp{\gamma_t}
    \iff
    v_i - \gamma_t \in (0,C)
    &\overset{\dagger}{\iff}
    v_i - \gamma_t + \varepsilon' \in (0,C),\, \forall \epsilon' >0, \mbox{ sufficiently small}\\
    &\iff
    i \in \mdsetp{\gamma_t - \varepsilon'}
    , \forall \epsilon' >0, \mbox{ sufficiently small}
  \end{align*}
  where at ``$\overset{\dagger}{\iff}$'', we used the fact that $v_i - \gamma_t \ne 0$ for any $i \in [k-1]$.
  Thus, we have proven \cref{equation: perturbation proof 1}.
  Taking $\varepsilon'>0$ so small so that both \cref{equation: perturbation proof 1} and the condition in \cref{lemma: realizability of type 1} hold, we have
 \[ 
  (\mdcarp{t}, \upcarp{t})
  =
  (\mdcarp{\gamma_t - \varepsilon'}, \upcarp{\gamma_t - \varepsilon'})
  =
  (\mdcarp{\gamma_t}, \upcarp{\gamma_t})
  =
  (\mdcarp{\widetilde{\gamma}}, \upcarp{\widetilde{\gamma}}).
  \]
  This proves \cref{theorem: subproblem generic solver} under case 1.

  Next, we consider case 2.
  We claim that 
  \begin{equation}
    \label{equation: perturbation proof 2}
    (\mdcarp{\gamma_t},
    \upcarp{\gamma_t}
    )
    =
    (\mdcarp{\gamma_t+\varepsilon''},
    \upcarp{\gamma_t + \varepsilon''}
    )
    \quad
    \mbox{for all $\varepsilon''> 0$ sufficiently small.}
  \end{equation}
  We first note that 
  $\upcarp{\gamma_t} = \upcarp{\gamma_t - \varepsilon''}$ for all $\varepsilon'' > 0$ sufficiently small.
  To see this, let $i \in [k-1]$ be arbitrary. Note that
  \begin{align*}
    i \in \upsetp{\gamma_t}
    \iff
    v_i - \gamma_t \ge C
    &\overset{\ddagger}{\iff}
    v_i - \gamma_t - \varepsilon'' \ge C,\, \forall \epsilon'' >0, \mbox{ sufficiently small}\\
    &\iff
    i \in \upsetp{\gamma_t + \varepsilon''}
    , \forall \epsilon'' >0, \mbox{ sufficiently small.}
  \end{align*}
  where at ``$\overset{\ddagger}{\iff}$'', we used the fact that $v_i - \gamma_t \ne C$ for any $i \in [k-1]$.
  Next, we show that 
  $\mdcarp{\gamma_t} = \mdcarp{\gamma_t - \varepsilon''}$ for all $\varepsilon'' > 0$ sufficiently small.
  To see this, let $i \in [k-1]$ be arbitrary. Note that
  \begin{align*}
    i \in \mdsetp{\gamma_t}
    \iff
    v_i - \gamma_t \in (0,C)
    &\overset{\ddagger}{\iff}
    v_i - \gamma_t - \varepsilon'' \in (0,C),\, \forall \epsilon'' >0, \mbox{ sufficiently small}\\
    &\iff
    i \in \mdsetp{\gamma_t + \varepsilon''}
    , \forall \epsilon'' >0, \mbox{ sufficiently small}
  \end{align*}
  where again at ``$\overset{\ddagger}{\iff}$'', we used the fact that $v_i - \gamma_t \ne C$ for any $i \in [k-1]$.
  Thus, we have proven \cref{equation: perturbation proof 2}.
  Since $\widetilde{\gamma} = \gamma_t \in (0, v_{\max})$ and $\gamma_1 = v_{\max}$, we have in particular that $\gamma_t < \gamma_1$.
  Thus, there exists $\tau \in \mathtt{crit}_1(v)$ such that $\tau < t$ and $\gamma_t < \gamma_\tau$.
  Furthermore, we can choose $\tau$ such that for all $\gamma \in (\gamma_t, \gamma_\tau)$, $\gamma \not \in \mathtt{crit}_1(v)$.
  Let $\varepsilon''>0$ be so small that 
  $\gamma_t + \varepsilon'', \gamma_\tau - \varepsilon'' \in (\gamma_t, \gamma_{\tau})$,
  and furthermore both \cref{equation: perturbation proof 2} and the condition in \cref{lemma: realizability of type 1} hold. We have
 \begin{align*} 
  (\mdcarp{\tau}, \upcarp{\tau})
  &=
  (\mdcarp{\gamma_{\tau} - \varepsilon''}, \upcarp{\gamma_{\tau} - \varepsilon''})
  \quad \because \mbox{
    \cref{lemma: realizability of type 1}
  }
  \\
  &=
  (\mdcarp{\gamma_{t} + \varepsilon''}, \upcarp{\gamma_{t} + \varepsilon''})
  \quad \because\mbox{
  \cref{lemma: comp is locally constant}
  and 
  $\gamma_t + \varepsilon'', \gamma_\tau - \varepsilon'' \in (\gamma_t, \gamma_{\tau})$
  }
  \\
  &=
  (\mdcarp{\gamma_t}, \upcarp{\gamma_t})
  \quad \because \mbox{
    \cref{equation: perturbation proof 2}
  }
  \\
  &=
  (\mdcarp{\widetilde{\gamma}}, \upcarp{\widetilde{\gamma}})
  \quad\because
  \mbox{Assumption.}
  \end{align*}
  This proves \cref{theorem: subproblem generic solver} under case 2.

  Finally, we consider the last case.
  Under the assumptions, we have $t \in \mathtt{crit}_2(v)$.
  Then \cref{lemma: realizability of type 2} 
$
  (\mdcarp{t}, \upcarp{t})
  =
  (\mdcarp{\gamma_t}, \upcarp{\gamma_t})
  =
  (\mdcarp{\widetilde{\gamma}}, \upcarp{\widetilde{\gamma}})
$.
Thus, we have proven
  \cref{theorem: subproblem generic solver} under case 3.
  \qed

\subsection{Experiments}
\label{section: code availability}


The Walrus solver is available at:

\url{https://github.com/YutongWangUMich/liblinear}

The actual implementation is in the file \texttt{linear.cpp} in the class \texttt{Solver\_MCSVM\_WW}.

All code for downloading the datasets used, generating the train/test split, running the experiments and generating the figures are included.
See the \texttt{README.md} file for more information.

All experiments are run on a single machine with the following specifications:

Operating system and kernel:

\texttt{4.15.0-122-generic \#124-Ubuntu SMP Thu Oct 15 13:03:05 UTC 2020 x86\_64 GNU/Linux}

Processor:

\texttt{Intel(R) Core(TM) i7-6850K CPU @ 3.60GH}

Memory:

\texttt{31GiB System memory}

  \subsubsection{On Sharks linear WW-SVM solver}
  Shark's linear WW-SVM solver
  is publicly available in the GitHub repository \url{https://github.com/Shark-ML}.
  Specifically, the C++ code is in \texttt{Algorithms/QP/QpMcLinear.h} in the class \texttt{QpMcLinearWW}.
  Our reimplementation follows their implementation with two major differences.
  In our implementations, neither Shark nor Walrus use the shrinking heuristic.
  Furthermore, we use a stopping criterion based on duality gap, following \cite{steinwart2011training}.

  We also remark that 
  Shark solves the following variant of the WW-SVM which is equivalent to ours after a change of variables.
Let $0 < A \in \mathbb{R}$ be a hyperparameter.
\begin{equation}
  \label{equation: WW-SVM primal optimization - Shark}
  \min_{\mathbf{u} \in \mathbb{R}^{d\times k}}
  F_A(\mathbf{u})
  :=
  \frac{1}{2}\|\mathbf{u}\|_F^2 + A \sum_{i=1}^n 
  \sum_{\substack{j \in [k]:\\ j \ne y_i}}
  \mathtt{hinge}\left((u_{y_i}' x_i - u_j'x_i)/2\right).
\end{equation}
Recall the formulation 
  \cref{equation: WW-SVM primal optimization}
that we consider in this work, which we repeat here:
\begin{equation}
  \label{equation: WW-SVM primal optimization - Vapnik}
  \min_{\mathbf{w} \in \mathbb{R}^{d\times k}}
  G_C(\mathbf{w})
  :=
  \frac{1}{2}\|\mathbf{w}\|_F^2 + C \sum_{i=1}^n 
  \sum_{\substack{j \in [k]:\\ j \ne y_i}}
  \mathtt{hinge}(w_{y_i}' x_i - w_j'x_i).
\end{equation}
The formulation \cref{equation: WW-SVM primal optimization - Shark} is used by \citet{weston1999support},
while the formulation \cref{equation: WW-SVM primal optimization - Vapnik} is used by \citet{vapnik1998statistical}.
These two formulations are equivalent under the change of variables $\mathbf{w} = \mathbf{u}/2$ and $A = 4C$. To see this, note that
\begin{align*}
  G_C(\mathbf{w})
  &=
  G_C(\mathbf{u}/2)
  \\
  &
  =
  \frac{1}{2}\|\mathbf{u}/2\|_F^2 + C \sum_{i=1}^n 
  \sum_{\substack{j \in [k]:\\ j \ne y_i}}
  \mathtt{hinge}((u_{y_i}' x_i - u_j'x_i)/2)
  \\
  &
  =
  \frac{1}{8}\|\mathbf{u}\|_F^2 + C \sum_{i=1}^n 
  \sum_{\substack{j \in [k]:\\ j \ne y_i}}
  \mathtt{hinge}((u_{y_i}' x_i - u_j'x_i)/2)
  \\
  &
  =
  \frac{1}{4}\left(
  \frac{1}{2}\|\mathbf{u}\|_F^2 + 4C \sum_{i=1}^n 
  \sum_{\substack{j \in [k]:\\ j \ne y_i}}
  \mathtt{hinge}((u_{y_i}' x_i - u_j'x_i)/2)
  \right)
  \\
  &
  =
  \frac{1}{4}
  F_{4C}(\mathbf{u})
  =
  \frac{1}{4}
  F_{A}(\mathbf{u}).
\end{align*}
Thus, we have proven
\begin{proposition}
  Let $C > 0$ and $\mathbf{u} \in \mathbb{R}^{d\times k}$.
  Then
  $\mathbf{u}$ is a minimizer of $F_{4C}$ if and only if
  $\mathbf{u}/2$ is a minimizer of $G_C$.
\end{proposition}
In our experiments, we use the above proposition to rescale the variant formulation to the standard formulation.

\subsubsection{Data sets}
\label{section: datasets}
The data sets used are downloaded from the 
``LIBSVM Data: Classification (Multi-class)'' repository:

\url{https://csie.ntu.edu.tw/~cjlin/libsvmtools/datasets/multiclass.html}

We use the scaled version of a data set whenever available.
For testing accuracy, we use the testing set provided whenever available.
The data set \textsc{aloi} did not have an accompanying test set.
Thus, we manually created a test set using methods described in the next paragraph.
See \cref{table: data sets - supplemental} for a summary.

The original, unsplit \textsc{aloi} dataset has $k=1000$ classes, where each class has $108$ instances. 
For creating the test set, we split instances from each class such that first 81 elements are training instances while the last 27 elements are testing instances.
This results in a ``75\% train /25\% test'' split  with training and testing set consisting of 81,000 and 27,000 samples, respectively.

\begin{table}[t]
\caption{Data sets used from the 
``LIBSVM Data: Classification (Multi-class)'' repository.
 Variables $k,\,n$ and $d$ are, respectively, the number of classes, training samples, and features.
The \textsc{scaled} column indicates whether a scaled version of the dataset is available on the repository.
The \textsc{test set provided} column indicates whether a test set of the dataset is provided on the repository.
}
\label{table: data sets - supplemental}
\vskip 0.15in
\begin{center}
\begin{small}
\begin{sc}
\begin{tabular}{lrrrrr}
\toprule
Data set & $k$ & $n$ & $d$& scaled & test set available\\
\midrule
dna      & 3     & 2,000   & 180    & yes & yes\\
satimage & 6     & 4,435   & 36     & yes & yes\\
mnist    & 10    & 60,000  & 780    & yes & yes\\
news20   & 20    & 15,935  & 62,061 & yes & yes\\
letter   & 26    & 15,000  & 16     & yes & yes\\
rcv1     & 53    & 15,564  & 47,236 & no & yes\\
sector   & 105   & 6,412   & 55,197 & yes & yes\\
aloi     & 1,000 & 81,000 & 128   & yes & no\\
\bottomrule
\end{tabular}
\end{sc}
\end{small}
\end{center}
\vskip -0.1in
\end{table}

\subsubsection{Classification accuracy results}\label{section: supp mat - accuracies}

For both algorithms, we use the same stopping criterion: after the first iteration $t$ such that $\mathtt{DG}_{\bullet}^t < \delta \cdot \mathtt{DG}_{\bullet}^1$.
The results are reported in \cref{table: accuracies at 0.009} and \cref{table: accuracies at 0.0009} where $\delta = 0.009$ and $\delta = 0.0009$, respectively.
The highest testing accuracies are in \textbf{bold}.

Note that going from \cref{table: accuracies at 0.009} to \cref{table: accuracies at 0.0009}, the stopping criterion becomes more stringent.
The choice of hyperparameters achieving the highest testing accuracy are essentially unchanged.
Thus, for hyperparameter tuning, it suffices to use the more lenient stopping criterion with the larger $\delta$.

\newpage

\begin{table}[H]
\caption{Accuracies under the stopping criterion $\mathtt{DG}_{\bullet}^t <\delta \cdot \mathtt{DG}_{\bullet}^1$ with $\delta = 0.09$}
\label{table: accuracies at 0.09}
\vskip 0.15in
\begin{center}
\begin{small}
\begin{sc}
\begin{tabular}{lrrrrrrrrrr}
\toprule
$\log_2(C)$&     -6 &     -5 &     -4 &     -3 &     -2 &     -1 &      0 &      1 &      2 &      3 \\
data set &        &        &        &        &        &        &        &        &        &        \\
\midrule
dna      &  94.60 &  \textbf{94.69} &  94.52 &  94.44 &  93.59 &  92.92 &  93.09 &  92.83 &  92.50 &  92.83 \\
satimage &  81.95 &  82.45 &  82.85 &  83.75 &  83.55 &  \textbf{84.10} &  83.95 &  83.30 &  83.95 &  84.00 \\
mnist    &  92.01 &  \textbf{92.16} &  91.97 &  92.15 &  91.92 &  91.76 &  91.62 &  91.66 &  91.70 &  91.58 \\
news20   &  82.24 &  83.17 &  84.20 &  84.85 &  \textbf{85.45} &  85.15 &  85.07 &  84.30 &  84.40 &  83.90 \\
letter   &  69.62 &  \textbf{71.46} &  70.92 &  69.82 &  69.72 &  70.74 &  70.50 &  70.74 &  71.00 &  69.22 \\
rcv1     &  87.23 &  87.93 &  88.46 &  \textbf{88.79} &  88.78 &  88.68 &  88.51 &  88.29 &  88.19 &  88.09 \\
sector   &  93.08 &  93.33 &  93.64 &  93.92 &  \textbf{94.20} &  94.17 &  \textbf{94.20} &  94.08 &  94.14 &  94.14 \\
aloi     &  86.81 &  87.49 &  88.22 &  88.99 &  89.53 &  89.71 &  \textbf{89.84} &  89.53 &  89.06 &  88.21 \\
\bottomrule
\end{tabular}
\end{sc}
\end{small}
\end{center}
\vskip -0.1in
\end{table}

\begin{table}[H]
\caption{Accuracies under the stopping criterion $\mathtt{DG}_{\bullet}^t <\delta \cdot \mathtt{DG}_{\bullet}^1$ with $\delta = 0.009$}
\label{table: accuracies at 0.009}
\vskip 0.15in
\begin{center}
\begin{small}
\begin{sc}
\begin{tabular}{lrrrrrrrrrr}
\toprule
$\log_2(C)$ &     -6 &     -5 &     -4 &     -3 &     -2 &     -1 &      0 &      1 &      2 &      3 \\
data set &        &        &        &        &        &        &        &        &        &        \\
\midrule
dna      &  94.77 &  94.77 &  \textbf{94.94} &  94.69 &  93.59 &  93.09 &  92.24 &  92.24 &  92.16 &  92.16 \\
satimage &  82.35 &  82.50 &  82.95 &  83.55 &  83.55 &  84.10 &  \textbf{84.35} &  84.20 &  84.05 &  84.25 \\
mnist    &  92.34 &  92.28 &  \textbf{92.41} &  92.37 &  92.26 &  92.13 &  92.12 &  91.98 &  91.94 &  91.70 \\
news20   &  82.29 &  83.35 &  84.15 &  85.02 &  \textbf{85.45} &  85.30 &  84.97 &  84.40 &  84.12 &  84.07 \\
letter   &  69.98 &  71.02 &  \textbf{71.74} &  71.52 &  71.36 &  71.46 &  71.20 &  71.56 &  71.44 &  70.74 \\
rcv1     &  87.24 &  87.96 &  88.46 &  88.76 &  \textbf{88.80} &  88.70 &  88.48 &  88.25 &  88.15 &  88.03 \\
sector   &  93.14 &  93.36 &  93.64 &  93.95 &  94.04 &  \textbf{94.08} &  94.04 &  \textbf{94.08} &  93.98 &  93.92 \\
aloi     &  86.30 &  87.21 &  88.20 &  89.00 &  89.34 &  89.63 &  89.99 &  \textbf{90.18} &  89.78 &  89.80 \\
\bottomrule
\end{tabular}
\end{sc}
\end{small}
\end{center}
\vskip -0.1in
\end{table}

\begin{table}[H]
\caption{Accuracies under the stopping criterion $\mathtt{DG}_{\bullet}^t < \delta \cdot \mathtt{DG}_{\bullet}^1$ with $\delta = 0.0009$.}
\label{table: accuracies at 0.0009}
\vskip 0.15in
\begin{center}
\begin{small}
\begin{sc}
\begin{tabular}{lrrrrrrrrrr}
\toprule
$\log_2(C)$ &     -6 &     -5 &     -4 &     -3 &     -2 &     -1 &      0 &      1 &      2 &      3 \\
data set &        &        &        &        &        &        &        &        &        &        \\
\midrule
dna      &  94.77 &  94.69 &  \textbf{95.11} &  94.77 &  93.76 &  93.34 &  92.41 &  92.24 &  92.24 &  92.24 \\
satimage &  82.35 &  82.65 &  83.20 &  83.65 &  83.80 &  84.10 &  \textbf{84.20} &  84.10 &  84.15 &  84.10 \\
mnist    &  92.28 &  92.38 &  \textbf{92.43} &  92.24 &  92.21 &  92.13 &  92.16 &  91.92 &  91.79 &  91.65 \\
news20   &  82.27 &  83.45 &  84.00 &  85.00 &  \textbf{85.40} &  85.22 &  85.02 &  84.52 &  84.10 &  83.97 \\
letter   &  70.04 &  71.28 &  \textbf{71.70} &  71.66 &  71.48 &  71.30 &  71.26 &  71.30 &  71.02 &  71.22 \\
rcv1     &  87.23 &  87.98 &  88.46 &  88.76 &  \textbf{88.79} &  88.69 &  88.48 &  88.25 &  88.12 &  88.02 \\
sector   &  93.20 &  93.39 &  93.64 &  93.92 &  94.01 &  \textbf{94.04} &  \textbf{94.04} &  94.01 &  93.95 &  93.83 \\
aloi     &  86.17 &  87.01 &  87.99 &  88.66 &  89.04 &  89.46 &  89.64 &  \textbf{89.70} &  89.69 &  89.51 \\
\bottomrule
\end{tabular}
\end{sc}
\end{small}
\end{center}
\vskip -0.1in
\end{table}

\newpage

\begin{table}[H]
  \caption{Accuracies under the stopping criterion $\mathtt{DG}_{\bullet}^t <\delta \cdot \mathtt{DG}_{\bullet}^1$ with $\delta = 0.09$ (first row in each cell), $=0.009$ (second row) and $=0.0009$ (third row).}
\label{table: combined table}
\vskip 0.15in
\begin{center}
\begin{small}
\begin{sc}
\begin{tabular}{lrrrrrrrrrr}
\toprule
$\log_2(C)$&     -6 &     -5 &     -4 &     -3 &     -2 &     -1 &      0 &      1 &      2 &      3 \\
data set &        &        &        &        &        &        &        &        &        &        \\
\midrule
dna ($\delta = 0.09$)      &  94.60 &  \textbf{94.69} &  94.52 &  94.44 &  93.59 &  92.92 &  93.09 &  92.83 &  92.50 &  92.83 \\
$\delta = 0.009$ &  94.77 &  94.77 &  \textbf{94.94} &  94.69 &  93.59 &  93.09 &  92.24 &  92.24 &  92.16 &  92.16 \\
$\delta = 0.0009$ &  94.77 &  94.69 &  \textbf{95.11} &  94.77 &  93.76 &  93.34 &  92.41 &  92.24 &  92.24 &  92.24 \\
\midrule
satimage &  81.95 &  82.45 &  82.85 &  83.75 &  83.55 &  \textbf{84.10} &  83.95 &  83.30 &  83.95 &  84.00 \\
&  82.35 &  82.50 &  82.95 &  83.55 &  83.55 &  84.10 &  \textbf{84.35} &  84.20 &  84.05 &  84.25 \\
&  82.35 &  82.65 &  83.20 &  83.65 &  83.80 &  84.10 &  \textbf{84.20} &  84.10 &  84.15 &  84.10 \\
\midrule
mnist    &  92.01 &  \textbf{92.16} &  91.97 &  92.15 &  91.92 &  91.76 &  91.62 &  91.66 &  91.70 &  91.58 \\
&  92.34 &  92.28 &  \textbf{92.41} &  92.37 &  92.26 &  92.13 &  92.12 &  91.98 &  91.94 &  91.70 \\
&  92.28 &  92.38 &  \textbf{92.43} &  92.24 &  92.21 &  92.13 &  92.16 &  91.92 &  91.79 &  91.65 \\
\midrule
news20   &  82.24 &  83.17 &  84.20 &  84.85 &  \textbf{85.45} &  85.15 &  85.07 &  84.30 &  84.40 &  83.90 \\
&  82.29 &  83.35 &  84.15 &  85.02 &  \textbf{85.45} &  85.30 &  84.97 &  84.40 &  84.12 &  84.07 \\
&  82.27 &  83.45 &  84.00 &  85.00 &  \textbf{85.40} &  85.22 &  85.02 &  84.52 &  84.10 &  83.97 \\
\midrule
letter   &  69.62 &  \textbf{71.46} &  70.92 &  69.82 &  69.72 &  70.74 &  70.50 &  70.74 &  71.00 &  69.22 \\
&  69.98 &  71.02 &  \textbf{71.74} &  71.52 &  71.36 &  71.46 &  71.20 &  71.56 &  71.44 &  70.74 \\
&  70.04 &  71.28 &  \textbf{71.70} &  71.66 &  71.48 &  71.30 &  71.26 &  71.30 &  71.02 &  71.22 \\
\midrule
rcv1     &  87.23 &  87.93 &  88.46 &  \textbf{88.79} &  88.78 &  88.68 &  88.51 &  88.29 &  88.19 &  88.09 \\
&  87.24 &  87.96 &  88.46 &  88.76 &  \textbf{88.80} &  88.70 &  88.48 &  88.25 &  88.15 &  88.03 \\
&  87.23 &  87.98 &  88.46 &  88.76 &  \textbf{88.79} &  88.69 &  88.48 &  88.25 &  88.12 &  88.02 \\
\midrule
sector   &  93.08 &  93.33 &  93.64 &  93.92 &  \textbf{94.20} &  94.17 &  \textbf{94.20} &  94.08 &  94.14 &  94.14 \\
&  93.14 &  93.36 &  93.64 &  93.95 &  94.04 &  \textbf{94.08} &  94.04 &  \textbf{94.08} &  93.98 &  93.92 \\
&  93.20 &  93.39 &  93.64 &  93.92 &  94.01 &  \textbf{94.04} &  \textbf{94.04} &  94.01 &  93.95 &  93.83 \\
\midrule
aloi     &  86.81 &  87.49 &  88.22 &  88.99 &  89.53 &  89.71 &  \textbf{89.84} &  89.53 &  89.06 &  88.21 \\
&  86.30 &  87.21 &  88.20 &  89.00 &  89.34 &  89.63 &  89.99 &  \textbf{90.18} &  89.78 &  89.80 \\
&  86.17 &  87.01 &  87.99 &  88.66 &  89.04 &  89.46 &  89.64 &  \textbf{89.70} &  89.69 &  89.51 \\
\bottomrule
\end{tabular}
\end{sc}
\end{small}
\end{center}
\vskip -0.1in
\end{table}

\subsubsection{Comparison with convex program solvers}\label{section: supp mat - cvx solvers}
For solving 
  \cref{equation: dual subproblem generic},
we compare the speed of Walrus
(\cref{algorithm: subproblem generic solver})
versus
  the general-purpose, commercial convex program (CP) solver MOSEK.
  We generate random instances of the subproblem \cref{equation: dual subproblem generic} by randomly sampling $v$.
  The runtime results of Walrus and the CP solver are shown in \cref{table: benchmarking subproblem solver varying k} and \cref{table: benchmarking subproblem solver varying C}, where each entry is the average over 10 random instances.
\begin{table}[H]
  \caption{Runtime in seconds for solving random instances of the problem \cref{equation: dual subproblem generic}. The parameter $C = 1$ is fixed while $k$ varies.}
\label{table: benchmarking subproblem solver varying k}
\vskip 0.15in
\begin{center}
\begin{small}
\begin{sc}
\begin{tabular}{lrrrrrr}
\toprule
$\log_2(k-1)$ & 2 & 4 & 6 & 8 & 10 & 12 \\ 
\midrule
Walrus & 0.0009 & 0.0001 & 0.0001 & 0.0001 & 0.0002 & 0.0005 \\ 
\midrule
CP solver & 0.1052 & 0.0708 & 0.0705 & 0.1082 & 0.5721 & 12.6057 \\ 
\bottomrule
\end{tabular}
\end{sc}
\end{small}
\end{center}
\vskip -0.1in
\end{table}

\begin{table}[H]
  \caption{Runtime in seconds for solving random instances of the problem \cref{equation: dual subproblem generic}. The parameter $k = 2^8 +1$ is fixed  while $C$ varies.}
\label{table: benchmarking subproblem solver varying C}
\vskip 0.15in
\begin{center}
\begin{small}
\begin{sc}
\begin{tabular}{lrrrrrrr}
\toprule
$\log_{10}(C)$ & -3 & -2 & -1 & 0 & 1 & 2 & 3 \\ 
\midrule
Walrus & 0.0004 & 0.0001 & 0.0001 & 0.0001 & 0.0001 & 0.0001 & 0.0001 \\ 
\midrule
CP Solver & 0.1177 & 0.1044 & 0.1046 & 0.1005 & 0.1050 & 0.1127 & 0.1206 \\ 
\bottomrule
\end{tabular}
\end{sc}
\end{small}
\end{center}
\vskip -0.1in
\end{table}

As shown here, the analytic solver Walrus is faster than the general-purpose commercial solver by orders of magnitude.

\end{appendices}

\end{document}